\documentclass[11pt]{article}
\usepackage{amssymb,amsmath,amsfonts,amsthm,mathtools,color}

\usepackage{mathrsfs}
\usepackage{bm} 

\usepackage{comment} 
\usepackage[title,titletoc,header]{appendix}
\usepackage{graphicx,subfig}

\usepackage{paralist}

\usepackage{tikz}
\usetikzlibrary{arrows,positioning,shapes.geometric}

\usepackage[linesnumbered, ruled,vlined]{algorithm2e}
\SetKwInput{KwInput}{Input}                

\graphicspath{ {./figures/} }

\usepackage{indentfirst}
\usepackage{multicol}
\usepackage{booktabs}
\usepackage{url}
\usepackage[outdir=./]{epstopdf} 

\usepackage[shortlabels]{enumitem}

\usepackage{hyperref}
\hypersetup{
    colorlinks=true, 
    linktoc=all,     
    linkcolor=blue,  
}
\numberwithin{equation}{section}

\newtheorem{Theorem}{Theorem}[section]
\newtheorem{Lemma}[Theorem]{Lemma}
\newtheorem{Proposition}[Theorem]{Proposition}
\newtheorem{Assumption}{H.\!\!}

\theoremstyle{definition}
\newtheorem{Definition}{Definition}[section]

\newtheorem{Example}{Example}[section]

\theoremstyle{remark}
\newtheorem{Remark}{Remark}[section]

 \def\p{\partial} \def\nb{\nonumber}
\def\to{\rightarrow}
 \def\ol{\overline}    \def\ul{\underline}
\def\Om{\Omega}  \def\om{\omega} 
 
\newcommand{\q}{\quad}   \newcommand{\qq}{\qquad}

\def\l{\label}  
\def\f{\frac}  \def\fa{\forall}
\def\b{\beta}  \def\a{\alpha} 
 
\def\eps{\varepsilon}

 \def\t{\times}  
\def\ms{\medskip} \def\bs{\bigskip}

\def \la{\langle} \def\ra{\rangle}

\def\bs{\boldsymbol}

\def\cA{\mathcal{A}}
\def\cB{\mathcal{B}}
\def\cC{\mathcal{C}}

\def\cE{\mathcal{E}}
\def\cF{\mathcal{F}}
\def\cG{\mathcal{G}}
\def\cH{\mathcal{H}}

\def\cK{\mathcal{K}}

\def\cM{\mathcal{M}}
\def\cN{\mathcal{N}}
\def\cO{\mathcal{O}}

\def\cR{\mathcal{R}}
\def\cS{\mathcal{S}}

\def\cU{\mathcal{U}}
\def\cV{\mathcal{V}}

\def\cY{\mathcal{Y}}

\def\d{{\mathrm{d}}}

\def\bA{{\textbf{A}}}

\def\sB{\mathbb{B}}

\def\sE{{\mathbb{E}}}
\def\sF{{\mathbb{F}}}
\def\sG{{\mathbb{G}}}

\def\sN{{\mathbb{N}}}
\def\sP{\mathbb{P}}

\def\sR{{\mathbb R}}
\def\sS{{\mathbb{S}}}

\newcommand{\tr}{\textnormal{tr}}
\newcommand{\lf}{\lfloor}
\newcommand{\rf}{\rfloor}

\DeclareMathOperator*{\argmin}{arg\,min}
\DeclareMathOperator*{\dom}{dom}

\DeclareMathOperator*{\essinf}{ess\,inf}

\newcommand{\lc}
{\mathrel{\raise2pt\hbox{${\mathop<\limits_{\raise1pt\hbox
{\mbox{$\sim$}}}}$}}}

\newcommand{\gc}
{\mathrel{\raise2pt\hbox{${\mathop>\limits_{\raise1pt\hbox{\mbox{$\sim$}}}}$}}}

\newcommand{\ec}
{\mathrel{\raise2pt\hbox{${\mathop=\limits_{\raise1pt\hbox{\mbox{$\sim$}}}}$}}}

\def\bb{\begin{equation}} \def\ee{\end{equation}}
\def\bbn{\begin{equation*}} \def\een{\end{equation*}}

\def\beqn{\begin{eqnarray}}  \def\eqn{\end{eqnarray}}

\def\beqnx{\begin{eqnarray*}} \def\eqnx{\end{eqnarray*}}

\def\bn{\begin{enumerate}} \def\en{\end{enumerate}}

\def\bd{\begin{description}} \def\ed{\end{description}}



\begin{document}

\title{
Exploration-exploitation trade-off for 
 continuous-time
episodic reinforcement learning with 
 linear-convex models
 }

\author{
\and 
 Lukasz Szpruch\thanks{School of Mathematics, University of Edinburgh and Alan Turing Institute,  \texttt{L.Szpruch@ed.ac.uk}}
\and Tanut Treetanthiploet\thanks{Alan Turing Institute,  \texttt{ttreetanthiploet@turing.ac.uk}}
\and Yufei Zhang\thanks{Department of Statistics, London School of Economics and Political Science,  \texttt{y.zhang389@lse.ac.uk}}
}
\date{}
\maketitle

\noindent\textbf{Abstract.} 
We develop a probabilistic framework for analysing model-based reinforcement learning in the episodic setting. We then apply it to study finite-time horizon stochastic control problems with linear dynamics but unknown coefficients and convex, but possibly irregular, objective function. Using probabilistic representations, we study regularity of the associated cost functions and establish precise estimates for the performance gap between applying optimal feedback control derived from estimated and true model parameters.  We identify conditions under which this performance gap is quadratic, improving the linear performance gap in recent work [X. Guo, A. Hu, and Y. Zhang, \textit{arXiv preprint}, arXiv:2104.09311, (2021)], which matches the results obtained for stochastic linear-quadratic problems. Next, we propose a phase-based learning algorithm for which we show how to optimise exploration-exploitation trade-off and achieve sublinear regrets in high probability and expectation. When assumptions needed for the quadratic performance gap hold, the algorithm achieves an order $\mathcal{O}(\sqrt{N}\ln{N})$ high probability regret, in the general case, and an order $\mathcal{O}((\ln{N})^2)$  expected regret, in self-exploration case, over $N$ episodes, matching the best possible results from the literature. The analysis requires novel concentration inequalities for correlated continuous-time observations, which we derive.

\medskip
\noindent
\textbf{Key words.} 
 Continuous-time reinforcement learning,  
linear-convex,  Shannon’s entropy, quadratic performance gap,
 Bayesian inference, conditional sub-exponential random variable

\ms
\noindent
\textbf{AMS subject classifications.} 
68Q32, 62C15, 93C73, 93E11



\medskip
 
%
%

\section{Introduction}
 
 Reinforcement learning (RL) is a core topic in machine learning and is concerned with sequential decision-making in an uncertain environment (see \cite{sutton1998introduction}). Two key concepts in RL are \textit{exploration}, that corresponds to learning via interactions with the random environment,  and \textit{exploitation}, that corresponds to optimising the objective function given accumulated information. 
 The latter can be studied using the (stochastic) control theory, while the former relies on the theory of statistical learning.  When the state dynamics is not available while learning, we talk about model-free reinforcement learning, which is  only concern with the search for the optimal policy. On the other hand, model-based RL assumes the model is given, or it is learned from data. Model-based RL has advantage over model-free approaches as a) it can be systematically studied using powerful and well understood stochastic control techniques that have been developed over the last half-century, see \cite{krylov2008controlled,bertsekas1995neuro,bensoussan2004stochastic,fleming2006controlled}; b) it is more appropriate for high-stakes decision-making, overcoming some shortcomings of less interpretable ``black-box" approaches \cite{rudin2019stop}. 
  
  In this work we study finite time model-based RL when the underlying environment is modelled by a linear stochastic differential equation with unknown drift parameter and the agent is minimising (known) convex, but possibly irregular, cost function. Before introducing the learning algorithm and the main contributions of this work we briefly review the classical stochastic linear-convex (LC) control problems, with all the parameters assumed to be known, fixing the notation along the way.
 
\subsection{Linear-convex control problem with observable parameter}
\label{sec:linear_convex_control}

Let $T> 0$ be a given terminal time,
  $(\Omega, \cF,  \sP)$
be a complete probability space
which supports a $d$-dimensional standard Brownian motion $W$,
 and $\sF  $ be the filtration generated by $W$ augmented 
 by the $\sP$-null sets.
Let
 $\theta= (A ,B )\in 
 \sR^{d \t (d+p)} $ be 
 given parameters.\footnote{With a slight abuse of notation, we identity $\sR^{d\t d}\t \sR^{d\t p}$ with   $\sR^{d \t (d+p)}$ throughout the paper.}
Consider the following control problem with parameter $\theta $:
\bb\label{eq:lc}
V^\star(\theta )=
\inf_{\a\in \cH^2_{\sF}(\Omega  ; \sR^p)} J(\a;\theta ),
\q \textnormal{with}\q 
J(\a;\theta )=\sE\left[\int_0^T f(t,X^{\theta ,\a}_t,\a_t)\, \d t+g(X_T^{\theta ,\a})\right],
\ee
where
$X^{\theta ,\a}\in \cS^2_{\sF}(\Omega ; \sR^d)$ is the strong solution to the following  dynamics: 
\bb\label{eq:lc_sde}
\d X_t =({A}  X_t+{B} \a_t)\,\d t+  \d W_t, \q t\in [0,T],
\q X_0=x_0,
\ee
and  the  given initial state $x_0$  and functions $f$ and $g$ 
 satisfy the following conditions as in \cite{guo2021reinforcement}:

\begin{Assumption}\label{assum:lc}
$T>0$, $x_0\in \sR^d$, and $f:[0,T]\t \sR^d\t \sR^p\to\sR\cup\{\infty\}$ and $g:\sR^d\to \sR$ such that 
\begin{enumerate}[(1)]

\item\l{item:f0R}
there exist measurable functions 
 $f_0:[0,T]\t \sR^d\t \sR^p\to\sR$ 
and 
$h:  \sR^p\to \sR\cup \{\infty\}$
such that 
\bb\label{eq:f_decompose}
f(t,x,a)={f}_0(t,x,a)+h( a), \q  \fa (t,x,a)\in [0,T]\t \sR^d\t \sR^p.
\ee
For all $(t,x)\in [0,T]\t \sR^d$,
$f_0(t,x,\cdot)$ is convex,
$f_0(t,\cdot,\cdot)$ 
is 
differentiable
with 
a Lipschitz continuous derivative,
and 
$\sup_{t\in [0,T]}(|f_0(t,0,0)|+|\p_{(x,a)}f_0(t,0,0)|)<\infty$.
Moreover, 
$h$ is proper, lower semicontinuous, and convex;\footnotemark
\footnotetext{We say a function $h:\sR^p\to \sR\cup\{\infty\}$ is proper if it has a nonempty domain 
$\operatorname{dom} h\coloneqq \{a\in \sR^p \mid h(a)<\infty\}$.}

\item 
there exists $\lambda>0$ such that 
for all 
$t\in [0,T]$, $(x,a),(x',a')\in \sR^d\t \sR^p$, and $\eta\in [0,1]$,
$$ \l{eq:strong_convex}
\eta f(t,x,a)+(1-\eta ) f(t,x',a')
\ge 
f(t,\eta x+(1-\eta) x', \eta a+(1-\eta) a')
+ \eta(1-\eta)\tfrac{\lambda}{2}|a-a'|^2;
$$
\item\l{item:lc_g}
 $g$ is  convex and differentiable with
a Lipschitz continuous derivative.

\end{enumerate}
\end{Assumption}

\begin{Remark}
(H.\ref{assum:lc}) 
not only includes the most commonly used linear-quadratic models as special cases,
 but also
 allows for 
  practically important nonsmooth control problems  in engineering and machine learning. 
For example,
 problems with control constraints 
 correspond to $h$
 being
the indicator function of action sets,
and 
sparse control problems 
correspond to 
$h$ being 
the $\ell^1$-norm of control variables.
Moreover, in the reinforcement learning literature
(see e.g., \cite{cohen2020asymptotic,
guo2020entropy,
vsivska2020gradient,
wang2020reinforcement,
reisinger2021regularity}),
one often consider the entropy-regularised cost function:
\bb\label{eq:regularized_cost}
 f(t,x,a)\coloneqq  \bar{f}_0(t,x)^\top a + h_{\textrm{en}}(a),
\q \fa (t,x,a)\in [0,T]\t \sR^d\t \sR^p,
\ee
with  $\bar{f}_0:[0,T]\t \sR^d\to \sR^p$
being a sufficiently regular function,
and $h_{\textrm{en}}:\sR^p\to \sR\cup\{\infty\}$
being the Shannon's entropy 
 function such that 
\bb\label{eq:entropy}
 h_{\textrm{en}}(a)=
\begin{cases}
\sum_{i=1}^p a_i\ln (a_i), & a\in \Delta_p\coloneqq\{
a\in [0,1]^p\mid 
\sum_{i=1}^p a_i=1\},\\
\infty, & a\in \sR^p\setminus \Delta_p.
\end{cases}
\ee
It has been shown  
that including the entropy function
$h_{\textrm{en}}$ in the optimisation objective 
leads to  more robust decision-making
\cite{reisinger2021regularity}
and accelerates the convergence of gradient-descent algorithms
\cite{vsivska2020gradient}. 
Observe that 
$h_{\textrm{en}}$ is only differentiable on the interior of 
 $\Delta_p$, and its derivative blows up near the boundary. 
We refer the reader to 
\cite{guo2021reinforcement}
and the references therein for other applications of nonsmooth control problems.

\end{Remark}

In this work we focus on Lipschitz continuous   feedback controls defined as follows.

\begin{Definition}\l{def:fb}
For each $C\ge 0$,
let $\cV_C$ be the following 
space of 
 feedback controls:
 \begin{equation}\l{eq:lipschitz_feedback}
\cV_C \coloneqq 
\left\{ 
\psi: [0,T]\t \sR^d\to \sR^p
\,\middle\vert\, 
\begin{aligned}
&\textnormal{$\psi$ is measurable and 
for all $(t,x,y)\in [0,T]\t \sR^d\t \sR^d$,}
\\
&\textnormal{$|\psi(t,0)|\le C$
 and 
 $|\psi(t,x)-\psi(t,y)|\le C|x-y|$.}
 \end{aligned}
\right\}
\end{equation}
 
\end{Definition}

The following proposition shows that 
under (H.\ref{assum:lc}), 
the control problem \eqref{eq:lc}
admits  an optimal feedback control
$\psi_{\theta }$,
which depends continuously on the parameter $\theta $.
The proof can be found in 
Theorems 2.5 and 2.6 and Lemma 2.8
in 
\cite{guo2021reinforcement}.

\begin{Proposition}\label{prop:existence_fb}
Suppose (H.\ref{assum:lc}) holds,
and let $\Theta $ be a bounded subset
of $ \sR^{d\t (d+p)}$.
Then there exists a constant $C\ge 0$, such that for all $\theta\in \Theta$, 
there exists  $\psi_\theta\in \cV_C$ such that 
\begin{enumerate}[(1)]
\item
 \label{item:optimal}
$\psi_\theta$ is an optimal feedback control
of  \eqref{eq:lc}, i.e., 
$V^\star (\theta)=J(\psi_\theta;\theta)$, where for all 
$\theta\in \sR^{d\t (d+p)}$
and  $\psi\in \cup_{C\ge 0} \cV_C$,
\bb\l{eq:loss_feedback}
J(\psi;\theta)\coloneqq\sE\left[\int_0^T f(t,X^{\theta, \psi}_t,\psi(t,X^{\theta, \psi}_t))\, \d t+g(X^{\theta, \psi}_T)\right]\in \sR\cup\{\infty\},
\ee
and  
$X^{\theta, \psi}\in \cS^2_{\sF}(\Om; \sR^d)$ is the state process associated with $\theta$ and  $\psi$  
satisfying
for $t\in[0,T]$, 
\bb\l{eq:lc_sde_fb}
  \d X^{\theta,\psi}_t =\theta 
  Z^{\theta,\psi}_t
\,\d t+\, \d W_t,
\q X^{\theta,\psi}_0=x_0,
\q \textnormal{with} \q 
Z^{\theta,\psi}_t=
\begin{psmallmatrix}
X^{\theta,\psi}_t
\\
\psi(t,X^{\theta,\Psi}_t)
\end{psmallmatrix};
\ee
\item 
\label{item:f_psi}
for all $(t,x)\in [0,T]\t \sR^d$ and $\theta,\theta'\in \Theta$,  
$|\psi_\theta(t,x)-\psi_{\theta'}(t,x)|\le C(1+|x|)|\theta-\theta'|$ and 
$|f(t,x,\psi_\theta(t,x))|\le C(1+|x|^2)$;
\item 
$\sR^{d\t (d+p)}\ni \theta\mapsto 
V^\star (\theta)\in \sR$ is continuous. 
\end{enumerate}
\end{Proposition}

\subsection{Phased-based learning algorithm and our contributions}

\label{sec: overview}

In this section we provide a road map of the key ideas and contributions of this work without introducing needless technicalities. 
The precise assumptions and statements of the results can be found in Sections \ref{sec:framework} and \ref{sec:PEGE}. 
 
 \paragraph{Episodic learning and observations.}

To encode the 
 fact that the controller does not observe the 
true parameter $\theta$, we treat it as a random variable which we denote by $\bs{\theta}$. The probability distribution induced by $\bs{\theta}$ describes the uncertainty the controller has about unknown $\theta$. Agent learns about $\bs \theta$ by executing a sequence of feedback policies $(\psi_m)_{m\in \sN}$ and observing corresponding realisation of controlled dynamics \eqref{eq:lc_sde_fb}, $(X^{\bs{\theta}, \psi_m}_t)_{t\in [0,T],m\in \sN}$.
This is often referred to as the episodic  framework 
in the RL literature
(see e.g., \cite{osband2013more,basei2021logarithmic,guo2021reinforcement}).
At the beginning of the $m$-th episode, 
the information available to the agent
for designing $\psi_m$
  is incorporated in the $\sigma$-algebra $
  \cG_{m-1}=\sigma\{ X^{\bs{\theta},\psi_n}_s\mid s\in[0,T], n=1,\ldots, m-1 \}$
  generated by the state processes from  previous episodes. 

 It is  worth emphasising that the agent only observes the state process $(X^{\bs{\theta},\psi}_t)$, but  not the Brownian motion $W$, since otherwise the problem reduces to the classical control problem. Indeed,  consider \eqref{eq:lc_sde_fb}. Suppose that  one can choose $\psi$ such that 
 $\int_0^t Z^{\bs{\theta},\psi}_s (Z^{\bs{\theta},\psi}_s)^\top \d s$
  is almost surely invertible for some $t > 0$
  (see Lemma \ref{lemma:nondegenerate}). Then, for all $t > 0$, we can write
\[ 
\bs{\theta} = \left(
\Big(
 \int_0^t Z^{\bs{\theta},\psi}_s
 (\d X^{\bs{\theta},\psi}_s-\d W_s)^\top
 \Big)^\top
\right)\left(\int_0^t Z^{\bs{\theta},\psi}_s (Z^{\bs{\theta},\psi}_s)^\top \d s\right)^{-1}.
\] 
But this means that if we observe both $X^{\bs{\theta},\psi}$ and $W$, then we effectively know $\bs \theta$. 

\paragraph{Performance measure 
and incomplete learning.} 
We denote by $\Psi_m(\cdot)$ the (random) feedback  policy used in the $m$-th episode, in order to emphasise its dependence on 
  realisations of controlled dynamics from all previous episodes; see Definition \ref{Def: Admissible RL}  for details. 
To measure the performance of an algorithm $\bs{\Psi} = (\Psi_m)_{m \in \sN}$ with the corresponding state processes $(X^{\bs{\theta},\bs{\Psi}, m})_{ m \in \sN}$, we often consider its associated cost
\bb\label{eq:cost_m_intro}
\ell_m(\bs{\Psi},\bs{\theta})
\coloneqq
\int_0^T f(t,X^{\bs{\theta},\bs{\Psi}, m}_t,
\Psi_m(
\cdot, t,X^{\bs{\theta},\bs{\Psi}, m}_t))
\, \d t+g(X^{\bs{\theta},\bs{\Psi}, m}_T).
\ee
The regret of learning 
up to  $N\in \sN$ episodes is then defined by
\begin{equation}
\label{eq: Bayesian regret_intro}
\cR(N,  \bs{\Psi},\bs{\theta}) \coloneqq \sum_{m=1}^N \Big( \ell_m(\bs{\Psi},\bs{\theta})
 - V^\star(\bs{\theta}) \Big),
\end{equation}
where $V^\star(\theta)=J(\psi_{{\theta}}; {\theta})$ is the optimal cost that agent can achieve knowing the parameter $\bs{\theta}$ (cf.~\eqref{eq:loss_feedback}). 
The regret characterises the cumulative loss from taking sub-optimal policies in all episodes. Agent's aim is to construct a learning algorithm for which the regret (either in high-probability or in expectation) $\cR(N,  \bs{\Psi},\bs{\theta})$ growths \textit{sublinearly} in $N$.

To quantify the regret of an algorithm $\bs{\Psi} = (\Psi_m)_{m \in \sN}$, one can decompose  the regret into
\begin{equation}
\label{eq: Bayesian regret decomposition}
\cR(N,  \bs{\Psi},\bs{\theta})= \sum_{m=1}^N \Big( \ell_m(\bs{\Psi},\bs{\theta})
 - J(\Psi_m ; \bs{\theta}) \Big) + \sum_{m=1}^N \Big( J(\Psi_m; \bs{\theta}) - 
 J(\psi_{\bs{\theta}}; \bs{\theta})
\Big),
\end{equation}
where $J(\Psi_m ; \bs{\theta})$ is given in \eqref{eq:loss_feedback}
and $\psi_\theta$ is given in Proposition \ref{prop:existence_fb}. The first term  vanishes under expectation and can be bounded with $\cO(\sqrt{N} \ln N)$ in the high-probability sense (Propositions \ref{prop:subguassian_subexponential}
and 
\ref{prop: concentration inequality}).  This means that it suffices to analyse the second term which is the (expected) regret that has been analysed in
 \cite{basei2021logarithmic,guo2021reinforcement} assuming a self-exploration property of greedy policies
 (see Remark \ref{rmk:exploration_policy}). However, 
the following examples shows these  greedy policies in general do not guarantee exploration and consequently convergence to the optimal solution,
which is often referred to as incomplete learning in the literature (see e.g.,~\cite{
keskin2018incomplete}).

\begin{Example}
\label{ex:non_exploration}
Consider minimising the following quadratic cost
\[
J(\alpha ; {\theta}) = \sE \left[ \int_0^T
(\a^2_{1,t}+\a^2_{2,t}) \d t + X_T^2  \right]\,,
\]
over all  admissible controls $\alpha_1$ and $\alpha_2$, subject to a  one-dimensional state dynamics:
\[
\d X_t =(B_1 \a_{1,t}+B_2\a_{2,t}) \d t +  \d W_t\,,
\q t\in [0,T],
\q X_0=x_0, \q 
\textnormal{with $B_1, B_2 \neq 0$}.
\]
The classical linear-quadratic (LQ) control theory shows that
for any $B=(B_1,B_2)$, 
 the optimal feedback control is given by $\psi_\theta(t,x) =-p_tB^\top x$, where $(p_t)_t$ satisfies  
$p'_t - (BB^\top) p_t^2 = 0$ and $p_T =1$. 
Consequently,
if the agent starts with the initial estimate  $\tilde{B}^0 = (\tilde{B}_1^0,0)$ with some $\tilde{B}_1^0 \neq 0$,  then, assuming that $\tilde{B}_1^m \neq 0$ for all $m\in \sN$, the greedy strategy will result in  a control of the form $\alpha_t = (\alpha_{1,t}, 0)$ for all episodes and almost surely result in  estimates  $\tilde{B}^m = (\tilde{B}_1^m,0)$. In particular, the parameter $B_2$ and the optimal policy  will never be learnt.
Note that this example contradicts the full column rank condition of $B=(B_1,B_2)$ in \cite{guo2021reinforcement}.
\end{Example}

\paragraph{Separation of exploration and greedy exploitation episodes.}
%


Motivated by the example, in this work  we assume that there exists a policy $\psi^e$ that allows to explore the space (see (H.\ref{Assump: Exploration strategy})).
We prove that 
such an exploration policy can be explicitly constructed 
under very general conditions 
(see Proposition \ref{prop:construct_psi_e}),
and guarantees to improve the accuracy of parameter estimation
 (see  Remark \ref{rmk:exploration_policy}). With the  exploration policy $\psi^e$ at hand, we
separate the exploration episodes, in which  $\psi^e $ is exercised and the estimate $\tilde{\bs{\theta}}_m$ of $\bs{\theta}$ is updated based on available information $\cG_m$, from exploitation episodes   in which greedy policies $\psi_{\tilde{\bs{\theta}}_m}$ are exercised. From Example \ref{ex:non_exploration} we see that in the latter case there is no guarantee that the agent's knowledge of $\bs{\theta}$ increases. 

 Coming back to regret analysis and noting that $V(\theta)=J(\psi_\theta;\theta)$, the second term on the right hand side of \eqref{eq: Bayesian regret decomposition} can be decomposed as 
 \begin{align}
    \label{eq: expected regret decomposition}
 \sum_{m=1}^N \big( J(\Psi_m ; \bs{\theta}) - J(\psi_{\bs{\theta}}; \bs{\theta}) \big)
 &=\!\!
 \sum_{m \in [1,N] \cap \cE^{\bs{\Psi}}}
 \!
 \big( J(\psi^e ; \bs{\theta}) 
 - J(\psi_{\bs{\theta}}; \bs{\theta})
 \big) 
 + \! \sum_{m \in [1,N] \setminus \cE^{\bs{\Psi}} }
 \!\!
 \big( J(\psi_{\tilde{\bs{\theta}}_{m}} ; \bs{\theta}) 
 - J(\psi_{\bs{\theta}}; \bs{\theta})
 \big),
 \end{align}
  where  $\cE^{\bs{\Psi}} = \{m \in \sN | \Psi_m = \psi^e \}$.
 The first term of \eqref{eq: expected regret decomposition}  induces loss due to exploration
and increases linearly in the size of  exploration episodes. 
 The second term of  \eqref{eq: expected regret decomposition}, on the other hand, describes loss due to inaccuracy of our parameter  estimate.
  Balancing  exploration  and
exploitation episodes to control the growth of regret is precisely where the exploration-exploitation
trade-off appears in this work.



 
 


 To quantify the regrets of exploitation episodes, 
 one of the main contributions of this work is to show that there exist constants $L_\Theta,\beta>0,r\in [1/2,1]$
  such that
   for all $\theta_0 \in \Theta$,
   \bb\label{eq:performance_gap_intro}
   |J(\psi_{{\theta}}; \theta_0) - J(\psi_{{\theta_0}}; \theta_0)|\le L_\Theta|{\theta}-\theta_0|^{2r},
   \,\,\,\,\,\fa {\theta} \in\sB_\beta(\theta_0),
   \ee
   where
    $\psi_\theta$ is an  optimal feedback control
 with parameter 
$\theta$. In particular,
\begin{itemize}
    \item 
    We prove that the performance gap \eqref{eq:performance_gap_intro} 
    holds with $r=1$
provided that the cost functions in \eqref{eq:lc}
 are Lipschitz differentiable (Theorem \ref{thm:performance_gap_smooth})  or involve the 
 (nonsmooth) entropy regularisation function
  (Theorem \ref{thm:performance_gap_entropy}).
To the best of our knowledge, this is the first paper  on  the (optimal) quadratic performance gap for continuous-time RL problems beyond the  LQ setting.
\end{itemize}


\paragraph{Regret for Phased Exploration and Greedy Exploitation algorithm.}
While the general framework we present in this work can be used to study many   popular parameter estimators from the literature, we focus the exposition on a truncated 
 maximum a posteriori (MAP) estimate
$\tilde{\bs{\theta}}_m$ based on posterior distribution $\pi(
\bs{\theta} | \cG_{m} )$ for all $m\in\sN$.
 The accuracy of the MAP estimate 
 is  evaluated
 by using concentration inequalities for \textit{conditional} sub-exponential random variables,
 in order to address  correlations among all observations.
Based on the precise convergence rate of $ |\tilde{\bs{\theta}}_{m} - \bs{\theta}|$ in terms of $m$ and the  performance gap  $r$, 
we design
Phased Exploration and Greedy Exploitation 
 algorithms (see Algorithm \ref{Alg: PEGE}) that balance exploration and exploitation:

 \begin{itemize}
 \item In the case where
 the accuracy of $(\tilde{\bs{\theta}}_m)_{m\in\sN}$ improves only after the  exploration episodes,
 we 
 prove that  the regret of  the PEGE algorithm 
is of the magnitude 
  $\cO(N^{{1}/{1+r}}   )$ (up to logarithmic factors) with high probability; 
see Theorem \ref{Theorem: regret}.
In particular,  our result 
extends the $\cO(\sqrt{N})$-regret bound for 
tabular Markov decision problems 
and discrete-time LQ-RL problems 
to continuous-time LC-RL problems  with Lipschitz differentiable costs
(Theorem \ref{thm:performance_gap_smooth}) and entropy-regularised costs (Theorem \ref{thm:performance_gap_entropy}).

 \item In the case where
 the accuracy of $(\tilde{\bs{\theta}}_m)_{m\in\sN}$ improves  after both the  exploration and exploitation episodes,  we 
 achieve an improved expected regret 
 of the order  
 $\cO(N^{1-r} (\ln N)^r)$ if $r<1$
and $\cO((\ln N)^2)$ if $r=1$;
  see Theorem \ref{Theorem: regret_self explore}.
To the best of our knowledge, this is the first 
logarithmic regret bound beyond the linear parameterised Markov decision setting  or  the LQ-RL setting. 
 \end{itemize}

\subsection{Related works}


The problem that we study here falls under the umbrella of the partially observable stochastic control problems. There, one typically assumes that the controlled dynamics is not observable, and the observation process is given by another stochastic process and 
in addition, 
the coefficients of the controlled dynamics and the observation process are given to the agent.
One typically also is interested in continuous learning without explicit exploration.  We refer the reader to excellent monographs on this subject \cite{davis1977linear,bensoussan2004stochastic,liptser1977statistics} and more recent development \cite{allan2019parameter,allan2020pathwise,cohen2020asymptotic}.
The problem that we study here is different. We do observe controlled dynamics but treat its parameters as random variables. In that sense, the problem is degenerate.  Our main objective is to study regret in the episodic setting.

 There is a vast literature on sublinear regret bounds  for discrete-time RL algorithms
 (see e.g., \cite{rusmevichientong2010linearly,
 osband2013more,osband2014model,
 he2021logarithmic} for bandit problems and tabular Markov decision problems, and 
 \cite{abeille2018improved, mania2019certainty,dean2020sample,simchowitz2020naive} for discrete-time LQ-RL problems). In particular,  a Phased Exploration and Greedy Exploitation algorithm has been introduced in \cite{rusmevichientong2010linearly}
for 
linear bandit problems. Furthermore,  many results in literature  e.g.,
\cite{osband2013more} and  
\cite{ucrl2} give a regret with order depending on the size of state and action. 
 For RL problems with continuous-time models,
 attentions have been mostly on algorithm development (see e.g.,
 \cite{munos1998reinforcement,munos2000study,munos2006policy}), 
 while  
regret analysis of  learning algorithms  is limited.

 In the seminal work \cite{duncan1999adaptive}  authors established an asymptotically sublinear regret for 
regularised least-squares algorithms 
in an ergodic continuous-time LQ setup without an exact order of the regret bound.
Recently, \cite{basei2021logarithmic, guo2021reinforcement}
extend 
the
least-squares algorithms 
to the finite-time horizon episodic setting and analyze their non-asymptotic regrets. They  prove 
that if  optimal controls of the true model  automatically exploit the parameter space, 
then a greedy least-squares algorithm with suitable initialisation admits 
a non-asymptotically logarithmic expected regret for LQ models  \cite{basei2021logarithmic}, and a $\cO(\sqrt{N})$ expected regret for LC models \cite{guo2021reinforcement}.
Unfortunately, as shown in Example \ref{ex:non_exploration}  and in \cite{
keskin2018incomplete}, such a self-exploration property may not hold in general, even for LQ models. 
Furthermore, the learning algorithm studied here works for an arbitrary initialisation.

There are various reasons for the relatively slow theoretical progress in non-asymptotic performance analysis  of continuous-time RL algorithms.
Analysing exploitation (i.e., the performance gap \eqref{eq:performance_gap_intro}) for continuous-time models 
 often requires to study high-order stability (such as Lipschitz continuity or differentiability) of 
 associated 
fully nonlinear HJB PDEs with respect to model parameters, 
which 
has always been one of the formidable challenges in the control theory. 
Furthermore, developing an effective exploration strategy 
 for learning the environment in the nonlinear setting 
adds additional complexity to the problem that we study.  
Indeed, in the context of LQ-RL a natural exploration strategy 
is to add Gaussian noises with diminishing variances
to greedy policies (see e.g., \cite{duncan1999adaptive, mania2019certainty,dean2020sample,simchowitz2020naive}).
This technique 
exploits heavily  the fact that the greedy policy  
of a LQ control problem
is  affine in state variables,
which  clearly does not hold 
for general RL problems.
Moreover, 
to ensure a  finite exploration cost, 
one needs to construct exploration noises with range $\dom(h)=\{a\in \sR^k\mid h(a)<\infty\}$, 
which can be computationally challenging. 

\paragraph{Notation:} For each 
$T>0$,
 filtered probability space
$(\Om,\cF,\sF=\{\cF_t\}_{t\in[0,T]},\sP)$  satisfying the usual condition
 and
  Euclidean space $(E,|\cdot|)$,
we introduce the following spaces: 
\begin{itemize}
\item 
$\cS^q_\sF(\Om\t [t,T];E)$, 
$q\in [2,\infty)$, $t\in [0,T]$, is the space of 
 $E$-valued $\sF$-progressively measurable
 processes 
$Y: \Om\t [t,T]\to E$ 
satisfying
$\|Y\|_{\cS^q}=
\sE[\sup_{s\in [t,T]}|Y_s|^q]^{1/q}<\infty$;
\item
 $\cH^q_\sF(\Om\t [t,T];E)$, $q\in [2,\infty)$, $t\in [0,T]$, is the space of 
   $E$-valued $\sF$-progressively measurable
 processes 
$Z: \Om\t [t,T]\to E$ 
 satisfying $\|Z\|_{\cH^p}=\sE[(\int_t^T|Z_s|^2\,\d s)^{q/2}]^{1/q}<\infty$.
\end{itemize}
For notational simplicity, we denote $\cS^q_\sF(\Om; E) = \cS^q_\sF(\Om\t [0,T];E)$ and 
$\cH^q_\sF(\Om; E) = \cH^q_\sF(\Om\t [0,T];E)$.
We also denote 
 by $C\in[0,\infty)$ a generic constant, which depends only on the constants appearing in the assumptions and may take a different value at each occurrence.

\section{A probabilistic framework for 
episodic RL problems}
\label{sec:framework}

%

In this section, we introduce   a rigorous probabilistic framework  for the  episodic learning procedure outlined in Section \ref{sec: overview}.

We start by defining  admissible learning algorithms
rigorously.
An  essential step is to describe the available information for decision-making,
 namely the $\sigma$-algebra 
with respect  to which 
a learning policy is measurable at   each episode. 
The 
 fact that  the controller does not observe the 
true parameter $\theta $
indicates that  
the parameter is not measurable with respect to the available information $\sigma$-algebra. But because a  deterministic quantity is measurable  
with respect to 
 any  $\sigma$-algebra on the  space, we
 treat 
the parameter
as  
a  random variable $\bs{\theta} = (\bs{A},\bs{B})$.
The range $\Theta$
of the  parameter $\bs{\theta}$
 is assumed to be bounded and known to the controller.

\begin{Assumption}
\label{assum:bound}
$\Theta $ is a 
nonempty 
bounded measurable subset of $\sR^{d\t (d+p)}$. 
\end{Assumption}

As $\bs{\theta}$ is a random parameter, 
it is crucial to  distinguish among different sources of randomness throughout the learning process,
and to ensure that all sources of randomness 
 are supported on the same probability space.
To this end, we   work with a complete probability space
of the following form:
\bb\label{eq:space}
 (\Om, \cF,\sP)\coloneqq \bigg(\Om^\Theta\t \prod_{m=1}^\infty \Om^{W^m}, \cF^\Theta\otimes \bigotimes_{m=1}^\infty \cF^{W^m}, \sP^\Theta\otimes \bigotimes_{m=1}^\infty \sP^{W^m} \bigg).
 \ee
The space
 $(\Om^\Theta,\cF^\Theta,\sP^\Theta)$ 
 supports the  random variable  $\bs{\theta}$
  taking values in   $\Theta$, 
and 
 for each $m\in \sN$,
  the space $(\Om^{W^m},\cF^{W^m},\sP^{W^m})$
 supports a $d$-dimensional Brownian motion $W^m$,
 which 
 describes the randomness involved in making the new observation at the $m$-th episode.
Without loss of generality, we assume that
 $\Om^\Theta = \Theta$, 
 $\bs{\theta}$ is the identity map on $\Om^\Theta$,
  $\Om^{W^m}=C([0,T];\sR^d)$ and $W^m$ is the coordinate map of $\Om^{W^m}$. In particular, any element $\om\in \Om$ can be written
as $\om=(\om^\theta, \om^W) \in \Theta \t \prod_{m=1}^\infty \Om^{W^m}$ with
$\om^\theta \coloneqq \theta$ and
 $\om^W \coloneqq (\om_1^W, \om_2^W, \ldots)$. 
We extend canonically $\bs{\theta}$ on $\Om$ by setting 
 $\bs{\theta}(\om^\theta,\om^W)\coloneqq \bs{\theta}(\om^\theta)$,
 and extend similarly on $\Om$ any random variable on 
 $\Om^{W^m}$.

 It is important to notice that 
one execution of the learning algorithm with all episodes 
corresponds to one sample $\om=(\theta, 
\om_1^W, \om_2^W, \ldots) $ in the space $\Om^\Theta\t \prod_{m=1}^\infty \Om^{W^m}$.
This  describes the episodic learning problem precisely, where
the system parameter   $\bs{\theta}$ 
is realised
at the beginning of the learning process,
and the agent aims to learn the unknown parameter 
through    a sequence of episodes
driven by independent noises. 

Now we give the precise definition of an admissible learning algorithm.
Recall the space of Lipschitz feedback controls $\cV_C$
defined in \eqref{def:fb}.

\begin{Definition}
\label{Def: Admissible RL}
Let $(\Om,\cF,\sP)$ be 
defined in \eqref{eq:space},
 $\cN$ be the $\sigma$-algebra generated by $\sP$-null sets,
and 
 $\bs{\Psi} = (\Psi_m)_{m\in \sN}$ 
 be a sequence of functions
 $\Psi_m: \Om\t [0,T]\t \sR^d\to \sR^p$ 
 such that 
for all $m\in \sN$, 
 there exists a constant $C_m\ge 0$ such that 
$\Psi_m(\om,\cdot)\in \cV_{C_m}$ for all $\om\in \Om$.

 We say that  $\bs{\Psi}= (\Psi_m)_{m\in \sN}$ is an admissible  learning  algorithm
 (simply referred to as a learning algorithm)
 if for  all $m\in \sN$,
 $\Psi_m$ is 
 $(\cG^{\bs{\Psi}}_{m-1}\otimes\cB([0,T])\otimes \cB(\sR^d)  )/\cB(\sR^p)$-measurable, with
 the $\sigma$-algebras
  $(\cG^{\bs{\Psi}}_m)_{m\in \mathbb N\cup\{0\}}$   defined recursively as follows:
 $\cG^{\bs{\Psi}}_0\coloneqq \cN$ and 
 for all 
 $m\in\sN$,
take
the   $(\cG^{\bs{\Psi}}_{m-1}\otimes\cB([0,T])\otimes \cB(\sR^d)  )/\cB(\sR^p)$-measurable function $\Psi_m$, 
and define 
 $\cG^{\bs{\Psi}}_m\coloneqq 
\cG^{\bs{\Psi}}_{m-1}\vee \sigma\big\{{X}^{\bs{\theta}, \bs{\Psi}, m}_t \mid t \in [0,T] \big\}$,
where 
$X^{\bs{\theta},\bs{\Psi},m} \in \cS^2_{\sF^m}(\Om; \sR^d)$
is the strong solution of
  \bb\label{eq:X_m}
  \d X^{\bs{\theta},\bs{\Psi}, m}_t =\bs{\theta} 
  Z^{\bs{\theta},\bs{\Psi}, m}
\,\d t+\, \d W^m_t,
\q X^{\bs{\theta},\bs{\Psi}, m}_0=x_0,
\q \textnormal{with 
$Z^{\bs{\theta},\bs{\Psi}, m}_t=
\begin{psmallmatrix}
X^{\bs{\theta},\bs{\Psi}, m}_t
\\
\Psi_m(t,X^{\bs{\theta},\bs{\Psi}, m}_t)
\end{psmallmatrix},
$
}
 \ee
with  the filtration  $\sF^m$  
generated by $W^m$ augmented 
 by $\cG^{\bs{\Psi}}_{m-1}$.
 The $\sigma$-algebra 
$\cG^{\bs{\Psi}}_m$, $m\in \sN\cup\{0\}$,
describes the  available
 information 
 for the agent
 after the $m$-th episode,
 which 
is  generated by the  state processes
 from all previous episodes (up to $\sP$-null sets).
 \end{Definition}

\begin{Remark}
For  discrete-time Markov decision problems
where the  environment is driven by Markov chains and  the agent observes the state dynamics
in discrete-time,
a  learning algorithm is usually defined as 
  a sequence of deterministic 
  functions, 
each mapping 
the history of observations (i.e., 
a finite number of random variables)
to action space  (see e.g., \cite{osband2013more,osband2014model}). 
By the Doob--Dynkin lemma,
this is equivalent to defining
a learning algorithm as a sequence of random functions
measurable with respect to the $\sigma$-algebras generated by all historical  (random) observations. 
 
Definition \ref{Def: Admissible RL}
 extends these concepts to continuous-time RL problems,
 and specifies  the required
 measurability condition of a learning algorithm. It is required for the rigorous analysis of learning algorithms in model-based RL, see Proposition \ref{prop: PEGE as a learning}. 
As  $\bs{\theta}$ and $(W^n)_{n=1}^m$ 
in general
are not measurable with respect to $\cG^{\bs{\Psi}}_m$,
Definition \ref{Def: Admissible RL}
 restricts 
the choices of learning policies and prevents the agent from using the values of $\bs{\theta}$  and $(W^n)_{n=1}^m$  in the decision-making.
\end{Remark}

We then define  the regret of a learning algorithm
as introduced in  Section \ref{sec: overview}. 
For any given learning algorithm 
$\bs{\Psi} = (\Psi_m)_{m\in\sN}$, its associated cost at the $m$-th episode is given by
\bb\label{eq:cost_m}
\ell_m(\bs{\Psi},\bs{\theta})
\coloneqq
\int_0^T f(t,X^{\bs{\theta},\bs{\Psi}, m}_t,
\Psi_m(
\cdot, t,X^{\bs{\theta},\bs{\Psi}, m}_t))
\, \d t+g(X^{\bs{\theta},\bs{\Psi}, m}_T),
\ee
with $X^{\bs{\theta},\bs{\Psi}, m}$ being defined as in 
Definition \ref{Def: Admissible RL},
and the regret of learning 
up to  $N\in \sN$ episodes is defined by
\begin{equation}
\label{eq: Bayesian regret}
\cR(N,  \bs{\Psi},\bs{\theta}) \coloneqq \sum_{m=1}^N \Big( \ell_m(\bs{\Psi},\bs{\theta})
 -  
V^\star(\bs{\theta}) \Big),
\end{equation}
where for each $\theta\in \sR^{d\t (d+p)}$, 
$V^\star(\theta)=J(\psi_{{\theta}}; {\theta})$
as shown in Proposition  \ref{prop:existence_fb}.

Note that 
for any given $N\in \sN$,
$\cR(N,  \bs{\Psi},\bs{\theta}):\Om\to \sR\cup\{\infty\} $ is not deterministic since it depends on the 
realisations of the random  parameter 
$\bs{\theta}$
and 
the Brownian motions 
$(W^m)_{m=1}^N$.
Due to (H.\ref{assum:lc}) and 
the measurability of $\Psi_m$,
$\ell_m(\bs{\Psi},\bs{\theta}):\Om\to \sR\cup\{\infty\}$ 
is $\cF/\cB(\sR\cup\{\infty\})$-measurable
for all $m\in \sN$.
Moreover, by 
the continuity of 
$V^\star:\Theta\to \sR$
(see 
Proposition \ref{prop:existence_fb}),
$V^\star(\bs{\theta}):\Om\to \sR$ is 
$\cF/\cB(\sR)$-measurable,
and hence 
$\cR(N,  \bs{\Psi},\bs{\theta})$
is a random variable on 
$(\Om, \cF,\sP)$ for all $N\in \sN$.


 \section{Phase-based learning algorithm and  main results}
\label{sec:PEGE}

In this section, we propose and analyse
a phase-based learning algorithm
for LC-RL problems under the probabilistic framework 
introduced in Section  \ref{sec:framework}.

\subsection{Bayesian inference for  parameter estimation}
\label{sec:bayesian_inference}

As indicated by 
Definition \ref{Def: Admissible RL}, 
for the $m$-th episode, one  needs to choose a (random) strategy $\Psi_m$ 
based on the 
historical observation  described by 
the $\sigma$-algebra 
$\cG^{\bs{\Psi}}_{m-1}$.
The probabilistic setup in Section \ref{sec:framework}
suggests to first 
 infer the random parameter $\bs{\theta}$
via a Bayesian algorithm 
based on  a given prior and the samples observed so far,
and then design the strategy $\Psi_m$ using the estimated parameter.

In the sequel, 
we  
 derive a Bayesian estimation of $\bs{\theta}$
by imposing   a matrix normal prior distribution,
i.e.,
we assume the  conditional law
$\pi(\bs{\theta} | \cG^{\bs{\Psi}}_{0})=\cM\cN(\hat{\theta}_0,I_d, V_0)$
with some $\hat{\theta}_0\in \sR^{d\t (d+p)}$ and $V_0\in \sS^{d+p}_+$.\footnotemark{}
\footnotetext{Let $M\in \sR^{d\t n}$, $U\in \sS^{d}_+$ and $V\in \sS^n_+$,
we say 
a 
$ \sR^{d\t n}$-valued random variable $X$ follows the matrix normal distribution $\cM\cN(M,U,V)$ if 
$X$  has 
the probability density function 
$\pi(X)=\frac{\exp\left(-\frac{1}{2}\operatorname{tr}\left(
(X-M) V^{-1}(X-M)^\top U^{-1}
\right)\right)}{(2\pi)^{dn/2}\operatorname{det}(V)^{d/2}\operatorname{det}(U)^{n/2}}$.
}%
The normality of the prior distribution  enables representing the  posterior distribution
$\pi(\bs{\theta} | \cG^{\bs{\Psi}}_{m-1})$
 in terms of its sufficient statistic, 
while the  separable structure of the covariance matrix 
effectively reduces the dimension of the covariance process (from an $\sS^{d(d+p)}_+$-valued process to 
an $\sS^{d+p}_+$-valued process), which subsequently  
allowing for  a trackable update of the sufficient statistics for large $d$ and $p$.
Note that
the estimated parameters 
will  be
projected into  some bounded sets 
to recover the  desired 
boundedness of estimates (see \eqref{eq:project}). 
We also emphasise that the matrix normal prior distribution is only imposed to motivate the parameter estimation scheme, 
 and will not be used in the regret analysis of our learning algorithm.

Let  $\bs{\Psi}=(\Psi_m)_{m\in \sN}$
be a given learning algorithm, 
$(X^{\bs{\theta},\bs{\Psi}, m})_{m\in \sN}$
be the associated state processes,
 and $(\cG^{\bs{\Psi}}_m)_{ m \in \sN}$
be  the corresponding observation filtration
(cf.~Definition
\ref{Def: Admissible RL}).
Then for each $m\in \sN$,  by considering 
\eqref{eq:X_m},
the likelihood function of $\bs{\theta}$, or
the Radon--Nikodym derivative of 
$\frac{\d \sP^{X^{\bs{\theta},\bs{\Psi}, m}}}{\d \sP^{W^{ m}}}$
of the measure $\sP^{X^{\bs{\theta},\bs{\Psi}, m}}$, is given by   
\begin{align*}
&\frac{\d \sP^{X^{\bs{\theta},\bs{\Psi}, m}}}{\d \sP^{W^{ m}}}
(t,X^{\bs{\theta},\bs{\Psi}, m})
 = 
 \exp\bigg(\int_0^t (\bs{\theta } Z^{\bs{\theta}, \bs{\Psi}, m}_s)^\top \d X^{\bs{\theta}, \bs{\Psi}, m}_s 
- \frac{1}{2} \int_0^t (\bs{\theta } Z^{\bs{\theta}, \bs{\Psi}, m}_s)^\top (\bs{\theta } Z^{\bs{\theta}, \bs{\Psi}, m}_s) \d s
\bigg)
\\
 &\quad = \exp\bigg(\operatorname{tr}
 \bigg(
 \Big(
 \int_0^t Z^{\bs{\theta}, \bs{\Psi}, m}_s
 (\d X^{\bs{\theta}, \bs{\Psi}, m}_s)^\top
 \Big)^\top \bs{\theta}^\top 
- \bs{\theta}\Big(\frac{1}{2} \int_0^t  Z^{\bs{\theta}, \bs{\Psi}, m}_s (  Z^{\bs{\theta}, \bs{\Psi}, m}_s)^\top \d s\Big)\bs{\theta}^\top 
\bigg)
\bigg).
\end{align*}
Hence
given 
the initial belief  that $\bs{\theta}$ follows the prior distribution 
$\pi(\bs{\theta} | \cG^{\bs{\Psi}}_{0})=\cM\cN(\hat{\theta}_0,I_d, V_0)$,
 the posterior distribution of $\pi(\bs{\theta} | \cG^{\bs{\Psi}}_{m})$ after the $m$-the episode is given by
\begin{align}
\label{eq:posterior}
\begin{split}
 \pi(\bs{\theta} | \cG^{\bs{\Psi}}_{m})
& \propto \pi(\bs{\theta}| \cG^{\bs{\Psi}}_{0}) \prod_{n = 1}^m \frac{\d \sP^{X^{\bs{\theta},\bs{\Psi}, n}}}{\d \sP^{W^{ n}}}(T,X^{\bs{\theta},\bs{\Psi}, n})
 \\
 & \propto  \exp\bigg(  \operatorname{tr}
 \bigg(
\Big(\hat{\bs{\theta}}^{\bs{\Psi,m}} (V^{\bs{\theta }, \bs{\Psi}, m})^{-1}\Big) \bs{\theta}^\top 
-\frac{1}{2} \bs{\theta} (V^{\bs{\theta }, \bs{\Psi}, m})^{-1} \bs{\theta}^\top 
\bigg) \bigg)
\\
 & \propto  \exp\bigg( -\frac{1}{2} \operatorname{tr}
 \Big(
(\bs{\theta}-\hat{\bs{\theta}}^{\bs{\Psi,m}}) (V^{\bs{\theta }, \bs{\Psi}, m})^{-1}
(\bs{\theta}-\hat{\bs{\theta}}^{\bs{\Psi,m}})^\top
\Big) \bigg)
\end{split}
\end{align}
where 
$\propto$ stands for   proportionality  up to a constant independent of $\bs{\theta}$,
and 
for all $m\in \sN\cup\{0\}$,
\begin{equation}\label{eq: statistics}
\begin{aligned}
V^{\bs{\theta }, \bs{\Psi}, m}
&\coloneqq
\bigg(V_0^{-1}+
\sum_{n=1}^m
\int_0^T Z^{\bs{\theta}, \bs{\Psi}, n}_s (  Z^{\bs{\theta}, \bs{\Psi}, n}_s)^\top \d s
\bigg)^{-1}\in \sS^{d+p}_+,
\\
\hat{\bs{\theta }}^{ \bs{\Psi}, m}
&\coloneqq
\bigg(
\hat{\theta}_0
V_0^{-1}+
\sum_{n=1}^m
\Big(
 \int_0^T Z^{\bs{\theta}, \bs{\Psi}, n}_s
 (\d X^{\bs{\theta}, \bs{\Psi}, n}_s)^\top
 \Big)^\top
 \bigg)V^{\bs{\theta }, \bs{\Psi}, m}
  \in \sR^{d\t (d+p)},
\end{aligned}
\end{equation}
with $Z^{\bs{\theta}, \bs{\Psi}, n}$ defined in \eqref{eq:X_m} for all $n\in \sN$.
 A formal derivation of this argument can be found in \cite[Section 7.6.4]{liptser1977statistics} and is also related to the Kallianpur--Striebel formula in stochastic filtering (see \cite{bain2008fundamentals}).
In particular   $ \pi(\bs{\theta} | \cG^{\bs{\Psi}}_{m})=\cM\cN(\hat{\bs{\theta}}^{\bs{\Psi}, m},  I_{d}, V^{\bs{\theta }, \bs{\Psi}, m})$,
where $(\hat{\bs{\theta}}^{\bs{\Psi}, m}, V^{\bs{\theta }, \bs{\Psi}, m})$ depends only on
the policies $(\Psi_n)_{n=1}^m$.

It is clear that  
the random variables $(\hat{\bs{\theta}}^{\bs{\Psi}, m}, V^{\bs{\theta }, \bs{\Psi}, m})$ are   sufficient statistics to represent the posterior distribution $ \pi(\bs{\theta} | \cG^{\bs{\Psi}}_{m})=\cM\cN(\hat{\bs{\theta}}^{\bs{\Psi}, m},  I_{d}, V^{\bs{\theta }, \bs{\Psi}, m})$.  The following proposition iteratively  constructs  learning algorithms 
(cf.~Definition \ref{Def: Admissible RL})
based on 
deterministic feedback controls
and  the sufficient statistics.
The main step of the proof is to show 
that 
$X^{\bs{\theta}, \bs{\Psi}, m}$ in \eqref{eq:X_m}
is 
a semi-martingale 
with respect to the filtration generated by  state processes, whose details are given in Appendix \ref{sec:learning_alg}.

 \begin{Proposition}
\label{prop: PEGE as a learning}
Suppose (H.\ref{assum:bound}) holds. 
Let $\hat{\theta}_0\in \sR^{d\t (d+p)}$, 
$V_0\in \sS^{d+p}_+$,
$\hat{\bs{\theta}}^{\bs{\Psi}, 0}=\hat{\theta}_0$ and $V^{\bs{\theta }, \bs{\Psi}, 0}=V_0$.
For each $m\in \sN$, let 
$C_m\ge 0$,
let 
$\psi_m:
\sR^{d \times (d+p)} \times \sS^{d+p}_+ \times [0,T] \times \sR^d \to \sR^p$
be a measurable function
such that 
$\psi_m(\theta,S,\cdot)\in \cV_{C_m}$
for all $
(\theta,S)\in \sR^{d \times (d+p)} \times \sS^{d+p}_+$, 
let
$\Psi_m:\Om\t [0,T]\t \sR^d\to \sR^p$ be such that 
for all $(\om,t,x)\in \Om\t [0,T]\t \sR^d$,
 $\Psi_m(\om,t,x)=\psi_m(\hat{\bs{\theta}}^{\bs{\Psi}, n}(\om),V^{\bs{\theta }, \bs{\Psi}, n}(\om), t,x)$
 for some $n\le m-1$, 
 and let 
$(\hat{\bs{\theta}}^{\bs{\Psi}, m} , V^{\bs{\theta }, \bs{\Psi}, m}):\Om\to \sR^{d\t (d+p)}\t \sS^{d+p}_+$
be defined in \eqref{eq: statistics}.
 Then $\bs{\Psi} = (\Psi_m)_{m \in \sN}$ is a learning algorithm as in Definition \ref{Def: Admissible RL}.
\end{Proposition}


 We end this section by  quantifying the estimation error 
 of  $(\hat{\bs{\theta}}^{ \bs{\Psi}, m})_{m=0}^\infty$
 in terms of the number of learning episodes. 
For simplicity, we focus on the following set $\cU_C$ of  learning algorithms $\bs{\Psi}$ that are uniformly Lipschitz continuous in the $x$-variable: 
for each $C\ge 0$, 
\begin{equation}\l{eq:lipschitz_algorithm}
\cU_C \coloneqq
\left\{ 
\bs{\Psi}=(\Psi_m)_{m\in \sN} \,\middle\vert\, 
\begin{aligned}
&\textnormal{$ \bs{\Psi}$ is a learning algorithm as in Definition \ref{Def: Admissible RL} and}
\\
&
\textnormal{for all $m\in \sN$ and $\om\in \Om$, 
$\Psi_m(\om, \cdot)\in \cV_C$,
with $\cV_C$  in \eqref{eq:lipschitz_feedback}.
}
 \end{aligned}
\right\}
\end{equation}

\begin{Theorem} 
\label{Theorem: Information collapse}
Suppose  
(H.\ref{assum:bound}) holds. Let
$x_0\in \sR^d$,
$T, L\ge 0$, 
 $\hat{\theta}_0 \in \sR^{d \times (d+p)}$, $V_0 \in \sS^{d+p}_+$
  and 
  for each $\bs{\Psi} \in \cU_L$,
  let  $(\hat{\bs{\theta}}^{\bs{\Psi}, m} , V^{\bs{\theta }, \bs{\Psi}, m})_{m=0}^\infty$ be defined in \eqref{eq: statistics}.  
Then 
there exists a constant ${C}>0$
such that for all
$\bs{\Psi} \in \cU_L$, $\delta \in (0,1)$ and $m\in \sN\cap [2,\infty)$,
$$\sP\Big(\lambda_{\min}({G^{\bs{\theta}, \bs{\Psi}, m}}) \big| \hat{\bs{\theta}}^{\bs{\Psi}, m} - \bs{\theta} \big|^2 \leq {C} \big( \ln m + \ln(\tfrac{1}{\delta}) \big) \Big) \geq 1-\delta, \q \sP-a.s.$$
where 
$\lambda_{\min}(S)$ is the smallest eigenvalue of a matrix $S\in \sS^{d+p}_+$,
and 
${G^{\bs{\theta}, \bs{\Psi}, m}}=
({V^{\bs{\theta}, \bs{\Psi}, m}})^{-1}$.
\end{Theorem} 
  
The proof of Theorem \ref{Theorem: Information collapse} 
is   given 
in Section \ref{sec:proof_bayesian},
and is based on 
a combination of 
Propositions  \ref{prop: bound hat theta}
and \ref{prop:X_subgaussian}.
Note that 
the Bayesian estimator
$\hat{\bs{\theta}}^{\bs{\Psi}, m}$ in 
\eqref{eq: statistics}
is defined by using  all observed trajectories from  previous episodes,
which are not mutually independent as those for 
 least squares estimators in 
\cite{basei2021logarithmic, guo2021reinforcement}.
Hence, instead of applying concentration inequalities for independent observations, 
we establish Theorem \ref{Theorem: Information collapse} based on concentration inequalities for martingales with \textit{conditional} sub-exponential differences. 

Theorem \ref{Theorem: Information collapse}
indicates that the convergence rate of  $(\hat{\bs{\theta}}^{ \bs{\Psi}, m})_{m=0}^\infty$
depends on the growth rate of $\lambda_{\min}({G^{\bs{\theta}, \bs{\Psi}, m}})_{m\in \sN}$ with respect to the number  of episodes. 
In the  following sections, we  design a phased-based  learning algorithm such that 
  $(\lambda_{\min}({G^{\bs{\theta}, \bs{\Psi}, m}}) )_{m\in \sN}$ 
  blows  up to infinity
  sufficiently fast as $m\to \infty$.
  
\subsection{Structural assumptions for learning algorithms}
\label{sec:assumption}

As illustrated in Section \ref{sec: overview},
the following structural properties  of the control problem
\eqref{eq:lc}-\eqref{eq:lc_sde}
are essential for  the design and  analysis of  
  learning algorithms.

\begin{Assumption}
\label{Assump: Exploration strategy}
There exists $\psi^e \in \cup_{C\ge 0} \cV_C$
and $L_e\ge 0$ 
such that 
\begin{enumerate}[(1)]
\item \label{item:independence}
if  $u\in \sR^d$ and $v\in \sR^p$ satisfy
$u^\top x+v^\top \psi^e(t,x)=0$
for almost every $(t,x)\in [0,T]\t \sR^d$,
then $u$ and $v$ are zero vectors,
\item \label{item:exploration_quadratic}
for all $(t,x)\in [0,T]\t \sR^d$, 
$|f(t,x,\psi^e(t,x))|\le L_e(1+|x|^2)$
with $f$ in (H.\ref{assum:lc}).

\end{enumerate}

\end{Assumption}

\begin{Assumption}
\label{Assump: Performance gap}
There exist constants $L_\Theta,\beta>0,r\in(0,1]$
  such that
   for all $\theta_0 \in \Theta$,
   $$|J(\psi_{{\theta}}; \theta_0) - J(\psi_{{\theta_0}}; \theta_0)|\le L_\Theta|{\theta}-\theta_0|^{2r},
   \quad \fa {\theta} \in\sB_\beta(\theta_0)\coloneqq \{\theta\in \sR^{d\t (d+p)}\mid
  |\theta-\theta_0|\le \beta\},
   $$
   where
    $\psi_\theta$ is an  optimal feedback control
of \eqref{eq:lc}  and     $J(\psi;\theta_0)$ is defined  in \eqref{eq:loss_feedback} for all $\psi\in \cup_{C\ge 0}\cV_C$.

\end{Assumption}

\begin{Remark}
\label{rmk:exploration_policy}

(H.\ref{Assump: Exploration strategy}) assumes the existence of an  exploration policy $ \psi^e$, 
which will be exercised to explore the unknown system parameter.
In particular,
(H.\ref{Assump: Exploration strategy}\ref{item:independence})
ensures that 
the state and control processes 
associated with $ \psi^e$   
span the entire parameter space.
More precisely, 
let $\Lambda_{\min}:  \cup_{C \geq 0}\cV_C
\t \sR^{d\t (d+p)}
\to [0,\infty)$ 
be the function 
such that 
for all  $\psi \in \bigcup_{C \geq 0}\cV_C$
and $\theta\in \sR^{d\t (d+p)}$, 
\bb\label{eq: Information value}
\Lambda_{\min}(\psi, \theta) \coloneqq \lambda_{\min}\Bigg( \mathbb{E}^{\sP}\Bigg[\int_{0}^{T}\begin{pmatrix}
X^{{\theta}, {\psi}}_t \\ \psi\big(t, X^{{\theta }, {\psi}}_t\big) \end{pmatrix}
 \begin{pmatrix}
X^{{\theta }, {\psi}}_t \\ \psi\big(t, X^{{\theta }, {\psi}}_t\big) \end{pmatrix}^\top \d t \Bigg] \Bigg),
\ee
where 
 $(\Om,\cF, \sF, \sP)$ is a generic filtered probability space satisfying the usual condition,
$X^{\theta, \psi}\in \cS^2_{\sF}(\Om; \sR^d)$
is the  state process associated with $\theta$ and  $\psi$,
and 
$\lambda_{\min}(S)$ is the smallest eigenvalue of a symmetric positive semidefinite matrix
$S$. 
Then 
(H.\ref{Assump: Exploration strategy}\ref{item:independence}) is equivalent to the fact that 
there exists $\lambda_0>0$ such that 
 $\Lambda_{\min}( \psi^e,\theta) \geq \lambda_0$
  for all $\theta\in \Theta$
(see Lemma \ref{lemma:nondegenerate}). Hence, 
by 
\eqref{eq: statistics} and Theorem \ref{Theorem: Information collapse},
 exercising 
 $\psi^e$ increases  $\lambda_{\min}(G^{\bs\theta, \bs\Psi,m})$
 and
 subsequently results in  an improved parameter estimation
based on state trajectories 
generated by  $ \psi^e$.

(H.\ref{Assump: Exploration strategy}\ref{item:exploration_quadratic})
along with (H.\ref{assum:bound})
implies uniformly bounded exploration costs $|J(\psi^e;\theta)|\le C<\infty$ for all $\theta\in \Theta$. 
It also enables  estimating the  tail behaviour of the exploration cost 
and quantifying the regret 
for the exploration phase
of the learning algorithm.

Note that  \cite{basei2021logarithmic,guo2021reinforcement}
design RL algorithms by requiring
the optimal feedback control  associated with the unknown true parameter  
to satisfy 
(H.\ref{Assump: Exploration strategy}).
In the case with quadratic costs $f(t,x,a)=x^\top Q x+a^\top R a$ and $g(x)=0$ with $Q\in \sS^d_+$ and $R\in \sS^p_+$,
such a requirement is equivalent to the condition that 
the matrix 
$B $ in \eqref{eq:lc_sde} has  full column rank
(see \cite[Proposition 3.9]{basei2021logarithmic}).
Here we remove this requirement
and construct $\psi^e$ under more general conditions 
(see Proposition \ref{prop:construct_psi_e}).
\end{Remark}




The remaining  of this section is devoted to  verifying 
(H.\ref{Assump: Exploration strategy}) and (H.\ref{Assump: Performance gap}) for practically important LC-RL problems. 
The following proposition gives an equivalent characterisation 
of (H.\ref{Assump: Exploration strategy}), which allows for an
explicit construction of the  exploration strategy $\psi^e$.
The proof is given in Section \ref{sec: proof of assumptions}.

\begin{Proposition}
\label{prop:construct_psi_e}
Suppose (H.\ref{assum:lc}\ref{item:f0R}) 
holds,
and let $\operatorname{dom}(h)=
\{a\in \sR^p\mid h(a)<\infty\}$.
Then (H.\ref{Assump: Exploration strategy}) holds if and only if 
 $\operatorname{dom}(h)$ contains
linearly independent vectors $a_1,\ldots,a_p$.
In this case, 
for each partition $\{0=t_0<t_1<\ldots<t_p=T\}$, 
$\psi^e(t,x)=a_k$ for all 
$(t,x)\in [t_{k-1},t_k)\in \sR^d$ and $k=1,\ldots,p$ 
satisfies (H.\ref{Assump: Exploration strategy}).

\end{Proposition}

To verify  (H.\ref{Assump: Performance gap}), 
we first recall the following first-order performance gap 
under 
(H.\ref{assum:lc}), as  shown in \cite[Theorem 2.7]{guo2021reinforcement}.

\begin{Proposition}
\label{prop:gap_nonsmooth}
 Suppose (H.\ref{assum:lc}) and (H.\ref{assum:bound}) hold.
 Then for any $\beta>0$, 
 there exists $L_\Theta>0$ such that 
 (H.\ref{Assump: Performance gap}) holds with $r=1/2$
 and the constants $L_\Theta, \beta$.
\end{Proposition}

The fact that 
 the cost function $h$ 
is merely lower semicontinuous
restricts 
the model misspecification error
 to scale linearly in terms of 
 the magnitude of parameter perturbations.
By assuming a Lipschitz continuously differentiablity  of the cost function,
the following theorem
 improves the linear dependence in Proposition 
 \ref{prop:gap_nonsmooth} to a quadratic dependence,
which generalises  the well-known quadratic performance gap 
with quadratic costs
 (see e.g., \cite{osband2014model,
 basei2021logarithmic}).

 \begin{Theorem}\label{thm:performance_gap_smooth}
Suppose (H.\ref{assum:lc})  and (H.\ref{assum:bound})  hold, and $h\equiv 0$. Then for any $\beta>0$, 
 there exists $L_\Theta>0$ such that 
 (H.\ref{Assump: Performance gap}) holds with $r=1$
 and the constants $L_\Theta, \beta$.

\end{Theorem}

The proof of 
 Theorem \ref{thm:performance_gap_smooth}
 is  given in Section \ref{sec: proof of assumptions}.
 The main step is to show 
 the functional 
  $J(\cdot; \theta_0): \cH^2_\sF(\Om;\sR^p)\to \sR\cup\{\infty\}$ defined  in \eqref{eq:lc} (with $\theta =\theta_0$) is convex and has a Lipschitz continuous derivative,  which implies that $J(\a;{\theta}_0)-J(\a^{\theta_0};{\theta_0})\le C\|\a-\a^{\theta_0}\|^2_{\cH^2}$
 for all $\a\in  \cH^2_\sF(\Om;\sR^p)$,
 with   $\a^{\theta_0}\in \cH^2_\sF(\Om;\sR^p)$ being the minimiser of $J(\cdot; \theta_0)$. 
Unfortunately, 
such an argument in general cannot be applied to nonsmooth cost functions. 
 For example,  as  the entropy 
function $ h_{\textrm{en}}$ in  \eqref{eq:entropy} has unbounded derivatives, 
one can easily show that 
 $L^2(\Om;\sR^p)\ni \a \mapsto 
 \sE[|(\nabla h_{\textrm{en}})(\a)|]\in [0,\infty]$
is discontinuous  at  the minimiser 
 $\a^\star=\argmin _{\a\in L^2(\Om;\sR^p)}\sE[h_{\textrm{en}}(\a)]$.

In the following, we overcome the  difficulty by 
establishing a \textit{local regularity} of 
the  functional $J(\cdot;\theta_0)$ in \eqref{eq:lc}
along optimal feedback controls $\psi_\theta$ of  estimated models,
and recover the optimal quadratic performance gap
for entropy-regularised cost functions.
More precisely, we prove 
for  the  cost functions defined in 
\eqref{eq:regularized_cost},
the map $\sR^{d\t (d+p)}\ni \theta\mapsto J(\psi_\theta;\theta_0)\in \sR$ is $C^2$  for all $\theta_0 \in \sR^{d\t (d+p)}$,
 if $\bar{f}_0$ and $g$ are sufficiently smooth. 

The proof 
   is  given in Section \ref{sec: proof of assumptions}.
It relies on first 
characterising an optimal feedback control 
$\psi_\theta$ in terms of solutions to a nonlinear HJB equation,
and then 
establishing  high-order regularity  of the function 
$ \sR^{d\t (d+p)}\t \sR^d\ni (\theta,x)\mapsto\psi_\theta(t,x)\in \sR^p$ via a probabilistic argument. 
Note that 
the convexity condition (H.\ref{assum:lc}) on $f$ and 
$g$ will not be imposed in Theorem \ref{thm:performance_gap_entropy}.

\begin{Theorem}\label{thm:performance_gap_entropy}
Suppose   (H.\ref{assum:bound}) holds,  $x_0\in \sR^d$,
$T>0$, and
$f:[0,T]\t \sR^d\t \sR^p\to \sR\cup \{\infty\}$
be of the form 
\bb\label{eq:entropy_regularized_proof} 
 f(t,x,a)\coloneqq  \bar{f}_0(t,x)^\top a + h_{\textrm{en}}(a),
\q \fa (t,x,a)\in [0,T]\t \sR^d\t \sR^p,
\ee
with 
a continuous function $\bar{f}_0:[0,T]\t \sR^d\to \sR^p$,  
and the entropy function 
$h_{\textrm{en}}:\sR^p\to \sR\cup\{\infty\}$
defined in \eqref{eq:entropy}.
Assume further that 
  for all $t\in [0,T]$, 
$\bar{f}_0(t,\cdot)\in C^4( \sR^d)$ 
and $g\in C^4(\sR^d)$
with  bounded derivatives  uniformly in $t$.
Then for all $\theta\in \sR^{d\t (d+p)}$,
the control problem \eqref{eq:lc} (with parameter 
$\theta$) admits an optimal feedback control $\psi_\theta\in \cup_{C\ge 0}\cV_C$. Moreover, 
 for any $\beta>0$, 
 there exists $L_\Theta>0$ such that 
 (H.\ref{Assump: Performance gap}) holds with $r=1$
 and the constants $L_\Theta, \beta$.

\end{Theorem}

\subsection{Phased Exploration and Greedy Exploitation algorithm and its regret bound}
\label{sec:algorithm_regret}


Based on  (H.\ref{Assump: Exploration strategy})
and (H.\ref{Assump: Performance gap}),
we  propose a 
Phased Exploration and Greedy Exploitation (PEGE)
  algorithm
  which 
  achieves  sublinear regrets with high probability  and in expectation.
  The algorithm  alternates
between exploration and exploitation phases, and 
 extends the PEGE  algorithm in \cite{rusmevichientong2010linearly}
for discrete-time linear bandit problems
to the present setting with continuous-time linear-convex models.
For technical reasons, we  
focus on   learning algorithms of the class $\cup_{C\ge 0}\cU_C$ (cf.~\eqref{eq:lipschitz_algorithm})
by truncating the
maximum a posteriori (MAP) estimates
$(\hat{\bs{\theta}}^{\bs{\Psi}, m})_{m\in \sN}$ on a set compactly containing $\Theta$.

\begin{Definition}
\label{def: Truncate function}
A measurable function $\rho:  \sR^{d\t (d+p)} \times \sS^{d+p}_+ \to \sR^{d\t (d+p)}$ is called  a  truncation 
to compact neighborhoods of 
$\Theta$ (simply referred to as a truncation function)
if there exists 
a bounded  set
$\cK\subset \sR^{d\t (d+p)}$ such that 
$\operatorname{range}(\rho)= \cK$,
$\operatorname{cl}(\Theta)\subset
\operatorname{int}(\cK)$,
and $\rho(\theta,V)=\theta$ for all $\theta\in \cK$ and 
$V\in \sS^{d+p}_+ $.
For  simplicity, we denote by 
 $\rho_\cK$ a  truncation function $\rho$ with range $\cK$.
\end{Definition}

In practice,  one can construct  a truncation function 
by  fixing  
 a compact subset 
$\cK\subset \sR^{d\t (d+p)}$
with
$\operatorname{cl}(\Theta)\subset
\operatorname{int}(\cK)$
and considering 
$\rho: 
 \sR^{d\t (d+p)} \times \sS^{d+p}_+
 \ni (\theta, V)
\mapsto 
\theta \bs{1}_\cK (\theta)
+ {\theta}_0
 \bs{1}_{\cK^c}(\theta )\in \cK$, with some  ${\theta}_0 \in \cK$. 
Alternatively, one may consider a  measurable function $\rho$ such that 
$$
\sR^{d\t (d+p)} 
\times \sS^{d+p}_+
\ni (\theta, V)
\mapsto 
\rho(\theta, V) \in
\argmin_{\theta' \in \cK}  \tr\big((\theta'-\theta) V^{-1}  (\theta'-\theta)^\top\big),
$$
which corresponds to the MAP estimator truncated by  $\cK$ (cf.~\eqref{eq:posterior}).

We now proceed to describing the PEGE algorithm based on truncated MAP estimates.
The algorithm 
is initialised  with $\hat{\theta}_0\in \sR^{d\t (d+p)}$, $V_0\in \sS^{d+p}_+$ and 
a truncation function $\rho$.
Then it 
operates in cycles, and  each cycle consists of exploration and exploitation
phases. 
In the exploration phase of 
the $k$-th cycle with $k\in \sN$,
   we exercise the exploration  strategy $\psi^e$  
in (H.\ref{Assump: Exploration strategy}) for one episode,
and  obtain the current estimate $\tilde{\bs{\theta}}_m$ of $\bs{\theta}$ by
\bb\label{eq:project}
\tilde{\bs{\theta}}_m
\coloneqq 
\rho(\hat{\bs{\theta}}^{\bs{\Psi}, m} ,
V^{\bs{\theta},\bs{\Psi}, m} ),
\ee
where 
$(\hat{\bs{\theta}}^{\bs{\Psi}, m} ,
V^{\bs{\theta},\bs{\Psi}, m})$
are defined in 
  \eqref{eq: statistics}.
 During   the  exploitation phase of the $k$-th cycle, 
we exercise a sequence of greedy policies for  $\mathfrak{m}(k)$ 
consecutive episodes with some prescribed $\mathfrak{m}(k)\in \sN$.
For each exploitation episode, 
  based on the current estimate $\tilde{\bs{\theta}}_m$ of $\bs{\theta}$, 
we 
  execute the optimal feedback control 
  $\psi_{\tilde{\bs{\theta}}_{m}}$ 
    (cf.~Proposition \ref{prop:existence_fb})
for    $\mathfrak{m}(k)$ 
consecutive episodes with some prescribed $\mathfrak{m}(k)\in \sN$, 
  and possibly update the estimate $\tilde{\bs{\theta}}_{m+1}$ 
by using  the new observations.

  The PEGE algorithm  is summarised as follows.
In the sequel, we denote by $\bs{\Psi}^{PEGE}$
the sequence of strategies generated by 
Algorithm \ref{Alg: PEGE}.

\begin{algorithm}[H]
\label{Alg: PEGE}
\DontPrintSemicolon
\SetAlgoLined

  \KwInput{
  $\hat{\theta}_0\in \sR^{d\t (d+p)}$, $V_0\in \sS^{d+p}_+$, 
  a truncation function $\rho$,
and 
$\mathfrak{m} : \sN \to \sN$.
  }

 {
 Initialise $m=0$,
 $\hat{\bs{\theta }}^{ \bs{\Psi}, 0}=\hat{\theta}_0$ and 
 $V^{\bs{\theta }, \bs{\Psi}, 0}=V_0$.
 }\;
 
 \For{$k = 1, 2,\ldots$}
 {
 {Execute the exploration policy $\psi^e$ for one episode,
  and   $m \gets m+1$.}\;
 
 {
 Update 
 $(\hat{\bs{\theta}}^{\bs{\Psi}, m} ,
V^{\bs{\theta},\bs{\Psi}, m})$
via 
  \eqref{eq: statistics}, 
  $\tilde{\bs{\theta}}_m$ 
  via \eqref{eq:project},
and set $\ol{\bs{\theta}} = \tilde{\bs{\theta}}_m$}.\;
  \For{$l = 1, 2,\ldots,  \mathfrak{m}(k)$}
  {
	  {Execute the greedy policy $\psi_{\ol{\bs{\theta}}}$
	   in Proposition \ref{prop:existence_fb}
	   for one episode,
	   and  $m\gets m+1$. }\;
	{ 
	[Optional: 
	update 
 $(\hat{\bs{\theta}}^{\bs{\Psi}, m} ,
V^{\bs{\theta},\bs{\Psi}, m})$
via 
  \eqref{eq: statistics}, 
  $\tilde{\bs{\theta}}_m$ 
  via \eqref{eq:project},
and set $\ol{\bs{\theta}} = \tilde{\bs{\theta}}_m$.
] }
\label{optional}
\;
}
 }
 \caption{PEGE algorithm}
\end{algorithm}

\begin{Remark}

By the regularity of  
$\sR^{d\t (d+p)}
\t [0,T]\t \sR^d
\ni (\theta,t,x)\mapsto
\psi_\theta(t,x)\in \sR^p$
 (see Proposition 
\ref{prop:existence_fb}),
the truncation function $\rho$
and Proposition 
 \ref{prop: PEGE as a learning},
 $\bs{\Psi}^{PEGE}$
 (with or without Step \ref{optional})
is a learning algorithm
 as in Definition \ref{Def: Admissible RL}.
%

Observe that 
the exploration phase of 
Algorithm \ref{Alg: PEGE} does not depend on 
   the  confidence parameter $\delta$ in the high-probability regret bound \eqref{eq:high_prob_regret}
(cf.~the least squares algorithms in 
\cite{mania2019certainty,dean2020sample,simchowitz2020naive}).
Consequently, 
   Algorithm \ref{Alg: PEGE}  achieves a sublinear regret 
   both with high probability and in expectation. 


As we shall see soon,
Algorithm  \ref{Alg: PEGE}
admits the same regret order regardless of whether Step \ref{optional} is employed. 
Hence one may omit Step \ref{optional} to save the computational costs 
of matrix inversion in 
  \eqref{eq: statistics} 
and of control updates,
especially for 
high-dimensional  problems with large $d$ and $p$.
However, 
 including Step \ref{optional} in the algorithm incorporates all   available information in each exploitation episode, 
and hence may lead to more sample-efficient learning algorithms (see e.g., the posterior sampling algorithms in  \cite{osband2013more,osband2014model,abeille2018improved}).

\end{Remark}

Now we state the main result of this section, which shows that the regret of
Algorithm \ref{Alg: PEGE} 
(without Step \ref{optional})
grows sublinearly with respect to the number of episodes.
For the sake of presentation, we include a sketched proof at the end of this section 
and present the detailed arguments in
Section \ref{sec:proof_regret_bound}.

\begin{Theorem}
\label{Theorem: regret}
Suppose  (H.\ref{assum:lc}), (H.\ref{assum:bound}), (H.\ref{Assump: Exploration strategy}) and (H.\ref{Assump: Performance gap}) hold.
Let 
  $\hat{\theta}_0\in \sR^{d\t (d+p)}$, $V_0\in \sS^{d+p}_+$, 
 $\rho$ be a truncation function
 and 
$\mathfrak{m} : \sN \to \sN$
be such that 
$\mathfrak{m}(k) = \lf k^r \rf$ for all $k\in\sN$,
 with 
  $r\in (0,1]$ being the same as in (H.\ref{Assump: Performance gap}).
 Then there exists a constant $C \geq 0$  such that  for all $\delta \in (0,1)$, the regret \eqref{eq: Bayesian regret} of Algorithm   \ref{Alg: PEGE} 
 (without Step \ref{optional})
 satisfies
  with probability at least $1-\delta$,
\bb\label{eq:high_prob_regret}
{\cR}(N,  \bs{\Psi}^{PEGE},  \bs{\theta}) 
 \leq C 
 \Big(
 N^{\frac{1}{1+r}} \big( (\ln N)^r + \big( \ln(\tfrac{1}{\delta}) \big)^{r} \big)
 +\big( \ln(\tfrac{1}{\delta}) \big)^{1+r}
 \Big)
  \q
\fa N\in \sN\cap [2,\infty).
\ee
Consequently, there exists a constant $C \geq 0$  such that
$\sE[{\cR}(N,  \bs{\Psi}^{PEGE},  \bs{\theta})]  \leq C N^{\frac{1}{1+r}} (\ln N)^r $ for all $N\in \sN\cap [2,\infty)$.
\end{Theorem}

\begin{Remark}
As mentioned above, 
Algorithm \ref{Alg: PEGE} 
with Step \ref{optional}
enjoys the 
the same regret bounds. 
The proof follows essentially   the steps in the arguments for
Theorem \ref{Theorem: regret}.
The only difference is that 
due to a more frequent control update,
we sum over the regret of each exploitation episode
in \eqref{eq:regret_exploitatoin},
instead of the regret for each cycle.

\end{Remark}

In the case where the greedy policies admit a self-exploration property and improve the accuracy of  parameter estimation as in \cite{basei2021logarithmic,guo2021reinforcement},  
 one can increase the number of exploitation episodes    and obtain an improved  regret bound. 
 Note that 
 if (H.\ref{Assump: Performance gap}) holds with
 for $r=1/2$ (cf.~Proposition \ref{prop:gap_nonsmooth}), 
 then
 Theorem \ref{Theorem: regret_self explore} recovers the sublinear regret $\cO(\sqrt{N})$ 
 in \cite{guo2021reinforcement} for general (nonsmooth) LC-RL problems, 
 while  if (H.\ref{Assump: Performance gap}) holds with
 $r=1$ (cf.~Theorems \ref{thm:performance_gap_smooth}
and 
\ref{thm:performance_gap_entropy}),  
 Theorem \ref{Theorem: regret_self explore} generalises
the \textit{logarithmic}
  regret bound for LQ problems 
in \cite{basei2021logarithmic} 
to smooth convex costs  and entropy-regularised costs.

To simplify the presentation, we  assume that
 Step \ref{optional} of Algorithm \ref{Alg: PEGE} 
 is omitted,
 and estimate the regret \eqref{eq: Bayesian regret}  in expectation.
 However,  a similar high-probability regret 
 of expected costs can be found in the proof,
 whose details are given  
 in Section \ref{sec:proof_regret_bound}.

\begin{Theorem}
\label{Theorem: regret_self explore}
Suppose  (H.\ref{assum:lc}), (H.\ref{assum:bound}), (H.\ref{Assump: Exploration strategy}) and (H.\ref{Assump: Performance gap}) hold.
Let 
  $\hat{\theta}_0 \in \sR^{d\t (d+p)}$, $V_0\in \sS^{d+p}_+$ and
  $\rho_\cK$ be a truncation function.
Assume further  that there exists a constant $\lambda_0 > 0$ such that  $\Lambda_{\min}( \psi_{\theta'},\theta) \geq \lambda_0$
for all $\theta \in \Theta$ and $\theta' \in \cK$,
with the function 
$\Lambda_{\min}:  \cup_{C \geq 0}\cV_C
\t \sR^{d\t (d+p)}
\to [0,\infty)$ 
defined in 
\eqref{eq: Information value}.

Then there exists a constant $C \geq 0$  such that
if one sets $\mathfrak{m}(k)= 2^{k} $ for all $k\in \sN$, 
then 
 the regret \eqref{eq: Bayesian regret} of Algorithm   \ref{Alg: PEGE} satisfies
 for all $N\in \sN\cap [2,\infty)$,
 \bb\label{eq:high_prob_regret_self_explore}
\sE[{\cR}(N,  \bs{\Psi}^{PEGE},  \bs{\theta})]  \leq 
\begin{cases}
C  N^{1-r}  \big(
\ln N\big)^r,
& r\in (0,1),
\\
C \big( \ln N\big)^2,
& r=1,
\end{cases}
\ee
  with the constant  $r\in (0,1]$ in (H.\ref{Assump: Performance gap}).

\end{Theorem}
 
 \begin{Remark}
One can easily deduce from  Proposition \ref{prop:existence_fb} and 
 the stability of \eqref{eq:lc_sde_fb},
$\sR^{d\t (d+p)}\t \sR^{d\t (d+p)}\ni (\theta',\theta)\mapsto \Lambda_{\min}( \psi_{\theta'},\theta)\in [0,\infty)$ is continuous.
Hence,  
 a sufficient condition to ensure
$\Lambda_{\min}( \psi_{\theta'},\theta) \geq \lambda_0$
for all $\theta\in \Theta$ and $\theta'\in \cK$
is that for all $\theta'\in \operatorname{cl}(\cK)$,
the functions 
$(t,x)\mapsto \psi_{\theta'}(t,x)$ 
and $(t,x)\mapsto x$ 
are linearly independent as in 
(H.\ref{Assump: Exploration strategy}\ref{item:independence})
(cf.~Lemma \ref{lemma:nondegenerate}).
{Moreover, 
thanks to the exploration episodes, 
the same regret bound holds under a weaker condition: 
there exist constants $\lambda_0,\eta>0$ such that 
$\Lambda_{\min}( \psi_{\theta'},\theta) \geq \lambda_0$ for all   $\theta \in \Theta$ and $\theta' \in \sB_\eta (\Theta)$, 
with
$\sB_\eta (\Theta)\coloneqq \{{\theta}'\in \sR^{d\t (d+p)}\mid  
\textnormal{$\exists \theta\in \Theta$ s.t.~$|{\theta}'-\theta|\le \eta$
}\}$.
}
 \end{Remark}

\paragraph{Sketched proofs of Theorems \ref{Theorem: regret}
and \ref{Theorem: regret_self explore}.}

Here we outline the key steps of the proofs of Theorems \ref{Theorem: regret}
and \ref{Theorem: regret_self explore}.
For notational simplicity, we omit the dependence on ``PEGE" in the superscripts. 
By Proposition \ref{prop:existence_fb},  ${\cR}(N,  \bs{\Psi},  \bs{\theta}) $
admits the following decomposition:
\begin{align}
\label{eq: regret_decomposition_sketch}
\begin{split}
\cR(N,  \bs{\Psi},\bs{\theta})
&= \sum_{m=1}^N \Big( \ell_m(\bs{\Psi},\bs{\theta})
 - J(\Psi_m ; \bs{\theta}) \Big) 
+
\sum_{m \in [1,N] \cap \cE^{\bs{\Psi}}} \big( J(\psi^e ; \bs{\theta}) - 
J(\psi_{\bs{\theta}}; \bs{\theta})
 \big) 
 \\
 &+  \sum_{m \in [1,N] \setminus \cE^{\bs{\Psi}} } \big( J(\psi_{\tilde{\bs{\theta}}_{m^e}} ; \bs{\theta}) - J(\psi_{\bs{\theta}}; \bs{\theta}) \big), 
 \end{split}
\end{align}
where 
$\cE^{\bs{\Psi}}$ is the collection of exploration episodes and 
 $m^e$ is the last exploration episode before the $m$-th episode.

We now estimate the  terms on the right-hand side of \eqref{eq: regret_decomposition_sketch}.
The truncation of the MAP estimate and 
Proposition \ref{prop:existence_fb}
ensure that 
$\Psi\in \cU_C$ for some $C\ge 0$, based on which we prove  the first term in \eqref{eq: regret_decomposition_sketch} is a martingale with \textit{conditional} sub-exponential differences. 
Hence it   vanishes under expectation and can be bounded by $\cO(\sqrt{N} \ln N)$ with  high probability  due to a general Bernstein’s inequality.
By  (H.\ref{assum:bound}) and (H.\ref{Assump: Exploration strategy}), the second term in \eqref{eq: regret_decomposition_sketch} is of the order $\cO(\kappa(N))$,
where $\kappa(m)$ is   the total number of exploration
episodes up to the $m$-th episode.

The third term in  \eqref{eq: regret_decomposition_sketch} relies on  the accuracy of the MAP estimate $(\hat{\bs{\theta}}^{\bs{\Psi}, m})_{m\in \sN}$. For Theorem \ref{Theorem: regret}, we prove under  (H.\ref{Assump: Exploration strategy}) that there exists $\lambda_0>0$ such that 
\begin{align*}
\lambda_{\min}((V^{\bs{\theta }, \bs{\Psi}, m})^{-1}) \geq  \lambda_{\min}\Bigg( \sum_{n \in [1,m] \cap \cE^{\bs{\Psi}}} \int_{0}^T Z^{\bs{\theta}, \bs{\Psi}, n}_t (Z^{\bs{\theta}, \bs{\Psi}, n}_t)^\top \d t\Bigg) \geq \kappa(m) \lambda_0,
\end{align*}
 which along with 
 Theorem \ref{Theorem: Information collapse} leads to the estimate that  $\big| \hat{\bs{\theta}}^{\bs{\Psi}, m} - \bs{\theta} \big|^2 =  \cO( (\kappa(m))^{-1}\ln m)$ with high probability. 
Then, by exploiting properties of the truncation function $\rho$ and (H.\ref{Assump: Performance gap}), 
we prove that
$\tilde{\bs{\theta}}^{\bs{\Psi},  m}=\hat{\bs{\theta}}^{\bs{\Psi}, m}$ for all large $m$
and   
$J(\psi_{\tilde{\bs{\theta}}_{m^e}} ; \bs{\theta}) - J(\psi_{\bs{\theta}}; \bs{\theta})$ admits a high probability bound 
$\cO( (\kappa(m^e))^{-r}(\ln m^e)^r)$. 
 We then quantify  the relation between $\kappa(m), m^e$ and $m$ based on the choice of   $\mathfrak{m}$,  and  estimate the precise regret of Algorithm \ref{Alg: PEGE}.
 
For Theorem \ref{Theorem: regret_self explore}, by exploiting the self-exploration property, we establish with high probability that  
$\lambda_{\min}((V^{\bs{\theta }, \bs{\Psi}, m})^{-1})\ge m\lambda_0$ for sufficiently large $m$. This along with  Theorem \ref{Theorem: Information collapse},
the truncation function $\rho$ and (H.\ref{Assump: Performance gap}) proves  $J(\psi_{\tilde{\bs{\theta}}_{m^e}} ; \bs{\theta}) -J(\psi_{\bs{\theta}}; \bs{\theta})$ is of the magnitude   
$\cO( (m^e)^{-r}(\ln m^e)^r)$. 
Consequently, we choose 
$\mathfrak{m}(k)= 2^{k} $ to emphasise exploitation and obtain   improved regret bounds. 

\section{Proof of 
Theorem \ref{Theorem: Information collapse}}
\label{sec: bayesian}


This section analyses the accuracy of the MAP estimate 
 $(\hat{\bs{\theta}}^{ \bs{\Psi}, m})_{m=0}^\infty$
 associated with a learning algorithm $
 \bs{\Psi}$
(cf.~\eqref{eq: statistics}). 
The definition of 
 a learning algorithm $\bs{\Psi}$ implies that 
$\hat{\bs{\theta }}^{ \bs{\Psi}, m} $
is in general defined with 
correlated observations 
$(X^{\bs{\theta},\bs{\Psi}, n})_{n=1}^m$. 
Hence one cannot 
analyse $\hat{\bs{\theta }}^{ \bs{\Psi}, m} $
based on concentration inequalities
for independent  random variables as in 
\cite{basei2021logarithmic, guo2021reinforcement}.
In the subsequent analysis, we overcome this difficulty by introducing 
 a notation of conditional sub-exponential random variables
 and studying its associated concentration inequalities.

\subsection{Concentration inequality for conditional sub-exponential random variables}

We start by defining conditional  sub-Gaussian and sub-exponential random variables with a  general $\sigma$-algebra.
It is well-known that the  sub-Gaussianity/sub-exponentiality of a random variable can be equivalently characterised by the finiteness of its corresponding Orlicz norms (see e.g., \cite{vershynin2018high, wainwright2019high}).
The notion of Orlicz norms 
allows us to quantify  the precise tail behaviour of a random variable by 
establishing an upper bound of its Orlicz norm,
which is particularly important for continuous-time estimators defined via  integrals of stochastic processes; see \eqref{eq: statistics} and the least-squares estimators in \cite{basei2021logarithmic, guo2021reinforcement}. 

Motivated by the above applications, 
we first introduce    the precise  definition of  
a  conditional  Orlicz norm with respect to a  general $\sigma$-algebra,
as 
we are not aware of a standard notation in the existing literature.

\begin{Definition} 
\label{definition:conditional_orlicz}
 Let  $(\Omega, \cF, \sP)$ be a probability space.  For every $q\in [1,2]$ and 
$\sigma$-algebra  $\cG \subseteq \cF$,
we define 
  the  $(q,\cG)$-Orlicz norm 
  $\|\cdot\|_{q,\cG}:\Om\to [0,\infty]$ 
such that for any  random variable 
$X:\Om\to \sR$,
$$
\|X\|_{q,\cG}\coloneqq  \mathcal{G}\text{-}\essinf \big\{ Y \in L^0(\cG ; (0, \infty))\; \big| \; 
\sE\left[\exp (|X|^q/Y^q) \big| \cG\right]\le 2\big\},\footnotemark
$$
where  $L^0(\cG ; (0, \infty))$ is the set of all
$\cG/\cB((0,\infty))$-measurable functions
$Y:\Om\to (0,\infty)$. 
\end{Definition} 
 
  \footnotetext{
	Let $(\Omega, \cG, \mathbb{P})$ be a probability space
	and $\cY$ be  a family of $\mathcal{G}/\cB(\sR)$-measurable functions.
Then there exists a 	 function $Z:\Om\to [-\infty,\infty]$,
called  the $\mathcal{G}$-{essential infimum} of 
	$\mathcal{Y}$ and denoted by 
	$Z= \mathcal{G}\text{-}\essinf \cY$,
	 such that 
	\begin{enumerate}[(1)]
		\item\label{item:lower_bound}
		$Z$ is 
		$\cG/\cB([-\infty,\infty])$-measurable and  
		 $Z \leq Y$ $\; \mathbb{P}$-a.s. for all $Y \in \mathcal{Y}$,
		\item 
		$Z \ge Z'$ 
		$\mathbb{P}$-a.s.
		for all  functions $Z'$ satisfying property
		\ref{item:lower_bound}.
		\end{enumerate}	
 	If $\cY$ is directed downwards, that is for $Y, Y'\in \cY$, 
there exists  $\tilde{Y}\in \cY$ such that $\tilde{Y}\le \min (Y,Y')$, then there exists a decreasing
sequence $(Y_n)_{n\in \sN}\subset \cY$ such that $\lim_{n\to \infty}Y_n=\mathcal{G}\text{-}\essinf \cY $ $\sP$-a.s.~(see e.g., \cite[Theorem 1.3.40]{cohen2015stochastic}).

}

Definition \ref{definition:conditional_orlicz} is
a  natural extension of  the classical Orlicz norms in 
\cite{vershynin2018high, wainwright2019high} with $\cG=\{\emptyset, \Om\}$. 
Compared with  the classical Orlicz norm,
for any given random variable $X$,
 $\|X\|_{q,\cG}$ is a $\cG$-measurable function instead of a real number. 
 The following lemma extends some  basic properties   of the classical Orlicz norm to the conditional Orlicz norm.
The proof  follows directly from 
Definition \ref{definition:conditional_orlicz},
the monotonicity  and  convexity of 
$[0,\infty)\ni  x \mapsto \exp(x^q)\in \sR$   
for all $q \in [1,2]$,
and hence is omitted.

\begin{Lemma}
\label{lemma:norm_property}
 Let  $(\Omega, \cF, \sP)$ be a probability space.
 Then   for all 
$q\in [1,2]$,  $\sigma$-algebra $\cG\subseteq \cF$, random variables $X,Y:\Om\to \sR$ and $c\in \sR$, 
\begin{align}
&
 \textnormal{if $X$ is $\cG$-measurable, then
 $\|X\|_{q,\cG}\le (\ln 2)^{-1/q} |X|$,}
 \label{eq:bounded}
\\
& \|cX\|_{q,\cG}=|c| \|X\|_{q,\cG},
\q 
\|X+Y\|_{q,\cG}\le \|X\|_{q,\cG}+\|Y\|_{q,\cG},
\label{eq:triangle}
\\
&
\textnormal{if  $ |X|\le |Y|$,
then
$\|X\|_{q,\cG}\le  \|Y\|_{q,\cG}$,}
\label{eq:monotone}
\\
&\|X-\sE[ X \mid \cG]\|_{q,\cG}\le 2\|X\|_{q,\cG}.
\label{eq:centering}
\end{align}
\end{Lemma}

The following proposition establishes the relation between the conditional tail behaviour  and the 
$(q,\cG)$-Orlicz norms of the random variable. 
The proof essentially follows from the lines of 
\cite[Propositions 2.5.2   and 2.7.1]{vershynin2018high}, along with Definition \ref{definition:conditional_orlicz} and Lemma \ref{lemma:norm_property},
and hence is omitted.
 
 \begin{Proposition}
 \label{prop:sub_gaussian_exponential}
 Let  $(\Omega, \cF, \sP)$ be a probability space.
 Then for all 
$q\in [1,2]$,  $\sigma$-algebra $\cG\subseteq \cF$ and integrable random variable $X:\Om\to \sR$,
\begin{enumerate}[(1)]
\item
$\|X-\sE[X\mid \cG]\|_{2,\cG}\le C_1$
 if and only if 
\begin{align}
\sE[\exp(\gamma (X-\sE[X\mid \cG]))\mid \cG]\le \exp(C'_1\gamma^2),
\q \fa \gamma\in \sR,
\label{eq:subGaussian}
\end{align}
\item 
$\|X-\sE[X\mid \cG]\|_{1,\cG}\le C_1$
 if and only if 
\begin{align}
\sE[\exp(\gamma (X-\sE[X\mid \cG]))\mid \cG]\le \exp\big((C'_1\gamma)^2\big),
\q \fa |\gamma|\le 1/C'_1,
\label{eq:subexponential}
\end{align}
\end{enumerate}
where  $C_1,C_1'\ge 0$ 
differ by 
an absolute multiplicative factor.

\end{Proposition}

In the sequel, we say 
an integrable random variables $X:\Om\to \sR$ is 
$\cG$-sub-Gaussian (resp.~$\cG$-sub-exponential)
if it satisfies \eqref{eq:subGaussian} (resp.~\eqref{eq:subexponential})  with a constant ${C}'_1\ge 0$.
By \eqref{eq:centering} and Proposition \ref{prop:sub_gaussian_exponential}, 
$X$ is $\cG$-sub-Gaussian if $\|X\|_{2,\cG}\le C$ $\mathbb{P}$-a.s.~for some $C\ge 0$,
and is 
$\cG$-sub-exponential 
if $\|X\|_{1,\cG}\le C$ $\mathbb{P}$-a.s.~for some $C\ge 0$.
 
The next proposition 
 estimates the conditional Orlicz norm of deterministic integrals of stochastic processes.

\begin{Proposition}
\label{prop:subguassian_subexponential}
 Let  $(\Omega, \cF, \sP)$ be a probability space.
  Then 
for all   
measurable  functions
  $X:\Om\t [0,T]\to \sR^d$, 
  all  
  measurable functions  $\psi:\Om\t [0,T]\t \sR^d\to \sR$
and   $\sigma$-algebras  $\cG \subseteq \cF$,
 $$
 \left\|\int_0^T \psi(\cdot,t,X_t)\, \d t\right\|_{1,\cG}
\le  \|\psi\|_{2,\infty}
\bigg(\frac{T}{\ln 2}+\bigg\|\left(\int_0^T|X_t|^2\, \d t
\right)^{\frac{1}{2}}\bigg\|^2_{2,\cG}\bigg),
 $$ 
with  $ \|\psi\|_{2,\infty}=\sup_{(\om,t,x)\in \Om\t [0,T]\t \sR^d}\frac{|\psi(\om,t,x)|}{1+|x|^2}$.
  
\end{Proposition}

\begin{proof}
As  $|\int_0^T \psi(t,X_t)\, \d t|\le  \|\psi\|_{2,\infty}(T+\int_0^T|X_t|^2\, \d t)$, 
by
\eqref{eq:triangle} and \eqref{eq:monotone}, 
$\|\int_0^T \psi(t,X_t)\, \d t\|_{1,\cG}
\le  \|\psi\|_{2,\infty}(\|T\|_{1,\cG}+\|\int_0^T|X_t|^2\, \d t\|_{1,\cG})$.
A direct computation shows that $\|T\|_{1,\cG}\le T/\ln 2$. We then show  
$\|\int_0^T|X_t|^2\, \d t\|_{1,\cG}
\le\|(\int_0^T|X_t|^2\, \d t)^{\frac{1}{2}}\|^2_{2,\cG}$.
By the definition of $\|\cdot\|_{2,\cG}$,
there exists
 $(Y_n)_{n\in \sN}\subset L^0(\cG; (0,\infty))$ such that  
$\lim_{n\to \infty} Y_n=\|(\int_0^T|X_t|^2\, \d t)^{\frac{1}{2}}\|_{2,\cG}$
and 
$\sE[\exp (\int_0^T|X_t|^2\, \d t/Y^2_n \mid \cG]\le 2$
for all $n$ (see   \cite[Theorem 1.3.40]{cohen2015stochastic}). 
This implies that $\|\int_0^T|X_t|^2\, \d t\|_{1,\cG}\le Y_n^2$ for all $n$. Passing $n$ to infinity leads to the desired estimate. 
\end{proof}

We end this section with
a general Bernstein-type bound for a sub-exponential martingale difference sequence,
whose proof follows from 
\cite[Theorem 2.19]{wainwright2019high}
and \eqref{eq:subexponential}.

\begin{Proposition}
\label{prop: concentration inequality}
Let $(\Omega,  \cF, \sP)$ be a   probability space,
 $\sF=\{\cF_n\}_{n=0}^\infty\subset \cF$
be a 
filtration, 
 $C\ge 0$, and 
   $(D_n)_{n\in \sN}$ be an $\sF$-adapted sequence
  of random variables 
    such that
    for all $n\in \sN$, 
     $\sE[D_{n}|\cF_{n-1}] = 0$
     and 
     $\|D_n\|_{1,\cF_{n-1}}\le C$. 
 Then there exists a constant $C'>0$, depending only on $C$, such that  for all $N \in \sN$ and $\eps>0$,
$\sP\big( \big|\sum_{n=1}^N D_n \big| \geq N \eps \big) \leq 2 \exp \big(-  C' N \min(\eps^2,\eps)\big)
$.
\end{Proposition}

\subsection{Error bound of maximum a posterior estimate}
\label{sec:proof_bayesian}

In this section, we consider
the probability space 
 $(\Om,\cF,\sP)$ 
defined in \eqref{eq:space}
and analyse the MAP estimate 
$(\hat{\bs{\theta}}^{\bs{\Psi}, m} )_{m \in \sN}$ defined in  \eqref{eq: statistics}.
We first establish an error  bound of   $(\hat{\bs{\theta}}^{\bs{\Psi}, m} )_{m \in \sN}$ 
associated with  a general learning algorithm 
$\bs{\Psi}$ (cf.~Definition \ref{Def: Admissible RL}).
 
\begin{Lemma}
\label{Lemma: square loss}
Let $x_0\in \sR^d$,
$\hat{\theta}_0\in \sR^{d\t (d+p)}$, $V_0 \in \sS_+^{d+p}$, 
$\bs{\Psi}$ be a   learning algorithm 
as in Definition \ref{Def: Admissible RL},
and 
for each $m\in \sN$,
let 
 $Z^{\bs{\theta}, \bs{\Psi}, m}$
be defined in  \eqref{eq:X_m},
$V^{\bs{\theta}, \bs{\Psi}, m}$
be defined in   \eqref{eq: statistics}
and 
$G^{\bs{\theta}, \bs{\Psi}, m} = (V^{\bs{\theta}, \bs{\Psi}, m})^{-1}$.
Then for all
$\delta \in (0,1)$ and $m \in \sN$, 
$$
\sP\Bigg( \bigg\|
 \sum_{n=1}^m 
  \bigg( \int_0^T 
 Z^{\bs{\theta}, \bs{\Psi}, m}_t
 (\d W^m_t)^\top
 \bigg)^\top
  \bigg\|^2_{V^{\bs{\theta}, \bs{\Psi}, m}} \leq 2 \ln \bigg(\frac{\big( \det G^{\bs{\theta}, \bs{\Psi}, m}  \cdot \det V_0 \big)^{d/2}}{\delta}\bigg) \Bigg) \geq 1-\delta,
$$
 where  $\|S\|_{V}^2 \coloneqq \tr(SVS^\top)$
 for all $S\in \sR^{d\t (d+p)}$ with $V\in \sR^{(d+p)\t (d+p)}$.
\end{Lemma}

\begin{proof}
We first concatenate  
stochastic processes from all episodes
into  infinite-horizon 
processes. Let $W:\Om \times [0,\infty)\to\sR^d$ be the cumulative concatenation of  the Brownian motions $(W^m_t)_{t \in [0,T], m \in \sN}$ and let $Z:\Omega \times [0,\infty)\to \sR^{d+p}$ be the concatenation of  $(Z^{\bs{\theta}, \bs{\Psi}, m}_t)_{t\in [0,T], m \in \sN}$ such that 
$$
W_t \coloneqq W^{1+ \lf t/T \rf }_{t\;\textrm{mod}\; T}+ \sum_{n=1}^{\lf t/T \rf } W^n_T,
 \q Z_t \coloneqq Z^{\theta, \bs{\Psi}, 1+ \lf t/T \rf}_{t\;\textrm{mod}\; T},
 \q \fa t\geq0,
 $$
 Note that $W$ is a standard Brownian motion with respect to the 
filtration $\cF_t := \sigma\{\bs{\theta}, W_u | u \leq t\}$.
Let 
 $S_t \coloneqq (\int_0^{t}  Z_s(\d W_s)^\top)^\top$, $V_t \coloneqq (V_0^{-1} + \int_0^{t} Z_s Z_s^\top ds)^{-1}$ and $G_t \coloneqq V_t^{-1}$ for all $t\ge 0$. 
By  \eqref{eq: statistics},
$
S_{mT} = \sum_{n=1}^m \big(\int_0^T 
Z^{\bs{\theta}, \bs{\Psi}, m}_t
( \d W^m_t ) ^\top\big )^\top$,
 $V_{mT} = V^{\bs{\theta}, \bs{\Psi}, m}$
 and 
 $G_{mT} = G^{\bs{\theta}, \bs{\Psi}, m}$
 $\sP$-a.s.
 Hence it suffices to prove the required result by considering $W, S,G, V$.

Observe that $(G_t - V_0^{-1})_{t\ge 0}$ is the quadratic variation process $([S]_t)_{t\ge 0}$ of $(S_t)_{t\ge 0}$. 
Based on this observation, we  express 
 $\exp\big(\tfrac{1}{2} \|S_t\|^2_{V_t} \big) $ in terms of  integrals of suitable exponential martingales. 
 To this end, for each 
 $U \in \sS^{d+p}_+$,  let $c(U) \coloneqq \int_{\sR^{d\t (d+p)}} \exp \big(-\tfrac{1}{2} \|\theta\|^2_U \big)\, \d \theta = \sqrt{(2\pi)^{d(d+p)}/( \det U)^d} $ be the normalising constant corresponding to the density of the matrix normal distribution $\cM\cN(0,I_d, U^{-1})$. By completing  the square in the matrix Gaussian density, for all $t\ge 0$,
 \begin{align*}
 \label{eq: S exp expression}
&\exp\big(\tfrac{1}{2} \|S_t\|^2_{V_t} \big) = \frac{1}{c(G_t)}\int  \exp \Big( \tfrac{1}{2} \|S_t\|^2_{G_t^{-1}} - \tfrac{1}{2} \big\|\theta - G_t^{-1} S_t\big\|^2_{G_t}\Big) \d \theta \\
&\q= \frac{1}{c(G_t)}\int  \exp \Big( \tr (\theta^\top S_t )  - \tfrac{1}{2}\| \theta \|_{G_t}^2 \Big)  \,\d \theta =   \frac{1}{c(G_t)}\int  \exp \Big( \tr (\theta^\top S_t )  - \tfrac{1}{2}\| \theta \|_{[S]_t}^2  - \tfrac{1}{2}\| \theta \|^2_{V_0^{-1}}\Big)  \,\d \theta \\
&\q = 
 \frac{1}{c(G_t)}\int \cM^\theta_t \exp \big( - \tfrac{1}{2}\| \theta \|^2_{V_0^{-1}}\big)  \,\d \theta 
 = \frac{ ( \det G_t \cdot \det V_0)^{d/2} }{c(V_0^{-1})} \int \cM^\theta_t  \exp \Big(-\tfrac{1}{2} \|\theta\|^2_{V_0^{-1}} \Big)
 \,\d \theta,
 \end{align*}
where $\int $ denotes the integration  over $\sR^{d \times (d+p)}$ and for each $\theta\in \sR^{d \times (d+p)}$, 
$\cM^\theta$ is an exponential martingale defined by
$$
\cM_t^\theta \coloneqq \exp \big( \tr (\theta^\top S_t ) - \tfrac{1}{2}\| \theta \|_{[S]_t}^2\big),
\q \fa t\ge 0.
$$
 As $W$ is a standard Brownian motion, 
by It\^o's  formula, 
 $\cM^\theta $ is a non-negative local martingale with respect to   $(\cF_t)_{t \geq 0 }$ and hence a supermartingale.    In particular, for all $ t \geq 0$ and $\theta \in \sR^{d \t (d+p)}$, $\sE[\cM^\theta_t] \leq \mathbb{E}[\cM^\theta_0] = 1 $. 


Therefore, by  Markov's inequality and  Fubini's theorem,
\begin{align*}
&\sP\bigg(\|S_t\|^2_{V_t} >   2 \ln \bigg(\frac{( \det G_t \cdot \det V_0)^{d/2}}{\delta}\bigg) \bigg)  = \sP\bigg(\exp \Big( \tfrac{1}{2 }\|S_t\|^2_{V_t} \Big) >  \frac{( \det G_t \cdot \det V_0 )^{d/2}}{\delta} \bigg) 
\\ &\qq \leq \delta   \sE \bigg[\frac{\exp ( \tfrac{1}{2 }\|S_t\|^2_{V_t} )}{( \det G_t \cdot \det V_0)^{d/2}} \bigg] = \delta \sE \bigg[  \frac{1}{c(V_0^{-1})} \int \cM^\theta_t  \exp \Big(-\tfrac{1}{2} \|\theta\|^2_{V_0^{-1}} \Big)
 \,\d \theta\bigg]\\
  &\qq =   \frac{\delta }{c(V_0^{-1})} \int \sE[\cM^\theta_t]  \exp \Big(-\tfrac{1}{2} \|\theta\|^2_{V_0^{-1}} \Big)
 \,\d \theta \leq    \frac{\delta }{c(V_0^{-1})} \int \exp \Big(-\tfrac{1}{2} \|\theta\|^2_{V_0^{-1}} \Big)
 \,\d \theta= \delta.
\end{align*}
 where the last inequality follows from the fact that 
 $ \mathbb{E} [\cM^\theta_t ] \le 1$ for all $t\ge 0$.
 Substituting $t=mT$ in the above inequality leads to the desired estimate. 
\end{proof}

\begin{Proposition}
\label{prop: bound hat theta}
Suppose  (H.\ref{assum:bound}) holds.  
Let
$x_0\in \sR^d$,
$T>0$,
$\hat{\theta}_0\in \sR^{d\t (d+p)}$, $V_0 \in \sS_+^{d+p}$, 
$\bs{\Psi}$ be a   learning algorithm 
as in Definition \ref{Def: Admissible RL}, and for each $m \in \sN$,  let $\hat{\bs{\theta}}^{\bs{\Psi}, m}$ and 
 $G^{\bs{\theta}, \bs{\Psi}, m} \coloneqq (V^{\bs{\theta}, \bs{\Psi}, m} )^{-1}$ where $(\hat{\bs{\theta}}^{\bs{\Psi}, m}, V^{\bs{\theta}, \bs{\Psi}, m} )$
is defined in   \eqref{eq: statistics}.
There exists a constant $C\ge 0$ 
such that 
for all 
$\delta \in (0,1)$ and $m \in \sN$, 
$$\sP\Big(\lambda_{\min}({G^{ \bs{\theta}, \bs{\Psi} , m}}) | \hat{\bs{\theta}}^{\bs{\Psi} , m} - \bs{\theta} |^2 \leq C \big(1 + \ln ( \det G^{\bs{\theta}, \bs{\Psi}, m} \cdot \det V_0) + \ln(\tfrac{1}{\delta}) \big) \Big) \geq 1-\delta,$$
where 
 $\lambda_{\min}(S)$ is the smallest eigenvalue of a matrix $S\in \sS^{d+p}_+$.

 \end{Proposition}
 
  \begin{proof}
By \eqref{eq:X_m} and
\eqref{eq: statistics}, 
 \begin{align*}
 \hat{\bs{\theta}}^{\bs{\Psi}, m} - \bs \theta 
 &= \Bigg(\hat{\theta}_0 V_0^{-1} +  \sum_{n=1}^m 
  \bigg(
 \int_0^T Z^{\bs{\theta}, \bs{\Psi}, n}_t  (\d X^{{\theta}, \bs{\Psi}^{\theta }, n}_t)^\top
 \bigg)^\top \Bigg)\big(V^{\bs{\theta}, \bs{\Psi}, m} \big)-\bs \theta \\
& 
 = \Bigg(\hat{\theta}_0 V_0^{-1} +  \sum_{n=1}^m 
  \bigg(
 \int_0^T Z^{\bs{\theta}, \bs{\Psi}, n}_t  (
 \bs \theta Z^{\bs{\theta}, \bs{\Psi}, n}_t \d t + \d W^n_t 
 )^\top
 \bigg)^\top \Bigg)\big(G^{\bs{\theta}, \bs{\Psi}, m} \big)^{-1} -\bs \theta \\
 & =   \Bigg(\hat{\theta}_0 V_0^{-1} + \bs\theta \big(G^{\bs{\theta}, \bs{\Psi}, m} - V_0^{-1} \big) +  \sum_{n=1}^m
   \bigg( \int_0^T 
 Z^{\bs{\theta}, \bs{\Psi}, n}_t
 (\d W^n_t)^\top
 \bigg)^\top
  \Bigg)\big(G^{\bs{\theta}, \bs{\Psi}, m} \big)^{-1} -\bs \theta
 \\
 & =   \bigg(( \hat{\theta}_0 - \bs\theta) V_0^{-1}  +  \sum_{n=1}^m 
   \bigg( \int_0^T 
 Z^{\bs{\theta}, \bs{\Psi}, n}_t
 (\d W^n_t)^\top
 \bigg)^\top
  \bigg)\big(G^{\bs{\theta}, \bs{\Psi}, m} \big)^{-1}.
 \end{align*}
  Right multiplying the above identity by $G^{\bs{\theta}, \bs{\Psi}, m}(\hat{\bs{\theta}}^{\bs{\Psi}, m} -\bs \theta )^\top$, taking the trace 
  and using the fact that  $V^{\bs{\theta }, \bs{\Psi}, m} = (G^{\bs{\theta }, \bs{\Psi}, m})^{-1} $ give us that
 \begin{align*}
  \|  \hat{\bs{\theta}}^{\bs{\Psi} , m} - \bs\theta 
\|^2_{G^{\bs{\theta}, \bs{\Psi}, m}}
 &= \bigg\| (\hat{\theta}_0 - \bs\theta) V_0^{-1}  +  \sum_{n=1}^m 
   \bigg( \int_0^T 
 Z^{\bs{\theta}, \bs{\Psi}, n}_t
 (\d W^n_t)^\top
 \bigg)^\top \bigg\|^2_{V^{\bs{\theta}, \bs{\Psi}, m}} \\
 & \leq  
 \bigg(
 \big\| (\hat{\theta}_0 - \bs\theta) V_0^{-1} \big\| _{V^{\bs{\theta}, \bs{\Psi}, m}}  + \bigg\| 
 \sum_{n=1}^m 
   \bigg( \int_0^T 
 Z^{\bs{\theta}, \bs{\Psi}, n}_t
 (\d W^n_t)^\top
 \bigg)^\top \bigg\|_{V^{\bs{\theta}, \bs{\Psi}, m}}
 \bigg)^2,
 \end{align*}  
  where the last inequality follows from the fact that 
for all  $V\in \sS^{d+p}_+$,
  $S\mapsto \|S\|_{V}=\sqrt{\tr(SVS^\top)}$
  is a norm  on $\sR^{d\t (d+p)}$.
  By \eqref{eq: statistics}, 
  $ (V^{\bs{\theta }, \bs{\Psi}, m})^{-1}- 
  V_0^{-1}$ is symmetric positive semidefinite, 
  and hence 
    $  
  V_0-(V^{\bs{\theta }, \bs{\Psi}, m})$ is symmetric positive semidefinite. 
  Consequently,   for all $S\in \sR^{d\t (d+p)}$, 
  $\| S \| _{V^{\bs{\theta}, \bs{\Psi}, m}} \le \|S\|_{V_0}$.
The desired estimate then follows from {the boundedness of $\bs{\theta}$ (cf.~(H.\ref{assum:bound})), }
Lemma \ref{Lemma: square loss} and 
the fact that there exists $C\ge 0$ such that
$  \lambda_{\min}(V)|S|^2\le C\| S
\|^2_{V}
$ for all $S\in \sR^{d\t (d+p)}$ and $V\in \sS^{d+p}_+$.
  \end{proof}
 

 We then focus on  learning algorithms 
 of the class
 $\cup_{L\ge 0}\cU_L$
and obtain an upper bound of
$\det G^{\bs{\theta}, \bs{\Psi}, m}$.
 By extending
\cite[Corollary 4.1]{djellout2004transportation}
to SDEs with random coefficients,
 we 
 prove 
 Lipschitz functionals of 
  state processes $X^{\bs{\theta }, \bs{\Psi}, m}$ 
are conditional sub-Gaussian uniformly with respect to $m$.

{
\begin{Proposition}
\label{prop:X_subgaussian}
Suppose  (H.\ref{assum:bound}) holds.  
Let $x_0\in \sR^d$,
$T, L\ge 0$, 
$\bs{\Psi}$ be a   learning algorithm 
in  $\cU_L$  defined as in \eqref{eq:lipschitz_algorithm},
and 
for each $m\in \sN$,
let 
$X^{\bs{\theta},\bs{\Psi}, m}$
be defined in 
\eqref{eq:X_m}.
 Then there exists a constant $C\ge 0$, depending only on $x_0, T$ and $L$, such that  for all  
 $m\in \sN$ and  
$\phi: C([0,T];\sR^d)\to \sR$
with  
$\sup_{\rho_1\not=\rho_2}\f{|\phi(\rho_1)-\phi(\rho_2)|}{\|\rho_1-\rho_2\|_\infty}\le 1$ and $|\phi(0)|\le 1$,
we have $\|\phi(X^{\bs{\theta},\bs{\Psi}, m})\|_{2,\cF_{m-1}}\le C$,
where
 $\cF_{m-1}\coloneqq \sigma\{\bs{\theta}, W^n_t\mid t\in [0,T], n=1,\ldots, m-1\}\vee \cN$
 with $\cN$ being the $\sigma$-algebra generated by 
$\sP$-null sets.

\end{Proposition}
 
\begin{proof}
Throughout this proof, let $m\in \sN$ 
and the function $\phi:C([0,T]; \sR^d)\to \sR$
be fixed, and $C$ 
be a generic constant 
 depending only on $x_0,T,L$.

By \eqref{eq:X_m}, 
$X^{\bs{\theta}, \bs{\Psi}, m}$ satisfies the dynamics 
\bb\l{eq:random_beta_m}
\d X_t =
b_m(\cdot, t, X_t)
\,\d t+  \d W^m_t, \q t\in [0,T],
\q X_0=x_0,
\ee
with $b_m(\om,t,x)\coloneqq 
\bs{A}(\omega) x+\bs{{B}}(\omega)\Psi_m(\om, t, x)$
for all $(\om,t,x)\in \Om \t [0,T]\t \sR^d$.
The boundedness of   $\Theta$
and  $\Psi_m(\omega, \cdot) \in \cV_L$ give us that 
$|b_m(\om, t, 0)| \leq C$ and $|b_m(\om, t, x) - b_m(\om, t, x') | \leq C |x-x'|$.
The measurability of $\Psi_m$ implies that 
$b_m$ is $( \cF_{m-1} \otimes \cB([0,T]) \otimes \cB(\sR^d))/ \cB(\sR^d)$-measurable.
By the definition of $\cF_{m-1}$ 
and the Doob--Dynkin lemma, there exists 
a Borel measurable
 function $\mathfrak{b}_m:  \Theta \t C([0,T] ; \sR^d)^{ {m-1}} \times [0,T] \t \sR^d \to \sR^d$ such that $b_m(\omega,  t, x) 
=\mathfrak{b}_m (\bs{\theta}(\omega), W^1(\omega),\ldots, W^{m-1}(\omega),  t, x) 
 $ for all $(t,x) \in [0,T] \times \sR^d$ and $\sP$-almost all $\omega \in \Om$.

The definition  of
$(\Om, \cF, \sP) $
in \eqref{eq:space}
and $W^m$ is the canonical process on $\Om^{W^m}=C([0,T];\sR^d)$ (see Section \ref{sec:framework})
imply that  all $\om\in \Om$ can be uniquely decomposed into 
$\om= (\theta, \ul{\omega},  \ol{\omega})$ with $\theta = \bs{\theta}(\omega) \in \Theta$,
$\ul{\omega}=(W^n(\om))_{n=1}^{m-1}\in  \prod_{n=1}^{m-1} \Omega^{W^n}$ and $\ol{\omega}=(W^n(\om))_{n=m}^{\infty}\in \prod_{n=m}^\infty \Omega^{W^n}$.
Moreover, for each $\ul{\omega}\in  \prod_{n=1}^{m-1} \Omega^{W^n}$, 
let 
${X}^{\theta, \bs{\Psi}, m,\ul{\om}}$
be the unique solution  to the following dynamics
\bb\l{eq:sde_om}
\d X_t =
\mathfrak{b}_m(\theta, \ul{\omega}, t, X_t)
\,\d t+  \d W^m_t, \q t\in [0,T],
\q X_0=x_0,
\ee
on the space   $\big(\prod_{m=n}^\infty \Omega^{W^m}, \bigotimes_{m=n}^\infty \cF^{W^m}, \bigotimes_{m=n}^\infty \sP^{W^m} \big)$.
Applying \cite[Corollary 4.1]{djellout2004transportation}
to the law of ${X}^{ \theta, \bs{\Psi}, m,\ul{\om}}$
shows that there exists $C\ge 0$ such that 
for all 
$\ul{\om}\in \prod_{n=1}^{m-1} \Omega^{W^n}$ and
$\gamma\in \sR$,
\bb\l{eq:underline_om}
\sE^{\bigotimes_{m=n}^\infty \sP^{W^m} }
\big[\exp\big(\gamma (\phi({X}^{\theta, \bs{\Psi}, m,\ul{\om}})
-\sE^{\bigotimes_{m=n}^\infty \sP^{W^m} }[\phi({X}^{ \theta, \bs{\Psi}, m,\ul{\om}})])\big)\big]\le \exp(C\gamma^2).
 \ee
By \eqref{eq:random_beta_m} and the definition of $\mathfrak{b}_m$, ${X}^{\bs{\theta}, \bs{\Psi}, m}((\theta, \ul{\om},\ol{\om}))=
{X}^{ \theta, \bs{\Psi}, m,\ul{\om}}(\ol{\om})$ $\sP$-almost all $\om =  (\theta, \ul{\om},\ol{\om})$.
It then follows from the independence between $\bs{\theta}$, $(W^n)_{n=1}^{m-1}$ and $(W^n)_{n=m}^{\infty}$,
\cite[Theorem 13, p~76]{krylov2002introduction} and 
\eqref{eq:underline_om}  that 
 $$
\sE
\big[\exp\big(\gamma (\phi({X}^{\bs{\theta}, \bs{\Psi}, m})
-\sE[\phi({X}^{\bs{\theta}, \bs{\Psi}, m})\mid \cF_{m-1}
])\big)
\mid \cF_{m-1}
\big]\le \exp(C\gamma^2),
\q \fa \gamma\in \sR.
 $$
By \eqref{eq:bounded}, \eqref{eq:triangle} and \eqref{eq:subGaussian},
$\|\phi({X}^{\bs{\theta}, \bs{\Psi}, m})
\|_{2,\cF_{m-1}}
\le 
C\big(1+\sE[  |\phi({X}^{\bs{\theta}, \bs{\Psi}, m})|\mid \cF_{m-1}]\big)
$.
Standard moment estimate of \eqref{eq:random_beta_m}
shows that 
$\sE[\sup_{t\in [0,T]}|{X}^{\bs{\theta}, \bs{\Psi}, m}_t|^2\mid \cF_{m-1}]\le C$,
which along the growth condition of $\phi$ implies the desired sub-Gaussianity of $\phi({X}^{\bs{\theta}, \bs{\Psi}, m})$. 
\end{proof}
}

\begin{proof}[Proof of Theorem \ref{Theorem: Information collapse}]
For each $m\in \sN\cup\{0\}$,  let  $\cF_m \coloneqq \sigma\{\bs{\theta}, W^n_t \mid t \in [0,T], n =1, \ldots, m\} \vee \cN$, with $\cN$ being the $\sigma$-algebra generated by 
$\sP$-null sets,
and let  $(\hat{\bs{\theta}}^{\bs{\Psi}, m}, 
 V^{\bs{\theta}, \bs{\Psi}, m})$ 
be defined in   \eqref{eq: statistics},
$G^{\bs{\theta}, \bs{\Psi}, m} = (V^{\bs{\theta}, \bs{\Psi}, m})^{-1}$,
and let 
 $C$ 
be a generic constant independent of $m$.

Applying Proposition \ref{prop:X_subgaussian} with 
$C([0,T];\sR^d)\ni \rho\mapsto 
\phi(\rho)\coloneqq \frac{1}{\sqrt{T}}(\int_0^T |\rho_t|^2\, \d t)^{\frac{1}{2}}
\in 
\sR$
yields that
$\|(\int_0^T X^{\bs{\theta},\bs{\Psi}, m}_t\,\d t)^{\frac{1}{2}}\|_{2,\cF_{m-1}}\le C$. 
The fact that $\bs{\Psi}\in \cup_{L\ge 0} \cU_L$ implies that
there exists $L\ge 0$ such that 
$|\Psi_m(\om,t,x)|\le  L(|1+|x|)$
for all $m\in \sN$ and  $(\om,t,x)\in \Om\t [0,T]\t \sR^d$, 
which along with Proposition 
\ref{prop:subguassian_subexponential} implies that 
$$
\bigg\|\int_0^T Z^{\bs{\theta}, \bs{\Psi}, m}_t
(Z^{\bs{\theta}, \bs{\Psi}, m}_t)^\top\, \d t
\bigg\|_{1,\cF_{m-1}}
=
\bigg\|\int_0^T
\begin{psmallmatrix}
X^{\bs{\theta}, \bs{\Psi}, m}_t \\ \Psi_m\big(\cdot, t, X^{\bs{\theta}, \bs{\Psi}, m}_t\big)
 \end{psmallmatrix}
\begin{psmallmatrix}
X^{\bs{\theta}, \bs{\Psi}, m}_t \\ \Psi_m\big(\cdot, t, X^{\bs{\theta}, \bs{\Psi}, m}_t\big)
 \end{psmallmatrix}^\top\, \d t
\bigg \|_{1,\cF_{m-1}}\le C,
 $$
with $Z^{\bs{\theta}, \bs{\Psi}, m}$ defined as in \eqref{eq:X_m}.
Hence, by \eqref{eq:centering} and 
Proposition \ref{prop: concentration inequality},
there exists a constant $C\ge 0$ such that 
 for all $m\in \sN$ and $\delta\in (0,1)$,
 with probability  at least $1-\delta$, 
 \begin{align}
 \label{eq:concentration_Z}
 \begin{split}
&\bigg|\sum_{n=1}^m 
\bigg(
\int_0^T Z^{\bs \theta, \bs{\Psi}, n}_t
(Z^{\bs \theta, \bs{\Psi}, n}_t)^\top\, \d t
-
\sE \Big[\int_0^T Z^{\bs \theta, \bs{\Psi}, n}_t
(Z^{\bs \theta, \bs{\Psi}, n}_t)^\top\, \d t
\mid \cF_{n-1}\Big]
\bigg)
 \bigg|
 \\
 &  \le  m \max \bigg( \sqrt{\frac{\ln(2/\delta)}{Cm} },
\frac{\ln(2/\delta)}{Cm} \bigg)
\le C (m+ \ln (\tfrac{1}{\delta} )).
\end{split}
\end{align}
As shown in the proof of Proposition \ref{prop:X_subgaussian},
$\sE[\sup_{t\in [0,T]}|{X}^{\bs{\theta}, \bs{\Psi}, m}_t|^2\mid \cF_{m-1}]\le C$, which along with the fact that 
$\Psi_m(\om,\cdot)\in \cV_L$ for all $\om\in \Om$ implies that 
$\big|\sum_{n=1}^m\sE \big[\int_0^T Z^{\bs \theta, \bs{\Psi}, n}_t
(Z^{\bs \theta, \bs{\Psi}, n}_t)^\top\, \d t
\mid \cF_{n-1}\big]\big|\le Cm$.
Consequently, by  \eqref{eq: statistics} and the fact that $G^{\bs{\theta}, \bs{\Psi}, m} = (V^{\bs{\theta}, \bs{\Psi}, m})^{-1}$,
there exists $C\ge 0$ such that 
for all 
 $m\in \sN$ and $\delta>0$,
$
\sP\big(|G^{\bs{\theta}, \bs{\Psi}, m}|
\le C(m+\ln (\tfrac{1}{\delta})\big)>1-\delta$.
Then by Proposition \ref{prop: bound hat theta} 
 and the inequality $\det (G^{\bs \theta,\bs{\Psi}, m}) \le C|G^{\bs{\theta}, \bs{\Psi}, m}|$,  
we have for all 
 $m\in \sN$ and $\delta>0$,
$$
\sP\big(\lambda_{\min}({G^{\bs \theta,\bs{\Psi}, m}}) | \hat{\bs{\theta}}^{\bs{\Psi}, m} - \bs \theta |^2
\le C  (1+\ln m+\ln (\tfrac{1}{\delta})\big)>1-2\delta.
$$
{The desired estimate in 
Theorem \ref{Theorem: Information collapse}
with $m\ge 2$
then follows from the fact that $\ln 2 > 0$. }
\end{proof}

\section{Proofs of 
Proposition \ref{prop:construct_psi_e}
and Theorems \ref{thm:performance_gap_smooth}
and \ref{thm:performance_gap_entropy}}
\label{sec: proof of assumptions}

\begin{proof}[Proof of Proposition \ref{prop:construct_psi_e}]
For the if direction, let  $\bA\coloneqq\{a_1,\ldots,a_p\}\subset \operatorname{dom}(h)$ be a collection of linearly independent vector in $\sR^p$, $\psi^e$ be defined as in the statement,
and $(u,v)\in \sR^d\t \sR^p$ be such that 
$u^\top x+v^\top \psi^e(t,x)=0$
for almost every $(t,x)\in [0,T]\t \sR^d$.
By the boundedness of $\bA$ and the definition of $\psi^e$, $u=0$ and  $v\in \sR^p$ is orthogonal to the  space $\operatorname{span}(\bA)$
spanned by $\bA$.
Then the linear dependence of $\{a_i\}_{i=1}^p$ implies $\operatorname{span}(\bA)=\sR^p$ and hence $v=0$.
It is clear that $\psi^e\in \cup_{C\ge 0}\cV_C$,
and 
(H.\ref{Assump: Exploration strategy}\ref{item:exploration_quadratic}) follows directly from 
$\bA\subset \operatorname{dom}(h)$ and the quadratic growth of $f_0$ in (H.\ref{assum:lc}\ref{item:f0R}) ).

For the only if direction, 
observe that  any $\psi^e \in \cup_{C\ge 0} \cV_C$
satisfying (H.\ref{Assump: Exploration strategy}\ref{item:exploration_quadratic}) takes values in $\operatorname{dom}(h)$ a.e.
Hence, if (H.\ref{Assump: Exploration strategy}\ref{item:f0R}) holds, then the  space 
  $\operatorname{span}(\operatorname{dom}(h))$
spanned by $\operatorname{dom}(h)$
must have dimension $p$, since otherwise there exists a non-zero $v\in \sR^p$ orthogonal to $\operatorname{dom}(h)$ such that  for all 
$\psi^e$ satisfying (H.\ref{Assump: Exploration strategy}\ref{item:exploration_quadratic}),
$v^\top \psi^e(t,x)=0$
for almost every $(t,x)\in [0,T]\t \sR^d$.
Reducing $\operatorname{dom}(h)$
to a basis for 
$\operatorname{span}(\operatorname{dom}(h))$
shows that 
$\operatorname{dom}(h)$ contains $p$ linearly independent vectors. 
\end{proof}

\begin{proof}[Proof of Theorem 
\ref{thm:performance_gap_smooth}]
Throughout this proof,
let 
$\theta_0=(A_0,B_0)\in \Theta$
and $\beta>0$ be given constants,
and $J(\cdot;\theta_0):\cH^2_\sF(\Om;\sR^p)\to \sR$ be such that
\begin{align*}
J(\a;\theta_0)\coloneqq 
\sE\bigg[\int_{0}^{T} f(t, X^{\theta_0,\a}_t,   \a_t)\,\d t + g(X^{\theta_0,\a}_{T} )\bigg],
\q \fa \a\in \cH^2_\sF(\Om;\sR^p),
\end{align*}
where $X^{\theta_0,\a}\in\cS^2_\sF(\Om;\sR^d)$ is the strong solution to 
\bb\label{eq:mp_state}
\d X_t=(A_0 X_t+B_0 \a_t)\, \d t+\d W_t, \q t\in [0,T]; \q X_0=x_0.
\ee
We denote by $C\ge 0$ a generic constant  which
depends on $\Theta,\b$
but
is independent of $\theta_0$.


By
the differentiability of $f$ and $g$ and the standard variational analysis (see e.g., \cite{yong1999stochastic,vsivska2020gradient}),
$J(\cdot;\theta_0)$ is Fr\'{e}chet differentiable and its  derivative satisfies for all $\a\in \cH^2_\sF(\Om;\sR^p)$, \begin{align}
\label{eq:nabla_J}
\nabla J(\a;\theta_0)_t\coloneqq 
B_0^\top Y^{\theta_0,\a}_t +(\p_a f)(t, X^{\theta_0,\a}_t,   \a_t), 
\quad \textnormal{$\d \sP\otimes \d t$ a.e.,}
\end{align}
where $(Y^{\theta_0,\a},Z^{\theta_0,\a})\in \cS^2_\sF(\Om;\sR^d)\t \cH^2_\sF(\Om;\sR^{d\t d})$ is the unique solution to 
\bb\label{eq:mp_adjoint}
\d Y_t=-\big(A_0^\top Y_t+(\p_x f)(t, X^{\theta_0,\a}_t,   \a_t)
\big)\, \d t+Z_t\d W_t, \q t\in [0,T]; \q Y_T=(\nabla g)(X^{\theta_0,\a}_T).
\ee
Standard stability analysis of \eqref{eq:mp_state} 
shows that 
$\|X^{\theta_0,\a}-X^{\theta_0,\a'}\|_{\cS^2}\le C\|\a-\a'\|_{\cH^2}$ for all   $\a,\a'\in \cH^2_\sF(\Om;\sR^p)$,
which along with  
the Lipschitz continuity of $\p_x f$ and $\nabla g$ and 
the stability of \eqref{eq:mp_adjoint} in
\cite[Proposition 2.1]{el1997backward} gives 
$\|Y^{\theta_0,\a}-Y^{\theta_0,\a'}\|_{\cS^2 }\le C\|\a-\a'\|_{\cH^2}$.
Hence, 
by \eqref{eq:nabla_J} and 
the Lipschitz continuity of $\p_a f$, 
$\nabla J(\cdot;\theta_0):\cH^2_\sF(\Om;\sR^p)\to  \cH^2_\sF(\Om;\sR^p)$
is Lipschitz continuous.
Hence, there exists $C\ge 0$
such that $\a,\a'\in \cH^2_\sF(\Om;\sR^p)$,
\begin{align}
\label{eq:quadratic_bound}
\begin{split}
    &J(\a';\theta_0)-J(\a;\theta_0)-\langle \nabla J(\a;\theta_0),\a'-\a \rangle_{\cH^2}
\\
&=\int_0^1
\langle \nabla J(\a+s(\a'-\a);\theta_0)-\nabla J(\a;\theta_0),\a'-\a \rangle_{\cH^2}\,\d s
\le C\|\a'-a\|^2_{\cH^2},    
\end{split}
\end{align}
where $\langle \cdot, \cdot\rangle_{\cH^2}$
denotes the inner product on 
$\cH^2_\sF(\Om;\sR^p)$.

For each $\theta\in \sB_\b(\theta_0)$,
let 
$X^{\theta_0, \psi_{\theta}}\in \cS^2_{\sF}(\Om; \sR^d)$
be the  state process associated with $\theta_0$ and  $\psi_{\theta}$,
and let 
$\a^{\theta_0,\psi_{\theta}}\in \cH^2_{\sF}(\Om; \sR^p)$
be such that 
$\a^{\theta_0,\psi_{\theta}}_t=\psi_\theta(t,X^{\theta_0, \psi_{\theta}}_t)$
for $\d\sP\otimes \d t$ a.s.
For any given $\theta\in \sB_\b(\theta_0)$,
substituting 
$\a'=\a^{\theta_0,\psi_{\theta}}$
and $\a=\a^{\theta_0,\psi_{\theta_0}}$ into
\eqref{eq:quadratic_bound} leads to
$$
J(\a^{\theta_0,\psi_{\theta}};\theta_0)-J(\a^{\theta_0,\psi_{\theta_0}};\theta_0)-\langle \nabla J(\a^{\theta_0,\psi_{\theta_0}};\theta_0),\a^{\theta_0,\psi_{\theta}}-\a^{\theta_0,\psi_{\theta_0}} \rangle_{\cH^2}
\le C\|\a^{\theta_0,\psi_{\theta}}-\a^{\theta_0,\psi_{\theta_0}}\|^2_{\cH^2}.
$$
Observe from    \eqref{eq:loss_feedback} that 
$J(\psi_\theta;\theta_0)=
J(\a^{\theta_0,\psi_{\theta}};\theta_0)$.
Moreover, the definition of $\psi_{\theta_0}$ implies that 
$\a^{\theta_0,\psi_{\theta_0}}$ minimises the functional 
$J(\cdot;\theta_0):\cH^2_\sF(\Om;\sR^p)\to \sR$, which along with the differentiablity of $J(\cdot;\theta_0)$ shows that 
$\nabla J(\a^{\theta_0,\psi_{\theta_0}};\theta_0)=0$.
Hence, we have
$$
J(\a^{\theta_0,\psi_{\theta}};\theta_0)-J(\a^{\theta_0,\psi_{\theta_0}};\theta_0)
\le C\|\a^{\theta_0,\psi_{\theta}}-\a^{\theta_0,\psi_{\theta_0}}\|^2_{\cH^2}.
$$
By Proposition \ref{prop:existence_fb}, 
there exists $C\ge 0$ such that 
 for all $\theta\in \sB_\b(\theta_0)$
 and $(t,x)\in [0,T]\t \sR^d$,  $\psi_{\theta}\in \cV_C$
 and 
$|\psi_\theta(t,x)-\psi_{\theta_0}(t,x)|\le C(1+|x|)|\theta-\theta_0|$.
Hence 
standard stability analysis of SDEs shows that 
$\|X^{\theta_0, \psi_{\theta}}
-X^{\theta_0, \psi_{\theta_0}}
\|_{\cS^2}
+\|\a^{\theta_0,\psi_{\theta}}-\a^{\theta_0,\psi_{\theta_0}}\|_{\cH^2}
\le C|\theta-\theta_0|$
and finishes the proof.
\end{proof}

 Before proving Theorem \ref{thm:performance_gap_entropy},
 we first establish two technical lemmas for the function
 \bb\label{eq:FK}
  \sR^{d\t (d+p)}\t 
 [0,T]\t \sR^d
 \ni (\theta,t,x)\mapsto 
 Y^{\theta, t,x}_t\in \sR,
 \ee
 where for each $\theta=(A,B)$ and $(t,x)$, 
$(X^{\theta,t,x},Y^{\theta,t,x},Z^{\theta,t,x})\in \cS^2_\sF(\Om\t [t,T];\sR^d)\t \cS^2_\sF(\Om\t [t,T];\sR)\t \cH^2_\sF(\Om\t [t,T];\sR^{1\t d})$ be the unique solution to the following forward-backward stochastic differential equation (FBSDE): for all $s\in [t,T]$,
\begin{subequations}\l{eq:fbsde_general}
\begin{alignat}{2}
\mathrm{d}X_s&=
AX_s\, \d s
+\, \d W_s,
&&
\q X_t=x,
\l{eq:lc_fbsde_fwd}
\\
\mathrm{d}Y_s&
=f(s, X_s, B^\top Z^\top_s)\,\d s
+Z_s\, \d W_s,
&&
\q Y_T= g(X_T)
\end{alignat}
\end{subequations}
with  some Lipschitz continuous functions $f$ and $g$
specified as follows.
For notational simplicity, we shall denote by 
 $\cS^q(E)$ (resp.~$\cH^q(E)$)
 the space $\cS^q_\sF(\Om\t [t,T]; E)$ 
(resp.~$\cH^q_\sF(\Om\t [t,T]; E)$) for any $q\ge 2$,
$t\in[0,T]$  and Euclidean space $E$.

\begin{Lemma}\label{lemma:bsde_estimate_dx}
Let $f:[0,T]\t \sR^d\t  \sR^p\to \sR$ and $g:\sR^d\to \sR$ be continuous functions such that for all $t\in [0,T]$, $f(t,\cdot,\cdot)\in C^4( \sR^d\t  \sR^p)$
and $g\in C^4(\sR^d)$
with  bounded derivatives uniformly in $t$.
Let 
$Y: \sR^{d\t (d+p)}\t 
 [0,T]\t \sR^d
\mapsto \sR$
be defined in 
\eqref{eq:FK}.
Then for all 
bounded open set  $\cK\subset \sR^{d\t (d+p)}$,
there exists a constant $C>0$ such that 
for all $\theta\in \cK$,  the function 
$[0,T]\t \sR^d\ni(t,x)\mapsto Y^{\theta, t,x}_t\in \sR$ is of the class $ C^{1,3}([0,T]\t \sR^d)$,
and satisfies 
$|\p_{x_i} Y^{\theta, t,x}_t|+
|\p_{x_ix_j} Y^{\theta, t,x}_t|+
|\p_{x_ix_jx_l} Y^{\theta, t,x}_t|\le C$
for all $i,j,l\in \{1,\ldots, d\}$ and $(\theta,t,x)\in \cK\t  [0,T]\t \sR^d$.

\end{Lemma}

\begin{proof}
We shall assume 
 $d=p=1$ for notational simplicity, but the same arguments can be easily extended to a multidimensional setting. 
Throughout this proof, let $\cK$
be a given  bounded open set  of $\sR^{d\t (d+p)}$,
and  $C$ be a generic constant   independent of $(t,x,\theta)\in [0,T]\t \sR^d\t \cK$.

As $f(t,\cdot,\cdot)\in C^4$, $t\in [0,T]$,
and $g\in C^4$ have bounded derivatives, 
 by \cite[Theorem 3.2]{pardoux1992backward}, 
the function 
$[0,T]\t \sR\ni(t,x)\mapsto  Y^{\theta, t,x}_t\in \sR$ is of the class $ C^{1,2}([0,T]\t \sR)$,
while iterating the  arguments there again leads to the third-order differentiability in $x$. 
Hence, it suffices to establish the boundedness of the derivatives. 
By It\^{o}'s formula  and \cite[Lenama 2.5]{pardoux1992backward}, 
the gradient of $x\mapsto  Y^{\theta, t,x}_t$
can be identified as $\p_x Y^{\theta,t,x}_t$
for all $(t,x)\in [0,T]\t \sR$,
where
$(\p_x Y^{\theta, t,x},\p_x Z^{\theta, t,x})\in \cS^2( \sR)\t \cH^2(\sR)$ satisfies 
the following linear BSDE:
\begin{align}\label{eq:bsde_dx}
\begin{split}
\mathrm{d} \p_x Y^{\theta, t,x}_s
&=
\big( F^{\theta,t,x}_{x,s}  \p_x X^{\theta, t, x}_s
+
 F^{\theta,t,x}_{z,s} \p_x   Z^{\theta, t,x}_s
\big)
\,\d s
+ \p_x Z^{\theta, t,x}_s\, \d W_s,
\quad t\in [0,T),
\\
  \p_x Y^{\theta, t,x}_T&= (\nabla g)(X^{\theta, t,x}_T) \p_x X^{\theta, t, x}_T,
  \end{split}
\end{align}
with 
 $\p_x X^{\theta, t, x}_s=\exp(A(s-t))$
 for all $s\in [0,T]$,
$ F^{\theta,t,x}_{x,s}=
(\p_x f)(t,X^{\theta, t,x}_s,B Z^{\theta, t,x}_s)$
and 
$ F^{\theta,t,x}_{z,s}=(\p_z f)(t,X^{\theta, t,x}_s,B Z^{\theta, t,x}_s)B$ 
for all $s\in [t,T]$.
The a-priori estimate \cite[Theorem 4.4.4]{zhang2017backward}
shows that for all $q\ge 2$, there exists $C_q>0$ such that 
\begin{align*}
&\|\p_x   Y^{\theta, t,x}\|^q_{\cS^q}
+
\|\p_x   Z^{\theta, t,x}\|^q_{\cH^q}
\\
&\le C_q\bigg(\sE\bigg[| (\nabla g)(X^{\theta, t, x}_T) \p_x X^{\theta, t,x}_T|^q
+\Big(\int_0^T| F^{\theta,t,x}_{x,s}  \p_x X^{\theta, t, x}_s|\,\d t\Big)^q\bigg]\bigg)
\le C_q,
\end{align*}
due to the boundedness of $\nabla g$ and $\nabla f$,
which subsequently implies that 
$|\p_x Y^{\theta,t,x}_t|\le C$.
Taking the derivative again and using the fact that $\p_x  X^{\theta, t, x}$ is independent of $x$, 
we identify the gradient of $x\mapsto \p_x Y^{\theta, t,x}_t$
with  $\p_{xx} Y^{\theta,t,x}_t$
for all $(t,x)\in [0,T]\t \sR$,
where 
$(\p_{xx} Y^{\theta, t,x},\p_{xx} Z^{\theta, t,x})\in \cS^2( \sR)\t \cH^2(\sR)$ satisfying 
the following BSDE:
for all $ t\in [0,T)$,
\begin{align}\label{eq:bsde_dxx}
\begin{split}
\mathrm{d} \p_{xx} Y^{\theta, t,x}_s
&=
\big( F^{\theta,t,x}_{xx,s}  (\p_x X^{\theta, t, x}_s)^2
+2 F^{\theta,t,x}_{xz,s}  \p_x X^{\theta, t, x}_s\p_x   Z^{\theta, t,x}_s
+ F^{\theta,t,x}_{zz,s}    (\p_{x} Z^{\theta, t,x}_s)^2
\\
&\quad +
 F^{\theta,t,x}_{z,s} \p_{xx}   Z^{\theta, t,x}_s
\big)
\,\d s
+ \p_{xx} Z^{\theta, t,x}_s\, \d W_s,
\\
  \p_{xx} Y^{\theta, t,x}_T&= (\nabla^2 g)(X^{\theta, t,x}_T) (\p_x X^{\theta, t, x}_T)^2,
  \end{split}
\end{align}
with $\nabla^2 g$ being the second-order derivative of $g$, 
$ F^{\theta,t,x}_{xx,s}=
(\p_{xx} f)(t,X^{\theta, t,x}_s,B Z^{\theta, t,x}_s)$,
$ F^{\theta,t,x}_{xz,s}=
(\p_{xz} f)(t,X^{\theta, t,x}_s,B Z^{\theta, t,x}_s)B$
and 
$ F^{\theta,t,x}_{zz,s}=(\p_{zz} f)(t,X^{\theta, t,x}_s,B Z^{\theta, t,x}_s)B^2$.
Then for all $q\ge 2$,
\begin{align*}
&\|\p_{xx}   Y^{\theta, t,x}\|^q_{\cS^q}
+
\|\p_{xx}   Z^{\theta, t,x}\|^q_{\cH^q}
\\
&\le C_q\bigg(\sE\bigg[| (\nabla^2 g)(X^{\theta, t,x}_T) (\p_x X^{\theta, t, x}_T)^2|^q
\\
&\quad 
+\Big(\int_0^T|  F^{\theta,t,x}_{xx,s}  \p_x X^{\theta, t, x}_s
+2 F^{\theta,t,x}_{xz,s}  \p_x X^{\theta, t, x}_s\p_x   Z^{\theta, t,x}_s
+ F^{\theta,t,x}_{zz,s}    (\p_{x} Z^{\theta, t,x}_s)^2
|\,\d t\Big)^q\bigg]\bigg)
\\
&\le 
C_q(1+\|\p_x   Z^{\theta, t,x}\|_{\cH^{2q}}^{2q})\le C_q,
\end{align*}
which follows from the fact that $f,g$ have bounded derivatives and 
the   Cauchy--Schwarz inequality.
This proves that $x\mapsto Y^{\theta, t,x}_t$ has a bounded  second-order derivative.
Iterating the above argument again shows the boundedness of the third-order derivative of 
 $x\mapsto  Y^{\theta, t,x}_t$ and finishes the proof.
\end{proof}

\begin{Lemma}\label{lemma:bsde_estimate_dtheta}
Assume the same setting as in Lemma \ref{lemma:bsde_estimate_dx}.
Then 
for all $t\in [0,T]$,  the function 
$
\sR^{d\t (d+p)}\t  \sR^d 
\ni (\theta,x) \mapsto \p_x Y^{\theta, t,x}_t\in \sR^d$ is of the class $ C^{2}(  \sR^{d\t (d+p)}\t \sR^d )$.
Moreover, 
for all 
bounded open set  $\cK\subset \sR^{d\t (d+p)}$,
there exists a constant $C>0$ such that  
for all $i,j\in \{1,\ldots, d\}$, $(\theta,t)\in \cK\t  [0,T]$ and $x,x'\in\sR^d$,
\begin{enumerate}[(1)]
\item
$|\p_{x\theta_i} Y^{\theta, t,x}_t|+|\p_{xx\theta_i} Y^{\theta, t,x}_t|\le C(1+|x|)$
and
$|\p_{x\theta_i\theta_j} Y^{\theta, t,x}_t|\le C(1+|x|^2)$,
\item
$|\p_{x\theta_i} Y^{\theta, t,x}_t-\p_{x\theta_i} Y^{\theta, t,x'}_t|
+|\p_{xx\theta_i} Y^{\theta, t,x}_t-\p_{xx\theta_i} Y^{\theta, t,x'}_t|\le C(1+|x|+|x'|)|x-x'|$,
and 
$|\p_{x\theta_i\theta_j} Y^{\theta, t,x}_t
-\p_{x\theta_i\theta_j} Y^{\theta, t,x'}_t
|\le C(1+|x|^2+|x'|^2)|x-x'|$.
\end{enumerate}
%
\end{Lemma}

\begin{proof}

To simplify the presentation,
we  focus on establishing the desired properties for $A\mapsto \p_x Y^{(A,B),t,x}_t$,
since the regularity of  $B\mapsto \p_x Y^{(A,B),t,x}_t$ follows from  similar arguments. 
We also assume 
 $d=p=1$ and write 
 $A\mapsto \p_x Y^{(A,B),t,x}_t$
 as 
 $A\mapsto \p_x Y^{\theta,t,x}_t$, but the same arguments can be easily extended to a multidimensional setting. 
Throughout this proof, let $\cK$
be a given  bounded open set  of $\sR$,
and  $C$ be a generic constant  independent of $(t,x,A)\in [0,T]\t \sR^d\t \cK$.

By  \cite[Proposition 2.4]{el1997backward}, 
$\sR\ni A\mapsto (X^{\theta,t,x}, Y^{\theta,t,x},Z^{\theta,t,x})\in \cS^q(\sR)\t \cS^q(\sR)\t  \cH^q(\sR)$
is differentiable for all $q\ge 2$, where the  derivative satisfies
the following BSDE: for all $t\in[0,T)$
\begin{align}\label{eq:bsde_dA}
\begin{split}
\mathrm{d}\p_A X^{\theta,t,x}_s&=
(X^{\theta,t,x}_s+A\p_A X^{\theta,t,x}_s)\, \d s,
\\
\mathrm{d}\p_A Y^{\theta,t,x}_s&
=(F^{\theta,t,x}_{x,s}\p_A X^{\theta,t,x}_s+
F^{\theta,t,x}_{z,s}\p_A Z^{\theta,t,x}_s
)\,\d s
+\p_A Z^{\theta,t,x}_s\, \d W_s,
\\
\p_A X^{\theta,t,x}_t& =0,
\q \p_A Y^{\theta,t,x}_T=( \nabla g)(X^{\theta,t,x}_T)\p_A X^{\theta,t,x}_T,
\end{split}
\end{align}
with  bounded coefficients $F^{\theta,t,x}_{x,s}$
and $ F^{\theta,t,x}_{z,s}$
defined as in \eqref{eq:bsde_dx}.
Then
\begin{align}\label{eq:d_AYZ_bound}
\begin{split}
&\|\p_A   Y^{\theta, t,x}\|^q_{\cS^q}
+
\|\p_A  Z^{\theta, t,x}\|^q_{\cH^q}
\\
&\le C_q\bigg(\sE\bigg[| (\nabla g)(X^{\theta, t, x}_T) \p_A X^{\theta, t,x}_T|^q
+\Big(\int_0^T| F^{\theta,t,x}_{x,s}  \p_A X^{\theta, t, x}_s|\,\d t\Big)^q\bigg]\bigg)
\le C_q(1+|x|^q),
\end{split}
\end{align}
where the last inequality follows from the fact that 
$\|\p_A   X^{\theta, t,x}\|_{\cS^q}\le C_q(1+|x|)$. 
Applying  the Lipschitz estimate
in \cite[Theorem 2.3]{lionnet2015time}
to \eqref{eq:fbsde_general}
implies that 
$$
\| X^{\theta, t,x}-    X^{\theta, t,x'}\|_{\cS^q}
+
\| Y^{\theta, t,x}-    Y^{\theta, t,x'}\|_{\cS^q}
+\|  Z^{\theta, t,x}-    Z^{\theta, t,x'}\|_{\cH^q}\le C_q|x-x'|,
\quad \fa q\ge 2, x,x'\in \sR.
$$
Then by the Cauchy--Schwarz inequality, for all $q\ge 2$ and $x,x'\in \sR$,
\begin{align}\label{eq:Y_A_terminal}
\begin{split}
&\|( \nabla g)(X^{\theta,t,x}_T)\p_A X^{\theta,t,x}_T-( \nabla g)(X^{\theta,t,x'}_T)\p_A X^{\theta,t,x'}_T\|_{L^q}
\\
&\le 
\|(( \nabla g)(X^{\theta,t,x}_T)-( \nabla g)(X^{\theta,t,x'}_T))\p_A X^{\theta,t,x}_T\|_{L^q}
+\|
( \nabla g)(X^{\theta,t,x'}_T)(\p_A X^{\theta,t,x}_T-\p_A X^{\theta,t,x'}_T)\|_{L^q}
\\
&\le 
\|( \nabla g)(X^{\theta,t,x}_T)-( \nabla g)(X^{\theta,t,x'}_T)\|_{L^{2q}}\|\p_A X^{\theta,t,x}_T\|_{L^{2q}}
+\|
\p_A X^{\theta,t,x}_T-\p_A X^{\theta,t,x'}_T\|_{L^q}
\\
&\le C(1+|x|+|x'|)|x-x'|,
\end{split}
\end{align}
where the last inequality used 
$\|
\p_A X^{\theta,t,x}-\p_A X^{\theta,t,x'}\|_{\cS^q}
\le C_q|x-x'|
$.
Similarly, we have 
$$
\|F^{\theta,t,x}_{x,\cdot}\p_A X^{\theta,t,x}_\cdot
-F^{\theta,t,x'}_{x,\cdot}\p_A X^{\theta,t,x'}_\cdot
+(F^{\theta,t,x}_{z,\cdot}-F^{\theta,t,x'}_{z,\cdot})\p_A Z^{\theta,t,x}_\cdot
\|_{\cH^q}\le C(1+|x|+|x'|)|x-x'|.
$$
Consequently, 
applying \cite[Theorem 2.3]{lionnet2015time}
to  \eqref{eq:bsde_dA} shows that for all $q\ge 2$,
$
\|\p_A Y^{\theta,t,x}-\p_A Y^{\theta,t,x'}\|_{\cS^q}  
+\|\p_A Z^{\theta,t,x}-\p_A Z^{\theta,t,x'}\|_{\cH^q} 
\le C(1+|x|+|x'|)|x-x'|.
 $

Now we   study the  differentiability of $A\mapsto \p_x Y^{\theta,t,x}_t$.
Observe that $A\mapsto \p_x X^{\theta,t,x}\in L^q(0,T)$
is differentiable for all $q\ge 2$,
and the   derivatives are  bounded 
uniformly in $(A,t,x)\in \cK\t  [0,T]\t \sR$. 
By \cite[Theorem 9, p.~97]{krylov2008controlled},  for all $q\ge 2$ and $A\in \sR$,
\begin{align*}
\sR\ni A' &\mapsto 
\big(
(\nabla g)(X^{\theta', t,x}_T) \p_x X^{\theta', t, x}_T,
 F^{\theta',t,x}_{x,\cdot}  \p_x X^{\theta', t, x}_\cdot
 + F^{\theta',t,x}_{z,s} \p_x   Z^{\theta, t,x}_\cdot
 \big)\in L^q(\Om)\t  \cH^q(\sR)
\end{align*}
is differentiable.
Hence,
applying \cite[Proposition 2.4]{el1997backward} to \eqref{eq:bsde_dx}
shows that 
$$\sR\ni A \mapsto 
(\p_x   Y^{\theta, t,x},\p_x   Z^{\theta, t,x})\in \cS^q(\sR)\t \cH^q(\sR)
$$  
is differentiable,
and the  derivatives
$(\p_{xA}   Y^{\theta, t,x},\p_{xA}   Z^{\theta, t,x})$ satisfies the BSDE:
\begin{align}\label{eq:d_xaYZ}
\begin{split}
\mathrm{d} \p_{xA} Y^{\theta, t,x}_s
&=
\big( F^{\theta,t,x}_{xa,s}  \p_x X^{\theta, t, x}_s
+F^{\theta,t,x}_{x,s}  \p_{xA} X^{\theta, t, x}_s
+
 F^{\theta,t,x}_{za,s} \p_x   Z^{\theta, t,x}_s
 +F^{\theta,t,x}_{z,s} \p_{xA}   Z^{\theta, t,x}_s
\big)
\,\d s
\\
&\quad 
+ \p_{xA} Z^{\theta, t,x}_s\, \d W_s,
\quad t\in [0,T),
\\
  \p_{xA} Y^{\theta, t,xa}_T&= (\nabla^2 g)(X^{\theta, t,x}_T)\p_A X^{\theta, t,x}_T \p_x X^{\theta, t, x}_T
  +(\nabla g)(X^{\theta, t,x}_T)\p_{xA} X^{\theta, t, x}_T,
  \end{split}
\end{align}
with the coefficients 
\begin{align*}
F^{\theta,t,x}_{xa,s} &=
(\p_{xx} f)(t,X^{\theta, t,x}_s,B Z^{\theta, t,x}_s)\p_A X^{\theta, t,x}_s
+(\p_{xz} f)(t,X^{\theta, t,x}_s,B Z^{\theta, t,x}_s)B\p_A Z^{\theta,t,x}_s,
\\
 F^{\theta,t,x}_{za,s}&=(\p_{zx} f)(t,X^{\theta, t,x}_s,B Z^{\theta, t,x}_s)B\p_A X^{\theta, t,x}_s
+(\p_{zz} f)(t,X^{\theta, t,x}_s,B Z^{\theta, t,x}_s)B^2\p_A Z^{\theta,t,x}_s.
\end{align*}

For the  growth rate of $x\mapsto \p_{xA} Y^{\theta,t,x}_t$,
by the a priori estimate 
 for \eqref{eq:d_xaYZ},
for all $q\ge 2$,
\begin{align*}
&\|\p_{xA} Y^{\theta,t,x}\|^q_{\cS^q}
+\|\p_{xA} Z^{\theta,t,x}\|^q_{\cH^q}
\\
&\le C_q\bigg(\sE\bigg[|  (\nabla^2 g)(X^{\theta, t,x}_T)\p_A X^{\theta, t,x}_T \p_x X^{\theta, t, x}_T
  +(\nabla g)(X^{\theta, t,x}_T)\p_{xA} X^{\theta, t, x}_T|^q
\\
&\quad
+\Big(\int_0^T|  F^{\theta,t,x}_{xa,s}  \p_x X^{\theta, t, x}_s
+F^{\theta,t,x}_{x,s}  \p_{xA} X^{\theta, t, x}_s
+
 F^{\theta,t,x}_{za,s} \p_x   Z^{\theta, t,x}_s
|\,\d t\Big)^q\bigg]\bigg)
\\
& 
\le C_q\bigg( 1+|x|^q+\sE\bigg[\Big(\int_0^T
(|    \p_A X^{\theta, t, x}_s|+| \p_A Z^{\theta, t, x}_s|)(1+|\p_x   Z^{\theta, t,x}_s|)
\,\d t\Big)^q
\bigg]\bigg)
\\
&
\le C_q(1+|x|^q+(\|    \p_A X^{\theta, t, x}\|_{\cH^{2q}}^q+\| \p_A Z^{\theta, t, x}\|_{\cH^{2q}}^q)
(1+\|\p_x   Z^{\theta, t,x}\|_{\cH^{2q}}^q)
\le  C_q(1+|x|^q),
\end{align*}
which  follows from the    Cauchy--Schwarz inequality,  \eqref{eq:d_AYZ_bound} and the fact that $\|\p_x   Z^{\theta, t,x}\|_{\cH^{q}}\le C_{q}$ for all $q\ge 2$.
To prove the local Lipschitz continuity of 
$x\mapsto \p_{xA} Y^{\theta,t,x}_t$, we introduce 
for all $q\ge 2$ and $x,x'\in \sR^d$, the random variables 
\begin{align*}
G^{\theta,t,x}
&\coloneqq (\nabla^2 g)(X^{\theta, t,x}_T)\p_A X^{\theta, t,x}_T \p_x X^{\theta, t, x}_T
  +(\nabla g)(X^{\theta, t,x}_T)\p_{xA} X^{\theta, t, x}_T,
  \\
H^{\theta,t,x',x}_{s}
&\coloneqq 
  F^{\theta,t,x'}_{xa,s}  \p_x X^{\theta, t, x'}_s
+F^{\theta,t,x'}_{x,s}  \p_{xA} X^{\theta, t, x'}_s
+
 F^{\theta,t,x'}_{za,s} \p_x   Z^{\theta, t,x'}_s
  +F^{\theta,t,x'}_{z,s} \p_{xA}   Z^{\theta, t,x}_s,
  \q s\in [t,T],
  \end{align*}
  By following similar arguments as in \eqref{eq:Y_A_terminal}
  and 
using the regularity of $f$ and $g$, 
we have  for all $q\ge 2$ and $x,x'\in \sR^d$, 
$$
\|G^{\theta,t,x}-G^{\theta,t,x'}\|_{L^q}
+\|H^{\theta,t,x,x}-H^{\theta,t,x',x}\|_{\cH^q}
\le C_q(1+|x|+|x'|)|x-x'|,
$$
 which along with the Lipschitz estimate for \eqref{eq:d_xaYZ}
shows 
\begin{align*}
&\|\p_{xA} Y^{\theta,t,x}-\p_{xA} Y^{\theta,t,x'}\|_{\cS^q}
+\|\p_{xA} Z^{\theta,t,x}-\p_{xA} Z^{\theta,t,x'}\|_{\cH^q}
\le C_q(1+|x|+|x'|)|x-x'|.
\end{align*}

For  the  differentiability of $A\mapsto \p_{xA} Y^{\theta,t,x}_t$,
applying \cite[Proposition 2.4]{el1997backward} to \eqref{eq:d_xaYZ}
implies that 
$\sR\ni A \mapsto 
(\p_{xA}   Y^{\theta, t,x},\p_{xA}   Z^{\theta, t,x})\in \cS^q(\sR)\t \cH^q(\sR)$  is differentiable,
and its derivative satisfies the BSDE:
for all  $t\in [0,T)$,
\begin{align}\label{eq:d_xaaYZ}
\begin{split}
\mathrm{d} \p_{xAA} Y^{\theta, t,x}_s
&=
\big( F^{\theta,t,x}_{xaa,s}  \p_x X^{\theta, t, x}_s
+2F^{\theta,t,x}_{xa,s}  \p_{xA} X^{\theta, t, x}_s
+F^{\theta,t,x}_{x,s}  \p_{xAA} X^{\theta, t, x}_s
+
 F^{\theta,t,x}_{zaa,s} \p_x   Z^{\theta, t,x}_s
\\
&\quad 
 +2 F^{\theta,t,x}_{za,s} \p_{xA}   Z^{\theta, t,x}_s
 +F^{\theta,t,x}_{z,s} \p_{xAA}   Z^{\theta, t,x}_s
\big)
\,\d s
+ \p_{xAA} Z^{\theta, t,x}_s\, \d W_s,
\\
  \p_{xAA} Y^{\theta, t,x}_T&= 
  (\nabla^3 g)(X^{\theta, t,x}_T)(\p_A X^{\theta, t,x}_T)^2 \p_x X^{\theta, t, x}_T
  +(\nabla^2 g)(X^{\theta, t,x}_T)\p_{AA} X^{\theta, t,x}_T \p_x X^{\theta, t, x}_T
  \\
  &\quad
   +2(\nabla^2 g)(X^{\theta, t,x}_T)\p_A X^{\theta, t,x}_T\p_{xA} X^{\theta, t, x}_T
 +(\nabla g)(X^{\theta, t,x}_T)\p_{xAA} X^{\theta, t, x}_T,
  \end{split}
\end{align}
with the coefficients 
\begin{align*}
F^{\theta,t,x}_{xaa,s} &=
(\p_{xxx} f)(t,X^{\theta, t,x}_s,B Z^{\theta, t,x}_s)(\p_A X^{\theta, t,x}_s)^2
+
2(\p_{xxz} f)(t,X^{\theta, t,x}_s,B Z^{\theta, t,x}_s)\p_A X^{\theta, t,x}_sB\p_A Z^{\theta, t,x}_s
\\
&\quad
+
(\p_{xx} f)(t,X^{\theta, t,x}_s,B Z^{\theta, t,x}_s)\p_{AA} X^{\theta, t,x}_s
+(\p_{xzz} f)(t,X^{\theta, t,x}_s,B Z^{\theta, t,x}_s)(B\p_A Z^{\theta,t,x}_s)^2
\\
&\quad +(\p_{xz} f)(t,X^{\theta, t,x}_s,B Z^{\theta, t,x}_s)B\p_{AA} Z^{\theta,t,x}_s,
\\
 F^{\theta,t,x}_{zaa,s}&
 =(\p_{zxx} f)(t,X^{\theta, t,x}_s,B Z^{\theta, t,x}_s)B(\p_A X^{\theta, t,x}_s)^2
 +2(\p_{zxz} f)(t,X^{\theta, t,x}_s,B Z^{\theta, t,x}_s)B^2\p_A X^{\theta, t,x}_s\p_A Z^{\theta, t,x}_s
  \\
&\quad
 +(\p_{zx} f)(t,X^{\theta, t,x}_s,B Z^{\theta, t,x}_s)B\p_{AA} X^{\theta, t,x}_s
 +(\p_{zzz} f)(t,X^{\theta, t,x}_s,B Z^{\theta, t,x}_s)B^3(\p_A Z^{\theta,t,x}_s)^2
  \\
&\quad
+(\p_{zz} f)(t,X^{\theta, t,x}_s,B Z^{\theta, t,x}_s)B^2\p_{AA} Z^{\theta,t,x}_s.
\end{align*}
Repeating  all the above estimates, one can show 
$\| \p_{xAA} Y^{\theta, t,x}\|_{\cS^q}+\| \p_{xAA} Z^{\theta, t,x}\|_{\cH^q}\le C_q(1+|x|^2)$, 
and 
$\| \p_{xAA} Y^{\theta, t,x}-\p_{xAA} Y^{\theta, t,x'}\|_{\cS^q}+\| \p_{xAA} Z^{\theta, t,x}-\p_{xAA} Z^{\theta, t,x'}\|_{\cH^q}\le C_q(1+|x|^2+|x'|^2)|x-x'|$.
Other partial derivatives of $\theta\mapsto \p_x Y^{\theta,t,x}_t$ can be estimated by using similar arguments, 
whose details are omitted here. %
\end{proof}

Based on Lemmas 
\ref{lemma:bsde_estimate_dx} and \ref{lemma:bsde_estimate_dtheta},
we prove the entropy regularised control problem 
\eqref{eq:lc} admits an optimal feedback control, 
which is sufficiently regular in the $(\theta,x)$-variables.

\begin{Proposition}\label{prop:feedback_control_entropy}
Assume the same setting as in Theorem  \ref{thm:performance_gap_entropy}.
Then for all $\theta\in \sR^{d\t (d+p)}$,
the control problem \eqref{eq:lc} with parameter 
$\theta$ admits an optimal feedback control $\psi_\theta\in \cup_{C\ge 0}\cV_C$, and 
\begin{enumerate}[(1)]
\item \label{item:psi_theta_differentiable}
 $\psi_\theta\in C^{0,2}([0,T]\t \sR^d)$ 
and 
for all $t\in [0,T]$, 
the function $
 \sR^{d\t (d+p)}\t \sR^d
\ni (\theta, x) \mapsto \psi_\theta(t,x)\in  \sR^{p}$ is of the class $ C^{2}(  \sR^{d\t (d+p)}\t \sR^d)$,
\item  \label{item:psi_theta_growth}
for all bounded open set  $\cK\subset \sR^{d\t (d+p)}$,
there exists a constant $C>0$ such that 
for all  $(\theta,t)\in \cK\t  [0,T]$ and $x,x'\in \sR^d$,
$|\nabla_x(\psi_\theta)(t,x)|+|\textnormal{Hess}_x(\psi_\theta)(t,x)|\le C$,
$|\nabla_\theta(\psi_\theta)(t,x)|
+|\nabla_\theta\big(\nabla_x(\psi_\theta)\big)(t,x)|
\le C(1+|x|)$, 
$|\textnormal{Hess}_\theta(\psi_\theta)(t,x)|\le C(1+|x|^2)$,
$|\nabla_\theta(\psi_\theta)(t,x)-\nabla_\theta(\psi_\theta)(t,x')|
+|\nabla_\theta\big(\nabla_x(\psi_\theta)\big)(t,x)-\nabla_\theta\big(\nabla_x(\psi_\theta)\big)(t,x')|
\le C(1+|x|+|x'|)|x-x'|$, 
and 
$|\textnormal{Hess}_\theta(\psi_\theta)(t,x)-\textnormal{Hess}_\theta(\psi_\theta)(t,x')|\le C(1+|x|^2+|x'|^2)|x-x'|$.

\end{enumerate}

\end{Proposition}

\begin{proof}
Let
 $h^*_{\textrm{en}}:\sR^p\to \sR$ be the convex conjugate of $h_{\textrm{en}}$ 
such that for all $z=(z_i)_{i=1}^p\in \sR^p$, $h^*_{\textrm{en}}(z)=\sup_{a\in \sR^p}(\la a,z\ra -h_{\textrm{en}}(a))=\ln\sum_{i=1}^p\exp(z_i)$.
A direct computation shows that $h^*_{\textrm{en}}\in C^\infty(\sR^p)$, 
$|h^*_{\textrm{en}}(z)|\le C(1+|z|)$ for all $z\in \sR^p$,
and  all derivatives of $h^*_{\textrm{en}}$ are bounded. 
For given  
$\theta=(A,B)\in \sR^{d\t (d+p)}$ and 
$(t,x)\in [0,T]\t \sR^d$,
consider  the following FBSDE:
for all $s\in [t,T]$,%
\begin{subequations}\l{eq:fbsde_feynmann}
\begin{alignat}{2}
\mathrm{d}X_s&=
AX_s\, \d s
+\, \d W_s,
&&
\q X_t=x,
\l{eq:lc_fbsde_fwd}
\\
\mathrm{d}Y_s&
=h^*_{\textrm{en}}(-B^\top Z^\top_s -\bar{f}_0(s,X_s))\,\d s
+Z_s\, \d W_s,
&&
\q Y_T= g(X_T).
\end{alignat}
\end{subequations}
By the Lipschitz continuity of  $\nabla h^*_{\textrm{en}}$ and $\bar{f}_0$,
\eqref{eq:fbsde_feynmann} admits a unique solution  $(X^{\theta,t,x},Y^{\theta,t,x},Z^{\theta,t,x})\in \cS^2(\sR^d)\t \cS^2(\sR)\t \cH^2(\sR^{1\t d})$
 (see e.g., \cite[Theorems 4.3.1]{zhang2017backward}). 
By
 Lemma \ref{lemma:bsde_estimate_dx}
 and the regularity of $h^*_{\textrm{en}}, \bar{f}_0$ and $g$,
 for all  $\theta\in \sR^{d\t (d+p)}$,
the function 
$[0,T]\t \sR^d\ni(t,x)\mapsto  V^\theta(t,x)\coloneqq Y^{\theta, t,x}_t\in \sR$ is of the class $ C^{1,3}([0,T]\t \sR^d)$,
and for all $(t,x)\in [0,T]\t \sR^d$, 
$
 \sR^{d\t (d+p)}
\ni \theta \mapsto \nabla_x V^\theta(t,x)\in  \sR^{d}$ is of the class $ C^{2}( \sR^{d\t (d+p)})$.
Moreover, for any bounded open subset $\cK\subset  \sR^{d\t (d+p)}$,
there exists  $C>0$ such that 
for all  $(\theta,t,x)\in \cK\t  [0,T]\t \sR^d$,
$|\nabla_x(\nabla_x V^\theta)(t,x)|\le C$.

By It\^{o}'s formula,  for all  $\theta\in \sR^{d\t (d+p)}$,
the function $V^\theta$ 
 solves the following PDE:
for all $(t,x)\in [0,T)\t \sR^d$,
\bb\label{eq:hjb}
\tfrac{\d }{\d t}V(t,x)+\tfrac{1}{2}\Delta_x V(t,x)+
\la Ax, \nabla_xV(t,x) \ra
-
h^*_{\textrm{en}}(-B^\top \nabla_xV(t,x)  -\bar{f}_0(t,x))=0,
\ee
and $V(T,x)=g(x)$ for all $x\in \sR^d$ (see e.g.,
\cite[Theorem 5.5.8]{zhang2017backward}).
By 
\eqref{eq:entropy_regularized_proof}
and the definition of $h^*_{\textrm{en}}$,
 \eqref{eq:hjb} can be equivalently written as:
$$
\tfrac{\d }{\d t}V(t,x)+\tfrac{1}{2}\Delta_x V(t,x)+
\inf_{a\in \sR^p}
\big(\la Ax+Ba, \nabla_xV(t,x) \ra+f(t,x,a)\big)=0,
$$
 which is 
  the Hamilton-Jacobi-Bellman equation for  \eqref{eq:lc} with parameter $\theta$.
Now let 
 $\psi_\theta:[0,T]\t \sR^d\to \Delta_p$ be the function such that for all $(t,x)\in [0,T]\t \sR^d$,
\bb\label{eq:feedback_entropy}
\psi_\theta(t,x)
=\nabla h^*_{\textrm{en}}(-B^\top \nabla_xV^\theta(t,x) -\bar{f}_0(t,x))
=\argmin_{a\in \Delta_p} \big(\la Ba,\nabla_xV^\theta(t,x) \ra +f(t,x,a)\big).
\ee
As $\nabla_x V^\theta$ has a bounded  first-order derivative  (and hence Lipschitz continuous) in $x$,
the  verification theorem in \cite[Theorem 5.1, p.~268]{yong1999stochastic}
shows  that $\psi_\theta$ is the optimal feedback control of  \eqref{eq:lc} with parameter $\theta$.
The regularity of $h^*_{\textrm{en}}$ and $\bar{f}_0$ implies that 
$\psi_\theta$ 
shares the same regularity of $\nabla_x  V^\theta$.
Hence 
by Lemmas \ref{lemma:bsde_estimate_dx}
and \ref{lemma:bsde_estimate_dtheta},
 $\psi_\theta$ admits the desired properties as stated in
Items 
\ref{item:psi_theta_differentiable} and \ref{item:psi_theta_growth}.
\end{proof}

\begin{Lemma}\label{lemma:state_C2}
Assume the same setting as in Theorem  \ref{thm:performance_gap_entropy}.
Let 
 $\theta_0=(A_0,B_0)\in \sR^{d\t (d+p)}$,
 and 
for each $\theta\in  \sR^{d\t (d+p)}$,
let $\psi_\theta:[0,T]\t \sR^d\to \sR^p$ be the feedback control in Proposition \ref{prop:feedback_control_entropy},
and let $X^{\theta_0,\psi_\theta}\in \cS^2(\sR^d)$ be the solution to the following dynamics:
$$
\d X_t=(A_0 X_t+B_0\psi_\theta(t,X_t))\d t+\d W_t,
\q t\in [0,T], \quad X_0=x_0.
$$
Then for all $q\ge 2$,  $  \sR^{d\t (d+p)}\ni \theta\mapsto X^{\theta_0,\psi_\theta}
\in \cS^q(\sR^d)$ is twice continuously differentiable.

\end{Lemma}
\begin{proof}
As $\theta_0$ is fixed throughout the proof, 
we write $X^{\theta_0,\psi_\theta}$
as $X^{\psi_\theta}$ for notational simplicity. 
Consider the function 
$$ \sR^{d\t (d+p)}
\t [0,T]\t \sR^d
\ni (\theta,t,x)\mapsto
b_\theta(t,x)\coloneqq A_0 x+B_0\psi_\theta(t,x)\in \sR^d.$$
By Proposition 
\ref{prop:feedback_control_entropy}
Item \ref{item:psi_theta_differentiable},
for all $t\in [0,T]$, $b_\theta(t,x)$ is twice differentiable with respect to $(\theta,x)$. 
By Proposition 
\ref{prop:feedback_control_entropy} 
Item \ref{item:psi_theta_growth},
for any bounded open set  $\cK\subset \sR^{d\t (d+p)}$,
the first- and second-order derivatives of $b_\theta(t,x)$ with respect to $(\theta,x)$
have at most quadratic growth in $x$, uniformly in $(\theta,t )\in  \cK\t [0,T]$.
Hence by 
\cite[Theorem 4, p.~105]{krylov2008controlled},
$\sR^{d\t (d+p)}\ni\theta\mapsto X^{\psi_\theta}\in \cS^q(\sR^d)$ is twice differentiable for all $q\ge 2$.

It remains to  establish the continuity of the second-order derivative
$\sR^{d\t (d+p)}\ni \theta\mapsto X^{\psi_\theta}\in \cS^q(\sR^d)$,
for which
we   assume without loss of generality  that 
 $\theta\in \sR$ and $X^{\psi_\theta}$ is one-dimensional. By \cite[Theorem 4, p.~105]{krylov2008controlled},
the second-order derivative
$\p_{\theta\theta} X^{\psi_\theta}$
 of $\sR\ni\theta\mapsto X^{\psi_\theta}$ satisfies the following SDE:
$\p_{\theta\theta} X^{\psi_\theta}_0=0$, and for all $t\in (0,T]$,
\begin{align*}
\d \p_{\theta\theta} X^{\psi_\theta}_t&=\Big(
B_0 (\p_{\theta\theta}\psi_\theta)(t,X^{\psi_\theta}_t)
+2B_0 (\p_{\theta x}\psi_\theta)(t,X^{\psi_\theta}_t)\p_\theta X^{\psi_\theta}_t
+B_0(\p_{xx}\psi_\theta)(t,X^{\psi_\theta}_t)(\p_\theta X^{\psi_\theta}_t)^2
\\
&\quad 
+\big(A_0 +
B_0 (\p_x\psi_\theta)(t,X^{\psi_\theta}_t)\big)\p_{\theta\theta} X^{\psi_\theta}_t
\Big)\d t.
\end{align*}
The twice differentiability of $\theta\mapsto X^{\psi_\theta}$
implies that 
$\sR\ni\theta\mapsto (X^{\psi_\theta},\p_\theta X^{\psi_\theta})\in \cS^q(\sR)\t  \cS^q(\sR)$
is continuous for all $q\ge 2$.
Now let  $\cK\subset \sR$ be an arbitrary open subset.
By Proposition 
\ref{prop:feedback_control_entropy}
Item \ref{item:psi_theta_growth},
for all $\theta,\eta\in \cK$ and  $q\ge 2$,
\begin{align*}
&\sE\bigg[\int_0^T|(\p_{\theta\theta}\psi_\eta)(t,X^{\psi_\eta}_t)-(\p_{\theta\theta}\psi_\theta)(t,X^{\psi_\theta}_t)|^q\, \d t
\bigg]
\\
&\le C_q
\sE\bigg[\int_0^T|(\p_{\theta\theta}\psi_\eta)(t,X^{\psi_\eta}_t)-(\p_{\theta\theta}\psi_\eta)(t,X^{\psi_\theta}_t)|^q
+|(\p_{\theta\theta}\psi_\eta)(t,X^{\psi_\theta}_t)-(\p_{\theta\theta}\psi_\theta)(t,X^{\psi_\theta}_t)|^q\, \d t
\bigg]
\\
&\le C_q
\sE\bigg[\int_0^T
(1+|X^{\psi_\eta}_t|^2+|X^{\psi_\theta}_t|^2)^q|X^{\psi_\eta}_t-X^{\psi_\theta}_t|^q
+|(\p_{\theta\theta}\psi_\eta)(t,X^{\psi_\theta}_t)-(\p_{\theta\theta}\psi_\theta)(t,X^{\psi_\theta}_t)|^q\, \d t
\bigg]
\\
&\le C_q\bigg(
(1+\|X^{\psi_\eta}\|_{\cS^{4q}}^{2q}+\|X^{\psi_\theta}\|_{\cS^{4q}}^{2q})
\|X^{\psi_\eta}-X^{\psi_\theta}\|_{\cS^{2q}}^q
\\
&\quad +
\sE\bigg[\int_0^T
|(\p_{\theta\theta}\psi_\eta)(t,X^{\psi_\theta}_t)-(\p_{\theta\theta}\psi_\theta)(t,X^{\psi_\theta}_t)|^q\, \d t
\bigg]\bigg).
\end{align*}
Then, as $\eta$ tends to $\theta$, the first term converges to zero due to the continuity of $\theta\mapsto X^{\psi_\theta}\in \cS^{4q}(\sR)$, and the second terms also converges to zero due to the continuity of $\p_{\theta\theta}\psi_\theta$ in $\theta$, the fact that $|(\p_{\theta\theta}\psi_\theta)(t,x)|\le C(1+|x|^2)$ and  the dominated convergence theorem.
Similar arguments show that 
$\sR\ni \theta\mapsto (\p_{\theta x}\psi_\theta)(\cdot,X^{\psi_\theta}_\cdot), (\p_{xx}\psi_\theta)(\cdot,X^{\psi_\theta}_\cdot),
(\p_x\psi_\theta)(\cdot,X^{\psi_\theta}_\cdot)\in \cH^q(\sR)$
are continuous for all $q\ge 2$. 
Thus  \cite[Theorem 9, p.~97]{krylov2008controlled} shows that for all $q\ge 2$ and $x\in \sR$,
\begin{align*}
\sR\ni \theta\mapsto 
&B_0 (\p_{\theta\theta}\psi_\theta)(\cdot,X^{\psi_\theta}_\cdot)
+2B_0 (\p_{\theta x}\psi_\theta)(\cdot,X^{\psi_\theta}_\cdot)\p_\theta X^{\psi_\theta}_\cdot
+B_0 (\p_{xx}\psi_\theta)(\cdot,X^{\psi_\theta}_\cdot)(\p_\theta X^{\psi_\theta}_\cdot)^2
\\
&\quad +\big(A_0  +B_0 
(\p_x\psi_\theta)(\cdot,X^{\psi_\theta}_\cdot)\big) x \in \cH^q(\sR)
\end{align*}
is continuous, 
which along with \cite[Corollary 2, p.~103]{krylov2008controlled}
implies the continuity  of 
$\cK\ni \theta\mapsto \p_{\theta\theta} X^{\theta_0,\psi_\theta}\in \cS^q(\sR)$ for all $q\ge 2$.
\end{proof}

\begin{proof}[Proof of Theorem \ref{thm:performance_gap_entropy}]
For every $\theta\in  \sR^{d\t (d+p)}$,
the existence of optimal feedback control $\psi_\theta$ has been established in Proposition \ref{prop:feedback_control_entropy}.
Hence it remains to establish the quadratic performance gap. 
To this end, let 
$\theta_0=(A_0,B_0)\in \Theta$
and $\beta>0$ be given constants,
and   $C\ge 0$ be a generic constant  which
depends on $\Theta,\b$
but
is independent of $\theta_0$.
For each $\theta\in \sB_\b (\theta_0)$,
  consider  the 
 dynamics:
$$
\d X_t=(A_0 X_t+B_0\psi_\theta(t,X_t))\d t+\d W_t,
\q t\in [0,T], \quad X_0=x_0,
$$
which admits a strong solution  $X^{\theta_0,\psi_\theta}\in \cS^q(\sR^d)$ for all $q\ge 2$, due to the Lipschitz continuity of $\psi_\theta$. 
Lemma \ref{lemma:state_C2} shows that 
 $\sR^{d\t (d+p)} \ni\theta\mapsto X^{\theta_0,\psi_\theta}\in \cS^q(\sR^d)$ is twice continuously differentiable.

 We start by proving 
the twice continuous differentiability of the  map
$
\sR^{d\t (d+p)}\ni\theta\mapsto
J(\psi_\theta;\theta_0)\in \sR
$, with 
$J(\psi_\theta;\theta_0)$ being defined as in 
\eqref{eq:loss_feedback}.
Applying  the Fenchel-Young identity to   $h^*_{\textrm{en}}$
gives that $h^*_{\textrm{en}}(x)+h_{\textrm{en}}((\nabla h^*_{\textrm{en}})(x))=\la x, (\nabla h^*_{\textrm{en}})(x)\ra $
for all $x\in \sR^p$. Hence, by \eqref{eq:entropy_regularized_proof} and
\eqref{eq:feedback_entropy},  \begin{align*}
f(t,x,{\psi_\theta}(t,x))
&=\la \bar{f}_0(t,x), {\psi_\theta}(t,x)\ra+h_{\textrm{en}}\big((\nabla h^*_{\textrm{en}})(-B^\top \nabla_xV^\theta(t,x) -\bar{f}_0(t,x))\big)
\\
&=\la \bar{f}_0(t,x), {\psi_\theta}(t,x)\ra +\la -B^\top \nabla_xV^\theta(t,x) -\bar{f}_0(t,x),{\psi_\theta}(t,x)\ra
\\
&\quad -h^*_{\textrm{en}}\big(-B^\top \nabla_xV^\theta(t,x) -\bar{f}_0(t,x)\big)
\\
&=-\la B^\top \nabla_xV^\theta(t,x) ,{\psi_\theta}(t,x)\ra-h^*_{\textrm{en}}\big(-B^\top \nabla_xV^\theta(t,x) -\bar{f}_0(t,x)\big).
\end{align*}
  Proposition \ref{prop:feedback_control_entropy}
shows  
for all $t\in [0,T]$,
$\nabla_xV^\theta(t,x)$ and ${\psi_\theta}(t,x)$ is twice differentiable in $(\theta,x)$. Moreover, the derivatives have at most quadratic growth in  $x$ and is locally Lipschitz continuous in $x$,
provided that $\theta\in \sB_\b(\theta_0)$. Hence, 
by  \cite[Theorem 9, p.~97]{krylov2008controlled}, for all $q\ge 2$,
$\sR^{d\t (d+p)}\ni \theta\mapsto 
(
\nabla_xV^\theta(\cdot,X^{\theta_0,\psi_\theta}_\cdot), \psi_\theta(\cdot,X^{\theta_0,\psi_\theta}_\cdot)
)
\in \cH^q(\sR^d)\t \cH^q(\sR^p)$ are twice continuously differentiable, 
which implies  the twice continuous differentiability of the following function:
$$
\sR^{d\t (d+p)}\ni\theta\mapsto
\sE\bigg[\int_0^T f(t,X^{\theta_0,\psi_\theta}_t,{\psi_\theta}(t,X^{\theta_0,\psi_\theta}_t))\, \d t\bigg]\in \sR.
$$
On the other hand, as $g\in C^4$ has bounded derivatives,  \cite[Theorem 9, p.~97]{krylov2008controlled} shows that 
for all $q\ge 2$,
$\sR^{d\t (d+p)}\ni \theta\mapsto g(X^{\theta_0, \psi^{\theta}}_T)\in L^q(\Om;\sR)$  is twice continuously differentiable,
which implies the same regularity of  $\sR\ni \theta\mapsto \sE[g(X^{\theta_0, \psi^{\theta}}_T)]\in \sR$.
 This finishes the proof of the statement that $\theta\mapsto J(\psi_\theta;\theta_0)$ is $C^2$.
 
 Observe that the function $\sR^{d\t (d+p)}\ni \theta\mapsto J(\psi_\theta;\theta_0)\in \sR$
is mininised at  $\theta=\theta_0$,
  which implies that  $\nabla_{\theta}J(\psi_\theta;\theta_0)\vert_{\theta=\theta_0}=0$. 
  Applying  Taylor's theorem  up to the second-order terms gives 
  the desired  quadratic performance gap.
\end{proof}

\section{Proofs of Theorems \ref{Theorem: regret}
and \ref{Theorem: regret_self explore}}
\label{sec:proof_regret_bound}

For the sake of presentation, 
we introduce the following notation  to identify the correspondence between the number of cycles and episodes used in   the PEGE algorithm (see Algorithm \ref{Alg: PEGE}). 
Let  $\cC: \sN\cup\{0\} \to \sN\cup\{0\}$ be the total number of  episodes after completing a given cycle satisfying 
{$\cC(K) \coloneqq K + \sum_{k=1}^K \mathfrak{m}(k)$}
 for all $K\in \sN\cup\{0\}$,
 $\kappa: \sN \to \sN$ be the corresponding cycle of a given  episode satisfying   $\kappa(m) \coloneqq \min\{k \in \sN : \cC(k)\ge m\}$ for all $m\in \sN$,
 and  
let   $\cE \coloneqq \{m^e_k\mid 
m^e_k =  \cC(k-1)+1, k\in \sN\}$
 be the collection  of all exploration  episodes.
The definition of $\kappa$ implies that for all $m\in \sN$, $\cC(\kappa(m)-1)< m\le \cC(\kappa(m))$. In particular,
there exist $\ul{c},\ol{c}\ge 0$ such that
 for all $m\in \sN$, 
$\underline{c} m \leq \kappa(m)^{1+r} \leq \bar{c} m$
{if $\mathfrak{m}(k) = \lf k^r \rf$ for all $k$,
and $\underline{c} m \leq 2^{\kappa(m)} \leq \bar{c} m$ if $\mathfrak{m}(k) = 2^k$ for all $k$.}

The following lemma shows that the minimum eigenvalue of
$G^{\bs{\theta},\bs{\Psi},m} = (V^{\bs{\theta}, \bs{\Psi}, m} )^{-1}$, with  $V^{\bs{\theta}, \bs{\Psi}, m}$ defined in \eqref{eq: statistics},
 increases
with the number of exploration episodes.
\begin{Lemma}
\label{lemma:nondegenerate}
  $\psi^e \in \cup_{C\ge 0} \cV_C$ satisfies (H.\ref{Assump: Exploration strategy}\ref{item:independence})
if and only if 
for any bounded subset $\Theta\subset \sR^{d\t (d+p)}$,
there exists $\lambda_0>0$ such that for all $\theta\in \Theta$, 
$\Lambda_{\min}( \psi^e,\theta) \geq \lambda_0$,
with the function 
$\Lambda_{\min}:  \cup_{C \geq 0}\cV_C
\t \sR^{d\t (d+p)}
\to [0,\infty)$ 
defined in 
\eqref{eq: Information value}.
%

\end{Lemma}

\begin{proof}
Without loss of generality, let us fix a feedback control
  $\psi^e \in \cup_{C\ge 0} \cV_C$, 
a filtered probability space $(\Om,\cF, \sF, \sP)$ satisfying the usual condition,
and a  $d$-dimensional standard 
$\sF$-Brownian motion   $W$ on $(\Om, \cF, \sF, \sP)$.
The Lipschitz continuity of $\psi^e$ in the $x$-variable
and stability of SDEs
(see \cite[Theorem 3.2.4]{zhang2017backward})
 imply the continuity of 
$\sR^{d\t (d+p)}\ni \theta\mapsto 
\|X^{\theta, \psi^e}\|_{\cS^2}\in \sR$
and 
$\sR^{d\t (d+p)}\ni \theta\mapsto 
\|\psi^e(\cdot, X^{\theta, \psi^e}_\cdot)\|_{\cS^2}\in \sR$.
This shows that each entry of the matrix in
$\Lambda_{\min}( \psi^e,\theta)$ 
is continuous in $\theta$, as it involves only the
products of
$X^{\theta, \psi^e}$ and 
 $\psi^e(\cdot, X^{\theta, \psi^e}_\cdot)$.
Then, by 
(H.\ref{assum:bound}) and 
the continuity of the minimum eigenvalue function, 
it suffices to show that
the statement that 
(H.\ref{Assump: Exploration strategy}\ref{item:independence}) holds is equivalent to 
for any given $\theta\in \sR^{d\t (d+p)}$, there exists $\lambda_\theta>0$ such that 
$\Lambda_{\min}( \psi^e,\theta) \ge \lambda_{\theta}$.
The latter statement holds if and only if for any
$u\in \sR^d$ and $v\in \sR^p$,
$$
\begin{pmatrix}
u & v
\end{pmatrix}^\top
\mathbb{E}^{\sP}\Bigg[\int_{0}^{T}
\begin{pmatrix}
X^{{\theta}, {\psi}}_t \\ \psi\big(t, X^{{\theta }, {\psi}}_t\big) \end{pmatrix}
 \begin{pmatrix}
X^{{\theta }, {\psi}}_t \\ \psi\big(t, X^{{\theta }, {\psi}}_t\big) \end{pmatrix}^\top \d t \Bigg]
\begin{pmatrix}
u \\ v
\end{pmatrix}=0
$$
implies that  $u$ and $v$ are zero vectors, which can be  equivalently formulated as 
$u^\top X^{{\theta}, {\psi^e}}_t+v^\top \psi^e(t,X^{{\theta}, {\psi^e}}_t)=0$ 
for $\d\sP\otimes \d t$-a.e. implies 
 $u$ and $v$ are zero vectors.

Based on the above observations,  the if direction of the desired result holds clearly
by substituting $x=X^{{\theta}, {\psi^e}}_t$ in (H.\ref{Assump: Exploration strategy}\ref{item:independence}).
We now fix an arbitrary
$\theta\in \sR^{d\t (d+p)}$
and  prove the only if direction.
Let $u\in \sR^d$ and $v\in \sR^p$ such that
$u^\top X^{{\theta}, {\psi^e}}_t+v^\top \psi^e(t,X^{{\theta}, {\psi^e}}_t)=0$ 
for $\d\sP\otimes \d t$-a.e., and 
$E=\{(t,x)\in [0,T]\t \sR^d\mid u^\top x+v^\top \psi^e(t,x) \not =0 \}$, we aim  to show $E$ has Lebesgue  measure zero. 
Suppose that $E$ has non-zero  measure.  
As $\psi^e$ is continuous in the $x$-variable, $E_t=\{x\in \sR^d\mid (t,x)\in E\}$ is open for all $t\in [0,T]$. 
By \cite[Theorem 28.7]{kechris2012classical},
$E=\cup_{n\in \sN} (\Xi_n\t O_n)$, where 
$\cup_{n\in \sN}O_n$ is an open basis for the standard topology on $\sR^d$, and $\Xi_n\subset [0,T]$ is Borel for all $n\in \sN$. Since $E$ has positive  measure,
there exists
a Borel set $\Xi_n\subset [0,T]$ 
and an open set $O_n\subset \sR^d$  such that 
 $\Xi_n$ and $O_n$ have  positive measures
 and $\Xi_n\t O_n\subset E$. 
By Section  7.6.4 of \cite{liptser1977statistics}
and $\|X^{\theta,\psi^e}\|_{\cS^2}<\infty$, 
the law $\sP^{X^{\theta,\psi^e}}$ of 
$X^{\theta,\psi^e}$ is equivalent to the law of 
 $(x_0+W_t)_{t\in [0,T]}$.
 Hence, the fact that 
$O_n$ is an open set with positive measure
implies   $\sP(\{\om\in \Om \mid X^{\theta,\psi^e}_t(\om) \in O_n\})>0$ for all $t\in (0,T]$, 
which along with the fact that $\Xi_n$ has positive measure shows that $\{(\om, t)\in \Om\t [0,T]\mid 
(t, X^{\theta,\psi^e}_t(\om) )\in E\}$ has positive  $\d \sP\otimes \d t$ measure. 
This contradicts to the statement that
$u^\top X^{{\theta}, {\psi^e}}_t+v^\top \psi^e(t,X^{{\theta}, {\psi^e}}_t)=0$ 
for $\d\sP\otimes \d t$-a.e.
Consequently, 
$u^\top x+v^\top \psi^e(t,x)  =0$ for almost every $(t,x)\in [0,T]\t \sR^d$,
which along with the assumption implies that $u$ and $v$ are zero vectors
and  proves the only if direction.
\end{proof}

Based on Lemma  \ref{lemma:nondegenerate},
we quantify the accuracy of the estimators 
  $(\tilde{\bs{\theta}}_m)_{m\in \sN}$  in \eqref{eq:project}
generated by  the PEGE algorithm
 (without Step \ref{optional}).

\begin{Lemma}
\label{Lemma: hat theta bound high prob}
Suppose  (H.\ref{assum:lc}), (H.\ref{assum:bound}) and (H.\ref{Assump: Exploration strategy}) hold.
Let $\beta>0$, 
$\hat{\theta}_0 \in \sR^{d\t (d+p)}$, $V_0 \in \sS_+^{d+p}$,
{$\rho$ be a truncation function (cf.~Definition \ref{def: Truncate function})},
$r\in (0,1]$, 
and 
$\mathfrak{m} : \sN \to \sN$
be such that 
$\mathfrak{m}(k) = \lf k^r \rf$ for all $k\in\sN$.
There exists a constant $C\ge 0$ 
such that for all 
$\delta>0$ and $m \geq C\big(1+ \ln(\tfrac{1}{\delta}) \big)^{1+r}$, 
$$
\sP\big(|\tilde{\bs{\theta}}_m^{PEGE} -\bs \theta|^2  \leq 
\min\{
C m^{-1/(1+r)}( \ln m + \ln(\tfrac{1}{\delta}) ),
\beta^2\}
\big)\ge 1-\delta,
$$
where $\tilde{\bs{\theta}}_m^{PEGE} \coloneqq  \rho(\hat{\bs{\theta}}^{\bs{\Psi}^{PEGE}, m},  V^{\bs{\theta}, \bs{\Psi}^{PEGE}, m})$ and $(\hat{\bs{\theta}}^{\bs{\Psi}^{PEGE}, m},  V^{\bs{\theta}, \bs{\Psi}^{PEGE}, m})$ is  defined in   \eqref{eq: statistics}
\end{Lemma}

\begin{proof}
For notational simplicity, 
we    denote by $C$ a generic constant independent of $m, \delta$,
and omit the dependence 
on ``PEGE" and $\bs{\Psi}^{PEGE}$
in the superscripts of all random variables, if no confusion occurs. 

We first 
analyse the accuracy of 
the estimator $\hat{\bs{\theta}}^{\bs{\Psi}, m}$  defined  in \eqref{eq: statistics}. 
For each $k\in \sN$, the  state process $X^{\bs{\theta}, \bs{\Psi}, m^e_k}$ for  the $k$-th exploration episode
  satisfies
$$
\d X_t =\big(\bs{A} X_t+\bs{B}
\psi^e(t,X_t) \big)\d t+\, \d W^{m^e_k}_t
, \q t\in [0,T],
\q X_0=x_0.
$$
Let $Z^{k,e}_t \coloneqq \begin{psmallmatrix}
X^{\bs{\theta}, \bs{\Psi}, m^e_k}_t \\  \psi^e(t,X^{\bs{\theta}, \bs{\Psi}, m^e_k}_t)
\end{psmallmatrix}$
for all $t\in [0,T]$, 
and  let 
${G}^{k,e} \coloneqq V_0^{-1} + \sum_{n=1}^{k}  \int_0^T Z^{k,e}_t (Z^{k,e}_t)^\top \d t$.
By following the same argument  as that for
\eqref{eq:concentration_Z},
there exists $C\ge 0$ such that 
for all $k\in \sN$ and $\delta>0$,
 \begin{align}
 \label{eq: concentrate explore}
 \begin{split}
\sP(\|{G}^{k,e} -\sE [{G}^{k,e}| \sigma\{\bs{\theta}\} ]\|_{\textrm{op}}
\ge  k \eps(k,\delta) 
)
\le \delta,
\q
\textnormal{
with $
\eps(k,\delta)\coloneqq 
\max \Big( \sqrt{\tfrac{\ln(2/\delta)}{Ck} },
\tfrac{\ln(2/\delta)}{Ck}\Big)$,
}
\end{split}
\end{align}
where $\|\cdot\|_{\textrm{op}}$ denotes the operator norm of a matrix. 
By Lemma   \ref{lemma:nondegenerate} and 
\cite[Theorem 13, p~76]{krylov2002introduction},
there exists $\lambda_0>0$ such that 
$\lambda_{\min}(\sE [{G}^{k,e} | \sigma\{\bs{\theta}\}])\ge k\lambda_0$
$\sP$-almost surely. Moreover, observe from \eqref{eq: concentrate explore} that 
there exists $C\ge 0$ such that 
for all $\delta>0$ and  $k \geq C\big(1+ \ln(\tfrac{1}{\delta}) \big)$, $\eps(k,{\delta}) \leq \lambda_0/2$. 
Hence there exists $C\ge 0$ 
such that for all 
$\delta>0$ and $k \geq C\big(1+ \ln(\tfrac{1}{\delta}) \big)$, with probability at least $1-\delta$, 
\begin{align}
\label{eq: min info bound}
\lambda_{\min}\big({G}^{k,e} \big) 
 &= \inf_{|u| = 1} \big(u^\top {G}^{k,e} u \big) = \inf_{ |u| = 1} \Big(u^\top \big(\sE [{G}^{k,e} | \sigma\{\bs{\theta}\}] \big) u  + u^\top \big(
{G}^{k,e} -\sE [{G}^{k,e} | \sigma\{\bs{\theta}\}]
\big) u\Big) \nonumber
\\
&\geq
\inf_{ |u| = 1} \big(u^\top \big(\sE [{G}^{k,e} | \sigma\{\bs{\theta}\}] \big) u \big) - 
\|{G}^{k,e} -\sE [{G}^{k,e} | \sigma\{\bs{\theta}\}] \|_{\textrm{op}}
    \geq  k \lambda_0/2.
\end{align} 
Using the fact that 
$\underline{c} m \leq \kappa(m)^{1+r} \leq \bar{c} m$
for all $m\in \sN$, 
there exists $C \ge 0$ 
such that for all
$\delta>0$ and $m \geq C\big(1+ \ln(\tfrac{1}{\delta}) \big)^{1+r}$, with probability at least $1-\delta$, 
$$
\lambda_{\min}\big(G^{\bs \theta, \bs \Psi, m}\big) \geq \lambda_{\min}\big(G^{\kappa(m),e}\big) \geq \tfrac{\kappa(m)\lambda_0}{2} \geq \tfrac{\lambda_0(\underline{c} m)^{{1}/({r+1})}}{2}.
$$

Observe that  $\bs{\Psi}^{PEGE}$ only executes the feedback control of form $\psi^e$ and $\psi_{\tilde{\bs{\theta}}_m}$ for some $m\in \sN$. As  $\tilde{\bs{\theta}}_m = \rho(\hat{\bs{\theta}}^{m},  V^{\bs{\theta},  m})$ takes values in a bounded set $\cK\coloneqq \operatorname{range}(\rho)$ (cf.~Definition \ref{def: Truncate function}), 
by
Proposition \ref{prop:existence_fb}
and the fact that $\psi^e\in \cup_{C\ge 0} \cV_C$, we have 
$\bs{\Psi}^{PEGE} \in  \cup_{C\ge 0}\cU_C$.
Hence, by Theorem \ref{Theorem: Information collapse} 
and the above estimate for $\lambda_{\min}(G^{\bs \theta, \bs \Psi, m})$, 
there exists $C \ge 0$ 
such that for all 
$\delta>0$ and $m \geq C\big(1+ \ln(\tfrac{1}{\delta}) \big)^{1+r}$, 
\bb\label{eq:hat_theta_m_accuracy}
\sP\Big(|\hat{\bs{\theta}}^{\bs{\Psi} , m} - \bs \theta|^2  \leq C m^{-1/(1+r)}( \ln m + \ln(\tfrac{1}{\delta}) )\Big)\ge 1-\delta.
\ee

We then prove with high probability, $\tilde{\bs{\theta}}_m  = \hat{\bs{\theta}}^{\bs{\Psi} , m} $ 
and $|\hat{\bs{\theta}}^{\bs{\Psi} , m} -\bs \theta|\le \beta$ 
for all $m$ sufficiently large,
where  $\beta>0$  is the constant given in the statement.
  By Definition \ref{def: Truncate function},  $
\operatorname{cl}(\Theta)\subseteq
\operatorname{int}(\cK)$,
with $\cK=\operatorname{range}(\rho)$.
Hence,
there exists $\eta>0$ such that 
$\sB_\eta (\Theta)\subseteq  \cK$,
with 
$\sB_\eta (\Theta)\coloneqq \{\tilde{\theta}\in \sR^{d\t (d+p)}\mid  
\textnormal{$\exists \theta\in \Theta$ s.t.~$|\tilde{\theta}-\theta|\le \eta$
}\}$.
Observe that there exists $C_0\ge 0$ such that 
for all 
$\delta>0$, 
$m \geq C_0\big(1 + \ln(\tfrac{1}{\delta}) \big)^{r+1} $, 
$| C m^{-1/(1+r)} \big( \ln m +  \ln(\tfrac{1}{\delta})  \big) | \leq \min(\eta^2,\beta^2)$,
where $C$  is the same   constant  as that in \eqref{eq:hat_theta_m_accuracy},
and  $\beta$  is the constant given in the statement.
Combining this with 
\eqref{eq:hat_theta_m_accuracy} yields that 
there exists $C \ge 0$ 
such that for all 
$\delta>0$ and $m \geq C\big(1+ \ln(\tfrac{1}{\delta}) \big)^{1+r}$, 
\bb\label{eq:hat_theta_m_accuracy_beta}
\sP\Big(|\hat{\bs{\theta}}^{\bs{\Psi}, m} - \bs \theta|^2  \leq 
\min\big\{
C m^{-1/(1+r)}( \ln m + \ln(\tfrac{1}{\delta}) ),
\eta^2,\beta^2\big\}
\Big)\ge 1-\delta,
\ee
which
along with  $ \sB_\eta (\Theta)\subseteq  \cK$
and 
$\tilde{\bs{\theta}}_m = \rho(\hat{\bs{\theta}}^{m},  V^{\bs{\theta},  m})=\hat{\bs{\theta}}^{m}$
provided that $\hat{\bs{\theta}}^{m}\in \cK$
(see  Definition \ref{def: Truncate function}) implies that 
$\tilde{\bs{\theta}}_m  = \hat{\bs{\theta}}^{\bs{\Psi},m}\in \cK$
on the above event, and  finishes the proof of the desired estimate.
\end{proof}

The following theorem proves a sublinear bound for the expected regret of  the PEGE algorithm
 (without Step \ref{optional}).

\begin{Theorem}
\label{Theorem: expected regret}
Suppose  (H.\ref{assum:lc}), (H.\ref{assum:bound}), (H.\ref{Assump: Exploration strategy}) and (H.\ref{Assump: Performance gap}) hold.
Let 
$\hat{\theta}_0 \in \sR^{d\t (d+p)}$, $V_0 \in \sS_+^{d+p}$,
$\rho$ be a truncation function (cf.~Definition \ref{def: Truncate function}),
and 
$\mathfrak{m} : \sN \to \sN$
be such that 
$\mathfrak{m}(k) = \lf k^r \rf$ for all $k\in\sN$
 with 
  $r\in (0,1]$ being the same as in (H.\ref{Assump: Performance gap}).
 Then there exists a constant $C \geq 0$  such that  for all $\delta \in (0,1)$, 
  with probability at least $1-\delta$,
$$
\sum_{m=1}^N \Big( J(\Psi^{PEGE}_m; \bs{\theta}) - 
 J(\psi_{\bs{\theta}}; \bs{\theta})
\Big)
 \leq C 
 \Big(
 N^{\frac{1}{1+r}} \big( (\ln N)^r + \big( \ln(\tfrac{1}{\delta}) \big)^{r} \big)
 +\big( \ln(\tfrac{1}{\delta}) \big)^{1+r}
 \Big), \q
\fa N\in \sN\cap [2,\infty).
$$
where for each $\om \in \Om$, $J(\Psi^{PEGE}_m; \bs{\theta})(\omega) \coloneqq J(\Psi^{PEGE}_m(\omega, \cdot, \cdot) ; \bs{\theta}(\omega))$ and 
$J(\psi_{\bs{\theta}}; \bs{\theta})(\omega)
\coloneqq 
J(\psi_{\bs{\theta}(\omega)}; \bs{\theta}(\omega))$ with $J$  and $\psi_\theta$ given in Proposition \ref{prop:existence_fb}.
\end{Theorem}

\begin{proof}
Throughout this proof, 
we denote by  $C$  a generic constant independent of $m, \delta$ and $\om$,
and omit the dependence 
on ``PEGE" 
in the superscripts (e.g., $\bs{\Psi} = \bs{\Psi}^{PEGE}$).

By Lemma \ref{Lemma: hat theta bound high prob} and 
 (H.\ref{Assump: Performance gap}),
 there exists a constant $C_0\ge 0$ 
such that for all 
$\delta>0$ and $m \geq C_0\big(1+ \ln(\tfrac{1}{\delta}) \big)^{1+r}$, $m \in\sN\setminus \cE$,
\bb\l{eq: One Step regret}
\sP\Big(
 \big|
J(\psi_{\tilde{\bs{\theta}}_m };\bs\theta)
 - J(\psi_{\bs{\theta}}; \bs{\theta})
 \big|
   \leq 
C_0 m^{-\frac{r}{1+r}}( \ln m + \ln(\tfrac{1}{\delta}) )^r
\Big)\ge 1-\delta.
\ee
Observe that there exists $C\ge 0$ such that for all $\delta>0$, 
$m\ge  C\big(1+ \ln(\tfrac{1}{\delta}) \big)^{1+r}$
implies that 
$m\ge C_0\big(1+ \ln({1}/{\delta}_m) \big)^{1+r}$
with $\delta_m\coloneqq 6\delta/(\pi^2 m^2)$.
Hence, applying \eqref{eq: One Step regret} with 
$\delta_m$
for all sufficiently large $m$, 
  summing 
 the corresponding probabilities 
  over all $m$
  and using $m\ge C \kappa(m)^{1+r}$
  yield  that,
for all 
$\delta>0$, with probability at least
$1-\delta$,
we have for all  
$m\in \sN$ with 
{$\kappa(m) \ge  C\big(1+ \ln(\tfrac{1}{\delta}) \big)$},

\bb\l{eq: One Step regret_uniform}
 \big|
J(\psi_{\tilde{\bs{\theta}}_m };\bs\theta)
 - J(\psi_{\bs{\theta}}; \bs{\theta})
 \big|
   \leq 
C m^{-\frac{r}{1+r}}( \ln m + \ln(\tfrac{1}{\delta}) )^r.
\ee
On the event considered above,  for all $N\in \sN$, we consider the following decomposition
\begin{align}
&\sum_{m=1}^N \Big( J(\Psi_m; \bs{\theta}) - 
 J(\psi_{\bs{\theta}}; \bs{\theta})
\Big) \nonumber
\\
&\q
 = \sum_{m \in \sN\cap [1,N]\cap \cA_{\delta,e}^c }\Big( J(\Psi _m; \bs{\theta}) - 
 J(\psi_{\bs{\theta}}; \bs{\theta})
\Big)
 +  \sum_{m \in \sN\cap [1,N]\cap \cA_{\delta,e}}\Big( J(\Psi _m; \bs{\theta}) - 
 J(\psi_{\bs{\theta}}; \bs{\theta})
\Big),
 \label{eq: regret bound 2}
\end{align}
where the set 
$\cA_{\delta,e}$ is defined by 
{$\cA_{\delta,e}\coloneqq 
\{ m\in \sN\mid \kappa(m) \ge  C\big(1+ \ln(\tfrac{1}{\delta}) \big) , m \not \in \cE\}$}
(with the same constant $C$  as in \eqref{eq: One Step regret_uniform}).

We first  estimate the first term in \eqref{eq: regret bound 2}.
It is  shown in Lemma \ref{Lemma: hat theta bound high prob} that
$\bs{\Psi} \in \cU_C$ for some constant $C \geq 0$.
Standard moment estimates of the state dynamics yield that $\sE[\sup_t |X^{\bs{\theta},\bs{\Psi}, m}_t|^2]\le C$.
By (H.\ref{assum:lc}),
 Proposition \ref{prop:existence_fb}
Item \ref{item:f_psi} and 
(H.\ref{Assump: Exploration strategy}\ref{item:exploration_quadratic}),
$|f(t,x,\Psi_m(\om,t,x))|+|g(x)|\le C(1+|x|^2)$
for all $m\in \sN$, $(\om,t,x)\in \Om \t [0,T]\t \sR^d$. This estimate along with (H.\ref{assum:bound})
implies that 
$|J(\Psi_m (\omega, \cdot, \cdot), \bs{\theta}(\omega)) |\le C$
for all $m \in \sN$  and $\omega \in \Omega$. 
Moreover, by using the continuity of $V^\star:\Theta\to \sR$ and 
(H.\ref{assum:bound}), 
$|V^\star(\theta)|=|J(\psi_{{\theta}}; {\theta})|\le C$ for all $\theta\in \Theta$.
Hence for all $N\ge 2$,
\begin{align*}
\sum_{m \in \sN\cap [1,N]\cap \cA_{\delta,e}^c }\Big( J(\Psi_m; \bs{\theta}) - 
J(\psi_{\bs{\theta}}; \bs{\theta})
\Big)
 &\le C\Big(
 |\{m\in\sN\mid \kappa(m) \ge  C\big(1+ \ln(\tfrac{1}{\delta}) \big)\}|
 +\kappa(N)\Big)
 \\
 &\le C \big(  \big( \ln(\tfrac{1}{\delta}) \big)^{1+r}
 +N^{\frac{1}{1+r}}   \big).
\end{align*}

We then estimate the second term in \eqref{eq: regret bound 2} for the PEGE Algorithm 
(without Step \ref{optional}).
For all $m \in\sN\setminus \cE$, 
$J(\Psi_m; \bs{\theta})
=J(\psi_{\ol{\bs{\theta}}};\theta)$ 
with 
$\ol{\bs{\theta}} = \tilde{\bs{\theta}}_{m^e_{\kappa(m)}}$, where $m^e_{\kappa(m)}$ is the latest  exploration episode  before the $m$-th episode.
By \eqref{eq: One Step regret_uniform},
%
%
%
%
\begin{align}
& \sum_{m \in \sN\cap [1,N]\cap \cA_{\delta,e}}\Big( J(\Psi_m; \bs{\theta}) - 
 J(\psi_{\bs{\theta}}; \bs{\theta})
\Big)
 \le 
 C\sum_{m \in \sN\cap [1,N]\cap \cA_{\delta,e}}\Big( 
  (m^e_{\kappa(m)})^{-\frac{r}{1+r}}( \ln (m^e_{\kappa(m)}) + \ln(\tfrac{1}{\delta}) )^r
  \Big)  
 \nb \\
 &\qq \qq 
 \leq  C\sum_{k =1}^{\kappa(N)} \lf k^r \rf \Big(  
 (m_k^e)^{-\frac{r}{1+r}}( \ln (m_k^e) + \ln(\tfrac{1}{\delta}) )^r \Big)
 \nb \\
 &\qq \qquad \leq C\sum_{k =1}^{\kappa(N)}  \lf k^r \rf \Big(  
k^{-r}(1 + \ln (k) + \ln(\tfrac{1}{\delta}) )^r \Big) \leq C\sum_{k =1}^{\kappa(N)}  \Big(  
1 + \ln (k) + \ln(\tfrac{1}{\delta}) \Big)^r 
 \nb  \\
 &\qq \qquad \leq {C} \kappa(N) \big( (\ln \kappa(N))^r + \big( \ln(\tfrac{1}{\delta}) \big)^{r} \big)  \leq {C} N^{\frac{1}{1+r}} \big( (\ln N)^r + \big( \ln(\tfrac{1}{\delta}) \big)^{r} \big)
 \label{eq:regret_exploitatoin}
\end{align}
where the interchange between episodes  and cycles follows from the fact that $\underline{c} m \leq \kappa(m)^{1+r} \leq \bar{c} m$ for all $m$.
Combining these estimates with \eqref{eq: regret bound 2} leads to  the desired result. 
\end{proof}

\begin{proof}[Proof of Theorem \ref{Theorem: regret}]

Throughout this proof, 
we denote by  $C$  a generic constant independent of $m, \delta$ and omit the dependence 
on ``PEGE" 
in the superscripts (e.g., $\bs{\Psi} = \bs{\Psi}^{PEGE}$).

It has been shown in the proof of Theorem \ref{Theorem: expected regret} that 
$\bs{\Psi} \in \cU_C$ and
$|g(x)|+|f(t,x,\Psi_m(\om,t,x))|\le C(1+|x|^2)$
for all $m\in \sN$,  $(\om,t,x)\in \Om \t [0,T]\t \sR^d$.
Then by \eqref{eq:cost_m},  Propositions 
 \ref{prop:subguassian_subexponential}
and \ref{prop:X_subgaussian},
$\|\ell_m(\bs{\Psi},\bs{\theta})\|_{1,\cF_{m-1}}\le C$ for all $m\in \sN$,
where $\cF_{m-1}$ is defined as in 
Proposition \ref{prop:X_subgaussian}.
Observe that 
$\sE [\ell_m(\bs{\Psi},\bs{\theta})\mid \cF_{m-1}]=J({\Psi}_m; \bs{\theta})$ for all $m\in \sN$.
Hence, by Proposition \ref{prop: concentration inequality}, there exists $C\ge 0$ such that for all 
$N\in \sN$ and $\delta>0$,
$$\sP\bigg( \bigg|\sum_{n=1}^N \Big(\ell_n(\bs{ \Psi}, \bs{\theta})- J({\Psi}_n; \bs{\theta})\Big) \bigg| \geq N\eps(N,\delta)\bigg) \leq \delta,$$
with 
$\eps(N,\delta)\coloneqq \max \Big( \sqrt{\tfrac{\ln(2/\delta)}{CN} },
\tfrac{\ln(2/\delta)}{CN}\Big)
\le CN^{-1/2}(1+{\ln(\tfrac{1}{\delta})})
$.
Then for each $N\in \sN$ and $\delta>0$, applying the above inequality with $\delta_N=6 \delta/(\pi^2 N^2)$ and summing the corresponding probabilities for all $N\ge 2$,
there exists $C\ge 0$ such that for all 
$\delta>0$,
$$\sP\bigg( \bigg|\sum_{n=1}^N \Big(\ell_n(\bs{ \Psi}, {\theta})-J({\Psi}_n; \bs{\theta}) \Big) \bigg| \le
{C}N^{1/2}\big(\ln(N) + \ln(\tfrac{1}{\delta})\big),
\; \fa N\in \sN\cap[2,\infty)
\bigg) \ge 1- \delta.$$
Combining the above estimate with 
Theorem \ref{Theorem: expected regret}  and the fact that $r \in (0,1]$
yields that 
there exists a constant $C \geq 0$  such that  for all $\delta > 0$, 
  with probability at least $1-\delta$,
\bb\label{eq:freq_regret_PW}
{\cR}(N,  \bs{\Psi} ,  \bs{\theta}) 
 \leq C 
  \Big(
 N^{\frac{1}{1+r}} \big( (\ln N)^r + \big( \ln(\tfrac{1}{\delta}) \big)^{r} \big)
 +\big( \ln(\tfrac{1}{\delta}) \big)^{1+r}
 \Big), \q
\fa N\in \sN\cap [2,\infty).
\ee

It now remains to prove the expected regret bound in Theorem \ref{Theorem: regret}.
For each $\lambda>0$, applying 
 \eqref{eq:freq_regret_PW}
 with $\delta = \exp(-\lambda^{\frac{1}{r+1}})$
 gives that 
there exists a constant $C > 0$  such that  for all $\lambda>0$, 
$$
\sP\Big(
CN^{-\frac{1}{1+r}}{\cR}(N,  \bs{\Psi},  \bs{\theta}) -(\ln N)^r \ge \lambda, \q
\fa N\in \sN\cap [2,\infty)
\Big)\le  \exp(-\lambda^{\frac{1}{r+1}}).
$$
For each $N\in \sN$, let 
$D_N\coloneqq 
CN^{-\frac{1}{1+r}}{\cR}(N,  \bs{\Psi},  \bs{\theta}) -(\ln N)^r $.
Then for all $ N\in \sN\cap [2,\infty)$ 
and $\lambda>0$, 
$\sP(\max(D_N,0)>\lambda)=\sP(D_N\ge \lambda)
\le \exp(-\lambda^{\frac{1}{r+1}})$.
Hence for all $N\in \sN\cap [2,\infty)$,
$$
\mathbb{E}[ D_N] \leq \mathbb{E} \Big[ \max(D_N, 0) \Big]  = \int_0^\infty \sP\big( \max(D_N, 0) > \lambda \big) \, \d \lambda \leq \int_0^\infty \exp(-\lambda^{\frac{1}{r+1}}) \,\d \lambda < \infty,
$$
where the last inequality follows from 
$r\in (0,1]$. This finishes the proof of Theorem \ref{Theorem: regret}.
\end{proof}

\begin{proof}[Proof of Theorem \ref{Theorem: regret_self explore}]

The proof  follows from similar arguments  as that of 
Theorem \ref{Theorem: regret},
 and we only present the main steps here.
For notational simplicity, we  omit the dependence on ``PEGE'' (e.g., $\bs{\Psi} = \bs{\Psi}^{PEGE}$), 
and 
denote by $C \geq 0$ a generic constant independent of $m$, $\delta$ and $N$.

 By \eqref{eq: statistics} and \eqref{eq:concentration_Z},
there exists a constant $C\ge 0$ such that 
 for all $m\in \sN$ and $\delta\in (0,1)$,
 with probability  at least  $1-\delta$, 
 \begin{align*}
 \begin{split}
&\bigg|
G^{\bs{\theta}, \bs{\Psi}, m}
- V_0^{-1} -
\sum_{n=1}^m  \sE \Big[\int_0^T Z^{\bs \theta, \bs{\Psi}, n}_t
(Z^{\bs \theta, \bs{\Psi}, n}_t)^\top\, \d t
\Big| \cF_{n-1}\Big]
 \bigg| 
 \le  m \max \bigg( \sqrt{\frac{\ln(\tfrac{1}{\delta})}{Cm} },
\frac{\ln(\tfrac{1}{\delta})}{Cm} \bigg).
\end{split}
\end{align*}
Observe that for all $m\in \cE$, $\Psi_m=\psi^e$
and for all  $m\in \sN\setminus \cE$, 
$\Psi_m=\psi_{\bar{\bs{\theta}}_m}$,
where  
$
\ol{\bs{\theta}}_m \coloneqq \tilde{\bs{\theta}}_{m^e_{\kappa(m)}}$ with $\tilde{\bs{\theta}}_m$ given in \eqref{eq:project}
and $m^e_{\kappa(m)}$ is the latest  exploration episode  before the $m$th episode;
without loss of generality,
we also define 
 $\ol{\bs{\theta}}_m  =\rho({\hat{\theta}}_{0},V_0)$
 for $m\in \cE$.
Then by  the additional 
assumption on $\Lambda_{\min}(\theta, \psi_{\theta'})$,
  Lemma \ref{lemma:nondegenerate},
  and the fact that $\Psi_m$  depends on $(W^n)_{n=1}^{m-1}$ and $\bs{\theta}$ only through $\bar{\bs{\theta}}_m$, 
  there exists $\lambda_0>0$ such that 
for all $m\in \sN$ and $u\in \sR^{d+p}$, 
\begin{align*}
&u^\top \sE \bigg[\int_0^T Z^{\bs{\theta}, \bs{\Psi}, m}_t
(Z^{\bs{\theta}, \bs{\Psi}, m}_t)^\top\, \d t
\Big| \cF_{m-1}\bigg] u
= 
u^\top
\sE \bigg[\int_0^T Z^{\bs{\theta}, \bs{\Psi}, m}_t
(Z^{\bs{\theta}, \bs{\Psi}, m}_t)^\top\, \d t
\Big| 
\sigma\{\bs \theta, \bar{\bs{\theta}}_{m}\} \bigg] 
u
 \\
 &\qquad \ge 
 \sE \bigg[
\min\Big( \Lambda_{\min}(\psi_{\bar{\bs{\theta}}_{m}},\bs \theta), \Lambda_{\min}( \psi^e,\bs \theta) \Big)
|u|^2
\Big| \sigma\{\bs{\theta}, \bar{\bs{\theta}}_{m}\}
 \bigg] 
 \geq \lambda_0|u|^2.
\end{align*}
Hence  a similar argument as that for  \eqref{eq: min info bound} shows that 
 there exists $C\ge 0$ 
such that for all 
$\delta>0$ and $m \geq C\big(1+ \ln(\tfrac{1}{\delta}) \big)$, 
$\sP\Big(\lambda_{\min}(G^{\bs{\theta}, \bs{\Psi}, m})\geq m \lambda_0/2\Big)  \ge 1-\delta $.
Proceeding along the lines of 
the proof 
of Lemma \ref{Lemma: hat theta bound high prob} yields that 
 there exists $C\ge 0$ 
such that for all 
$\delta>0$ and $m \geq C\big(1+ \ln(\tfrac{1}{\delta}) \big)$, 
\begin{equation}
\label{eq: error of Greedy self-explore}
\sP\Big( \big| \tilde{\bs{\theta}}_m -\bs{\theta} \big|^2 \leq 
\min\{
{C}m^{-1} \big( \ln m + \ln(\tfrac{1}{\delta}) \big),
\beta^2\}
\Big) \geq 1-\delta,
\end{equation}
with   the constant $\beta$ in (H.\ref{Assump: Performance gap}). 
By considering $\delta_m = 6\delta/(\pi^2m^2)$ as in Theorem \ref{Theorem: expected regret} and using (H.\ref{Assump: Performance gap}), 
 there exists $C\ge 0$ 
such that for all 
$\delta>0$, 
with probability at $1-\delta$,
\begin{equation}
\label{eq: error of Greedy self-explore2}
  \big|
J(\psi_{\tilde{\bs{\theta}}_m};\bs{\theta})
 -  J(\psi_{\bs{\theta}}; \bs{\theta})
 \big| \leq {C}m^{-r} \big( \ln m + \ln(\tfrac{1}{\delta}) \big)^r, 
 \q 
 \textnormal{
 $\fa m\in \sN$ s.t.~$\kappa (m) \geq C + \log_2\big( 1+\ln(\tfrac{1}{\delta})\big)$,}
\end{equation}
where the condition regarding $\kappa$ follows from the inequality  $ 2^{\kappa(m)} \leq \bar{c} m$ for all $m\in \sN$. 

On the event considered above,  for all $N\in \sN$, we consider the decomposition as in \eqref{eq: regret bound 2};
\begin{align} \label{eq: regret bound self_explore2}
\begin{split}
&\sum_{m=1}^N \Big( J(\Psi_m; \bs{\theta}) - 
 J(\psi_{\bs{\theta}}; \bs{\theta})
\Big)
\\
&\q
 = \sum_{m \in \sN\cap [1,N]\cap \cA_{\delta,e}^c }\Big( J(\Psi_m; \bs{\theta}) - 
 J(\psi_{\bs{\theta}}; \bs{\theta})
\Big)
 +  \sum_{m \in [1,N]\cap \cA_{\delta,e}}\Big( J(\Psi_m; \bs{\theta}) - 
  J(\psi_{\bs{\theta}}; \bs{\theta})
\Big),
\end{split}
\end{align}
with the set 
{$\cA_{\delta,e}\coloneqq 
\{ m\in \sN\mid \kappa(m) \geq C + \log_2(1+\ln(\tfrac{1}{\delta}))
, m \not \in \cE\}$},
where $C$  is the same constant as  in \eqref{eq: error of Greedy self-explore2},
and  $\cE$ is the collection of explorations episodes.
The  first term can be estimated as in the proof of Theorem \ref{Theorem: expected regret}:
for all $N\in \sN\cap [2,\infty)$, 
\begin{align*}
\sum_{m \in \sN\cap [1,N]\cap \cA_{\delta,e}^c }\Big( J(\Psi_m; \bs{\theta}) - 
  J(\psi_{\bs{\theta}}; \bs{\theta})
\Big)
 &\le C\big(   \ln(\tfrac{1}{\delta})  +\ln N \big),
\end{align*}
where the inequality follows from $\underline{c}m \leq 2^{\kappa(m)} \leq \bar{c}m$ for all $m\in\sN$.
To estimate the second term  in \eqref{eq: regret bound self_explore2},
recall that for the $k$th cycle,  the exploration policy $\psi_{\tilde{\bs{\theta}}_{m^e_k}}$ 
is executed 
for  $2^k$ episodes. 
Hence, by \eqref{eq: error of Greedy self-explore2},
\begin{align*}
& \sum_{m \in [1,N]\cap \cA_{\delta,e}}\Big(  J(\Psi_m; \bs{\theta})
 -  J(\psi_{\bs{\theta}}; \bs{\theta}) \Big)  \leq  C\sum_{k =1}^{\kappa(N)} 2^k \Big(  
 (m_k^e)^{-r}\big( \ln (m_k^e) + \ln(\tfrac{1}{\delta}) \big)^r \Big) \\
 &\qq \qquad \leq C\sum_{k =1}^{\kappa(N)}2^k\Big(  
2^{-rk}(k + \ln(\tfrac{1}{\delta}) )^r \Big) \leq C\sum_{k =1}^{\kappa(N)} 2^{(1-r)k} \Big(  
\kappa(N) + \ln(\tfrac{1}{\delta}) \Big)^r 
 \\
  &\qq \qquad \leq \begin{cases}
  CN^{1-r}\big(  (\ln N)^r  + \big( \ln(\tfrac{1}{\delta}) \big)^{r} \big),
   & r \in (0,1),\\
  C (\ln N) \big(  \ln N  +   \ln(\tfrac{1}{\delta})   \big),
   & r =1,
 \end{cases} 
\end{align*}
where the last inequality follows from the facts that 
$\kappa(N)\le C\log_2(N)$ and 
if $r\in (0,1)$, then $\sum_{k =1}^{m} 2^{(1-r)k}\le C2^{(1-r)m} $  for all $m\in\sN$.

Finally, recall the random variables 
$(\bar{\bs{\theta}}_m)_{m\in\sN}$   defined
at the beginning of the proof.
Since 
$$
\sE[{\cR}(N,  \bs{\Psi},  \bs{\theta})]=
\sum_{m=1}^N\sE \Big[ \sE\big[ \ell_m(\bs{\Psi}, \bs{\theta})- 
 J(\psi_{\bs{\theta}}; \bs{\theta})
 \big|\sigma\{\bar{\bs{\theta}}_m, \bs{\theta}\}\big]\Big]
=
\sum_{m=1}^N\sE \big[ J(\Psi_m; \bs{\theta}) - 
 J(\psi_{\bs{\theta}}; \bs{\theta})
\big],$$
the desired regret bound in expectation follows from the above high probability regret bound and similar arguments as those in the proof of  Theorem \ref{Theorem: regret}
.
\end{proof}

\appendix

\section{Proof of Proposition \ref{prop: PEGE as a learning}}
\label{sec:learning_alg}

 \begin{proof}
 It is clear that for all $m\in \sN$ and $\om\in \Om$, 
 $\Psi_m(\om,\cdot)\in \cV_{C_m}$. 
 We now prove by induction that
 for all $m\in \sN\cup\{0\}$, 
  $(\hat{\bs{\theta}}^{\bs{\Psi}, n}, V^{\bs{\theta }, \bs{\Psi}, n})$, $0\le n\le m$,  are ${\cG}^{\bs{\Psi}}_{m}$-measurable and $\Psi_{m+1}$ is $(\cG^{\bs{\Psi}}_{m}\otimes\cB([0,T])\otimes \cB(\sR^d)  )/\cB(\sR^p)$-measurable,
  with the $\sigma$-algebra $\cG^{\bs{\Psi}}_{m}$ defined in Definition \ref{Def: Admissible RL}.
Note that the statement clearly holds for $m=0$, as 
$(\hat{\bs{\theta}}^{  \bs{\Psi}, 0}, V^{\bs{\theta }, \bs{\Psi}, 0})=(\hat{\theta}_0,V_0)$ are deterministic. 

Suppose that 
the induction hypothesis holds for some index $m-1\in \sN\cup\{0\}$, 
namely,
$(\hat{\bs{\theta}}^{  \bs{\Psi}, n}, V^{\bs{\theta }, \bs{\Psi}, n})$,
$0\le n\le m-1$,
are ${\cG}^{\bs{\Psi}}_{m-1}$-measurable and $\Psi_m$ is $(\cG^{\bs{\Psi}}_{m-1}\otimes\cB([0,T])\otimes \cB(\sR^d)  )/\cB(\sR^p)$-measurable. 
Then the process $Z^{\bs{\theta}, \bs{\Psi}, m}_t  \coloneqq 
\begin{psmallmatrix}
X^{\bs{\theta}, \bs{\Psi}, m}_t \\ \Psi_m\big(\cdot, t, X^{\bs{\theta}, \bs{\Psi}, m}_t\big) \end{psmallmatrix}$,
$t\in [0,T]$,  is progressively measurable with respect to the filtration $\sG^{X,\bs{\Psi},m}=(\cG^{X, \bs{\Psi}, m}_t )_{t\in [0,T]}$ defined by $\cG^{X, \bs{\Psi}, m}_t \coloneqq \cG^{\bs{\Psi}}_{m-1} \vee \sigma\big(X^{\bs{\theta}, \bs{\Psi}, m}_u \mid u \leq t \big)$,
which
along with
the measurability
of 
$V^{\bs{\theta }, \bs{\Psi}, m-1}$
and \eqref{eq: statistics}
implies that 
 $V^{\bs{\theta }, \bs{\Psi}, m}$
 and 
 $(V^{\bs{\theta }, \bs{\Psi}, m})^{-1}$
 are  ${\cG}^{\bs{\Psi}}_{m}$-measurable. 
 
 To prove the ${\cG}^{\bs{\Psi}}_{m}$-measurability of $\hat{\bs{\theta}}^{  \bs{\Psi}, m}$,
 we first establish  a semi-martingale representation of   the process $X^{\bs{\theta}, \bs{\Psi}, m}$ with respect to the filtration  $\sG^{X,\bs{\Psi},m}$.
 Let 
 $B^{\bs{\theta}, \bs{\Psi}, m}
 =(B^{\bs{\theta}, \bs{\Psi}, m}_t)_{t\in [0,T]} $ be  the process 
 such that 
\bb\label{eq:B}
 B^{\bs{\theta}, \bs{\Psi}, m}_t \coloneqq  X^{\bs{\theta }, \bs{\Psi}, m}_t -
x_0 -
\int_0^t 
 \sE\big[\bs{\theta} \big| \cG^{X, \bs{\Psi}, m}_s\big]\; Z^{\bs{\theta }, \bs{\Psi}, m}_s \d s,
\q \fa t\in [0,T].
 \ee
  Note that $B^{\bs{\theta}, \bs{\Psi}, m}$ is 
  $\sG^{X,\bs{\Psi},m}$-progressively measurable. 
  We now verify that $B^{\bs{\theta }, \bs{\Psi}, m}$ is in fact a   $\sG^{X,\bs{\Psi},m}$-Brownian motion.  
Since $Z^{\bs{\theta }, \bs{\Psi}, m}_t$ is $\cG^{\bs{\theta }, \bs{\Psi}, m}_t$-measurable,  it follows from \eqref{eq:X_m}, the tower property of conditional expectation and Fubini's theorem that for all $t \geq s$,
\begin{align*}
& \mathbb{E}[B^{\bs{\theta }, \bs{\Psi}, m}_t - B^{\bs{\theta }, \bs{\Psi}, m}_s| \cG^{X, \bs{\Psi}, m}_s] =
 \mathbb{E}\bigg[\int_s^t \d X^{\bs{\theta }, \bs{\Psi}, m}_u - \int_s^t \sE\big[\bs{\theta} \big| \cG^{X, \bs{\Psi}, m}_u \big]\; Z^{\bs{\theta }, \bs{\Psi}, m}_u \d u \Big| \cG^{X, \bs{\Psi}, m}_s\bigg] \\
&\qq =\mathbb{E}\bigg[\int_s^t \bs{\theta} Z^{\bs{\theta }, \bs{\Psi}, m}_u \d u - \int_s^t \sE\big[\bs{\theta} \big| \cG^{X, \bs{\Psi}, m}_u \big]\; Z^{\bs{\theta }, \bs{\Psi}, m}_u \d u + \int_s^t \d W^m_u \Big| \cG^{X, \bs{\Psi}, m}_s\bigg]\\
&\qq =\int_s^t \mathbb{E}\Big[ \bs{\theta} Z^{\bs{\theta }, \bs{\Psi}, m}_u  -  \sE\big[\bs{\theta} \big| \cG^{X, \bs{\Psi}, m}_u \big]\; Z^{\bs{\theta }, \bs{\Psi}, m}_u   \Big| \cG^{X, \bs{\Psi}, m}_s\Big] \d u \\
&\qq =\int_s^t \mathbb{E}\bigg[  \mathbb{E}\Big[ \bs{\theta} Z^{\bs{\theta }, \bs{\Psi}, m}_u  -  \sE\big[\bs{\theta} Z^{\bs{\theta }, \bs{\Psi}, m}_u \big| \cG^{X, \bs{\Psi}, m}_u \big]\Big|  \cG^{X, \bs{\Psi}, m}_u \Big] \bigg| \cG^{X, \bs{\Psi}, m}_s\bigg] \d u =0.
\end{align*}
 This shows that $B^{\bs{\theta }, \bs{\Psi}, m}$
 is a   $\sG^{X,\bs{\Psi},m}$-martingale. 
By \eqref{eq:B}, 
$B^{\bs{\theta }, \bs{\Psi}, m}_0=0$ and 
$B^{\bs{\theta }, \bs{\Psi}, m}$ has continuous sample paths, as  
$X^{\bs{\theta }, \bs{\Psi}, m}$ has continuous sample paths.
 Moreover, by \eqref{eq:X_m} and \eqref{eq:B},
 $B^{\bs{\theta }, \bs{\Psi}, m}$ has the same quadratic covariance as $W^m$.
 Hence, L\'{e}vy's characterisation shows that $B^{\bs{\theta }, \bs{\Psi}, m}$ is a $\sG^{X,\bs{\Psi},m}$-Brownian motion, which along with  \eqref{eq:B} implies that $X^{\bs{\theta }, \bs{\Psi}, m}$ admits the following semi-martingale representation
 with respect to the filtration $\sG^{X,\bs{\Psi},m}$:
\begin{equation}
\label{eq:equiv dynamic X}
\d X^{\bs{\theta }, \bs{\Psi}, m}_t =  \sE\big[\bs{\theta} \big| \cG^{X, \bs{\Psi}, m}_t  \big]\; Z^{\bs{\theta }, \bs{\Psi}, m}_t \d t  + \d B^{\bs{\theta }, \bs{\Psi}, m}_t,
\q t\in [0,T], \q X^{\bs{\theta }, \bs{\Psi}, m}_0=x_0.
\end{equation} 
This implies that  
 the  stochastic integral of a $\sG^{X,\Psi,m}$-progressively measurable process with respect to $X^{\bs{\theta}, \bs{\Psi}, m}$ is 
 $\sG^{X,\Psi,m}$-progressively measurable;
 note that
 for all $m\in \sN\cup\{0\}$, $\cG^{\bs{\Psi}}_{m} $ has been augmented with  $\sP$-null sets
 (see Definition \ref{Def: Admissible RL}). 
 Consequently, $\hat{\bs{\theta}}^{  \bs{\Psi}, m}$  in \eqref{eq: statistics}
 is $\cG^{\bs{\Psi}}_{m} $-measurable,
which along with
the fact that 
$\Psi_{m+1}$ depends only on 
$(\hat{\bs{\theta}}^{  \bs{\Psi}, n}, V^{\bs{\theta }, \bs{\Psi}, n})_{n=0}^m$,
the 
$\cG^{\bs{\Psi}}_{m} $-measurability of $(\hat{\bs{\theta}}^{  \bs{\Psi}, n}, V^{\bs{\theta }, \bs{\Psi}, n})_{n=0}^{m-1}$,
and 
the Borel measurability of $\psi_{m+1}$ 
implies the desired measurability of   $\Psi_{m+1}$.
This completes the induction step and finishes the proof.
 \end{proof} 
 
\bibliographystyle{siam}

\bibliography{lc_bayesian.bib}
\newpage

 \end{document}